\documentclass[twoside,11pt]{article}

\usepackage{jmlr2e_arxiv}

\usepackage{csquotes}
\usepackage{epsfig}
\usepackage{color}

\usepackage{balance} 
\usepackage{algorithmic}
\usepackage{algorithm}
\usepackage{multirow}
\usepackage{graphicx}
\usepackage{amsmath}
\usepackage{amssymb}
\usepackage{amsfonts}
\usepackage{xspace}
\usepackage{url}
\usepackage{booktabs}
\usepackage{bbold}
\usepackage{tcolorbox, changepage, dsfont}
\usepackage{pifont}

\newcommand{\IGNORE}[1]{}
\newcommand{\sbt}{\mathrm{s.t. }}
\newcommand{\sign}{\text{sign}}

\newcommand{\B}{\boldsymbol}
\newcommand{\M}{\mathbf}
\newcommand{\cmark}{\ding{51}}%
\newcommand{\xmark}{\ding{55}} 

\newcommand\eqdef{\mathrel{\overset{\makebox[0pt]{\mbox{\normalfont\tiny\sffamily def}}}{=}}}
\newcommand{\normaltilde}{\raise.17ex\hbox{$\scriptstyle\sim$}}

\newtheorem{prop}{Proposition}

\newtheorem{cor}{Corollary}
\newtheorem{asu}{Assumption}

\DeclareMathOperator*{\argmin}{argmin}

\ShortHeadings{Learning Sparse Classifiers}{Dedieu, Hazimeh, and Mazumder}
\firstpageno{1}

\begin{document}

\title{Learning Sparse Classifiers: Continuous and Mixed Integer Optimization Perspectives}

\author{\name Antoine Dedieu \email tonio.dedieu@gmail.com  \\
       \addr Operations Research Center\\
       Massachusetts Institute of Technology\\
       Cambridge, MA 02139, USA
       \AND
        \name Hussein Hazimeh \email hazimeh@mit.edu \\
        \addr Operations Research Center\\
               Massachusetts Institute of Technology\\
               Cambridge, MA 02139, USA
       \AND
       \name Rahul Mazumder \email rahulmaz@mit.edu \\
       \addr Sloan School of Management\\
       Operations Research Center \\
       MIT Center for Statistics\\
             Massachusetts Institute of Technology\\
             Cambridge, MA 02142, USA}

\editor{}

\maketitle

\begin{abstract}
We consider a discrete optimization formulation for learning sparse classifiers, where the outcome depends upon a linear combination of a small subset of features. Recent work has shown that mixed integer programming (MIP) can be used to solve (to optimality) $\ell_0$-regularized regression problems at scales much larger than what was conventionally considered possible. Despite their usefulness, MIP-based global optimization approaches are  significantly slower compared to the relatively mature algorithms for $\ell_1$-regularization and heuristics for nonconvex regularized problems. We aim to bridge this gap in computation times by developing new MIP-based algorithms for $\ell_0$-regularized classification. We propose two classes of scalable algorithms: an exact algorithm that can handle $p\approx 50,000$ features in a few minutes, and approximate algorithms that can address instances with $p\approx 10^6$ in times comparable to the fast $\ell_1$-based algorithms. Our exact algorithm is based on the novel idea of \textsl{integrality generation}, which solves the original problem (with $p$ binary variables) via a sequence of mixed integer programs that involve a small number of binary variables. Our approximate algorithms are based on coordinate descent and local combinatorial search. In addition, we present new estimation error bounds for a class of $\ell_0$-regularized estimators. Experiments on real and synthetic data demonstrate that our approach leads to models with considerably improved statistical performance (especially, variable selection) when compared to competing methods.
\end{abstract}

\begin{keywords}
  sparsity, sparse classification, l0 regularization, mixed integer programming
\end{keywords}

\section{Introduction}
We consider the problem of sparse linear classification, where the output depends upon a linear combination of a small subset of features. This is a core problem in high-dimensional statistics~\citep{hastie2015statistical} where the number of features $p$ is comparable to or exceeds the number of samples $n$. In such settings, sparsity can be useful from a statistical viewpoint and can lead to more 
interpretable models. We consider the typical binary classification problem with samples $(\M{x}_i, y_i), i = 1, \ldots, n$, features $\M{x}_i \in \mathbb{R}^p$, and outcome $y_i \in \{-1,+1\}$. In the spirit of best-subset selection in linear regression~\citep{miller2002subset}, we consider minimizing the empirical risk (i.e., a surrogate for the misclassification error) while penalizing the number of nonzero coefficients:
\begin{equation} \label{eq:l0vanilla}
\min \limits_{ \B{\beta}  \in \mathbb{R}^p } \frac{1}{n} \sum_{i=1}^n f \left( \langle \mathbf{x}_i,  \B{\beta} \rangle , y_i \right) + \lambda_0 \| \B{\beta}  \|_0,
\end{equation} 
where $f: \mathbb{R} \times \{-1,+1\} \to \mathbb{R}$ is the loss function (for example, hinge or logistic loss). The term $\| \B{\beta} \|_0$ is the $\ell_0$ (pseudo)-norm of $\B\beta$ which is equal to the number of nonzeros in $\B{\beta}$, and $\lambda_0>0$ is a regularization parameter which controls the number of nonzeros in $\B{\beta}$. We ignore the intercept term in the above and throughout the paper to simplify the presentation. Problem \eqref{eq:l0vanilla} is known to be NP-Hard and poses computational challenges~\citep{natarajan1995sparse}. In this paper, we introduce scalable algorithms for this optimization problem using techniques based on both continuous and discrete optimization, specifically, mixed integer programming~\citep{wolsey1999integer}.

There is an impressive body of work on obtaining approximate solutions to Problem~\eqref{eq:l0vanilla}: popular candidates include greedy (a.k.a. stepwise) procedures~\citep{bahmani2013greedy}, 
proximal gradient methods~\citep{blumensath2009iterative}, among others. The $\ell_{1}$-norm~\citep{tibshirani1996regression} is often used as a convex surrogate to the $\ell_0$-norm, leading to a convex optimization problem. Nonconvex continuous penalties (such as MCP and SCAD)~\citep{zhang2010nearly}
provide better approximations of the $\ell_0$-penalty but lead to nonconvex problems, for which gradient-based methods~\citep{gong2013general,boyd2014,li2015accelerated} and coordinate descent~\citep{ncvreg,sparsenet} are often used. These algorithms may not deliver optimal solutions for the associated nonconvex problem.  Fairly recently, there has been considerable interest in exploring Mixed Integer Programming (MIP)-based methods~\citep{wolsey1999integer,bestsubset,ustun2016supersparse,sato2016feature} to solve variants of Problem~\eqref{eq:l0vanilla} to optimality. MIP-based methods create a branch-and-bound tree that simultaneously leads to feasible solutions and corresponding lower-bounds (a.k.a. dual bounds). Therefore, these methods deliver  optimality certificates for the nonconvex optimization problem. Despite their appeal in delivering nearly optimal solutions to Problem~\eqref{eq:l0vanilla}, MIP-based algorithms are usually computationally expensive compared to convex relaxations or greedy (heuristic) algorithms~\citep{fastbestsubset,hastie2020}---possibly limiting their use in time-sensitive applications that arise in practice.

The vanilla version of best-subset selection is often perceived as a gold-standard for high-dimensional sparse linear regression, when the signal-to-noise ratio (SNR) is high. However, it suffers from overfitting when the SNR becomes moderately low~\citep{friedman2001elements,lowsnr,hastie2020}. A possible way to mitigate this shortcoming is by imposing additional continuous regularization---see for example,~\citet{lowsnr,fastbestsubset} for studies in the (linear) regression setting. Thus, we consider an extended family of estimators which combines $\ell_0$ and $\ell_q$ (for $q \in \{1, 2\}$) regularization:
\begin{equation} \label{eq:mainlagrangian}
\min \limits_{\B{\beta}  \in \mathbb{R}^p}~~~\frac{1}{n} \sum_{i=1}^n f \left( \langle \mathbf{x}_i,  \B{\beta} \rangle , y_i \right) + \lambda_0 \| \B{\beta}  \|_0 + \lambda_q \| \B{\beta} \|_q^q,
\end{equation}
where the regularization parameter $\lambda_0 \geq 0$ explicitly controls the sparsity in $\B\beta$, and $\lambda_q \geq 0$ controls the amount of continuous shrinkage on the nonzero coefficients of $\B\beta$ (for example, the margin in  linear SVM). 
In what follows, for notational convenience, we will refer to the combination of regularizers $\lambda_0 \| \B{\beta}  \|_0$ and $ \lambda_q \| \B{\beta} \|_q^q$, as the $\ell_0$-$\ell_q$ penalty. 
For flexibility, our framework allows for both choices of $q \in \{1, 2\}$, and the value of $q$ needs to be specified a-priori by the practitioner. When $q=2$, Problem~\eqref{eq:mainlagrangian} seeks to deliver a solution $\B\beta$ with few nonzeros (controlled by the $\ell_{0}$-penalty) and a small $\ell_{2}$-norm (controlled by the ridge penalty).  
Similarly, when $q=1$, we seek a model $\B\beta$ that has a small $\ell_{1}$-norm and a small $\ell_0$-norm. If $\lambda_{1}$ is large, the $\ell_1$-penalty may also encourage zeros in the coefficients. Note that the primary role of the $\ell_0$-penalty is to control the number of nonzeros in $\B\beta$; and that of the $\ell_1$-penalty is to shrink the model coefficients. In our numerical experiments we observe that both choices of $q\in \{1, 2\}$ work quite well, with no penalty uniformly dominating the other. We refer the reader to~\citet{lowsnr} for complementary discussions in the regression setting.

A primary focus of our work is to propose new scalable algorithms for solving Problem~\eqref{eq:mainlagrangian}, 
with certificates of optimality (suitably defined). Problem~\eqref{eq:mainlagrangian} can be expressed using MIP formulations. 
However, these formulations lead to computational challenges for off-the-shelf commercial MIP solvers (such as Gurobi and CPLEX). To this end, we propose a new MIP-based algorithm that we call ``integrality generation'', which allows for solving instances of Problem~\eqref{eq:mainlagrangian} with $p\approx 50,000$ (where $n$ is small) to optimality within a few minutes.\footnote{Empirically, we observe the runtime to depend upon the number of nonzeros in the solution. The runtime can increase if the number of nonzeros in an optimal solution becomes large---see Section~\ref{sec:experiments} for details.} This appears to be well beyond the capabilities of state-of-the-art MIP solvers, including recent MIP-based approaches, as outlined below. To obtain high-quality solutions for larger problem instances, in times comparable to the fast $\ell_1$-based solvers~\citep{glmnet}, we propose approximate algorithms based on coordinate descent (CD)~\citep{wright2015coordinate} and local combinatorial optimization,\footnote{The local combinatorial optimization algorithms are based on solving MIP problems over restricted search-spaces; and are usually much faster to solve compared to the full problem~\eqref{eq:mainlagrangian}.} where the latter leads to higher quality solutions compared to CD. Our CD and local combinatorial optimization algorithms are publicly available through our fast C++/R toolkit \texttt{L0Learn}: on CRAN at \url{https://cran.r-project.org/package=L0Learn} and also at \url{https://github.com/hazimehh/L0Learn}.

From a statistical viewpoint, we establish new upper bounds on the estimation error for solutions obtained by globally minimizing  $\ell_0$-based estimators~\eqref{eq:mainlagrangian}. These error bounds (rates) appear to be better than current known bounds for $\ell_1$-regularization; and have rates similar to the optimal minimax rates for sparse least squares regression~\citep{raskutti_wainwright}, achieved by $\ell_0$-based regression procedures.

\textbf{Related Work and Contributions:}
There is a vast body of work on developing optimization algorithms and understanding the statistical properties of various sparse estimators~\citep{hastie2015statistical,stats-HDD}.
We present a brief overview of work that relates to our paper.

\textbf{Computation:}
An impressive body of work has developed fast algorithms for minimizing the empirical risk regularized with convex or nonconvex proxies to the $\ell_0$-norm, e.g., \citet{glmnet,ncvreg,sparsenet,NesterovComposite,shalev2012proximal}. Below, we discuss related work that directly optimize objective functions involving an $\ell_0$ norm (in the objective or as a constraint). 

Until recently, global optimization with $\ell_0$-penalization was rarely used beyond $p = 30$ as popular software packages for best-subset selection (for example, \texttt{leaps} and \texttt{bestglm}) are unable to handle larger instances. \citet{bestsubset}~demonstrated that $\ell_0$-regularized regression problems could be solved to near-optimality for $p \approx 10^3$ by leveraging advances in first-order methods and the capabilities of modern MIP solvers such as Gurobi. \citet{bertsimas2017logistic,sato2016feature} extend the work of \citet{bestsubset} to solve $\ell_0$-regularized logistic regression by using an outer-approximation approach that can address problems with $p$ in the order of a few hundreds. \citet{bertsimas2017sparse} propose a cutting plane algorithm for the  $\ell_0$-constrained least squares problem with additional ridge regularization---they can handle problems with $n \approx p$, when the feature correlations are low and/or the amount of ridge regularization is taken to be sufficiently large. \citet{BertsimasSparseClassification} adapt \citet{bertsimas2017sparse}'s work to solve classification problems (e.g., with logistic or hinge loss). The approach of \citet{BertsimasSparseClassification} appears to require a fairly high amount of ridge regularization for the cutting plane algorithm to work well.

A separate line of research investigates algorithms to obtain feasible solutions for $\ell_0$-regularized problems. These algorithms do not provide dual bounds like MIP-based algorithms, but can be computationally much faster. These include: (i) first-order optimization algorithms based on hard thresholding, such as Iterative Hard Thresholding (IHT) \citep{blumensath2009iterative} and GraSP \citep{bahmani2013greedy}, (ii) second-order optimization algorithms inspired by the Newton method such as NTGP \citep{yuan2017}, NHTP \citep{zhou2019global}, and NSLR \citep{wang2019fast}, and  (iii) coordinate descent methods based on greedy and random coordinate selection rules \citep{BeckSparsityConstrained, randomCDL0}.

\citet{fastbestsubset} present algorithms that offer a \emph{bridge} between MIP-based global optimization and good feasible solutions for $\ell_0$-regularized problems, by using a combination of CD and local combinatorial optimization. The current paper is similar in spirit, but makes new contributions. We extend the work of~\citet{fastbestsubset} (which is tailored to the least squares loss function) to address the more general class of problems in~\eqref{eq:mainlagrangian}. Our algorithms can deliver solutions with better statistical performance (for example, in terms of variable selection and prediction error) compared to the popular fast algorithms for sparse learning (e.g., based on $\ell_1$ and MCP regularizers). Unlike heuristics that simply deliver an upper bound, MIP-based approaches attempt to solve~\eqref{eq:mainlagrangian} to optimality.
They can (i) certify via dual bounds the quality of solutions obtained by our CD and local search algorithms; and (ii) improve the solution if it is not optimal. However, as off-the-shelf MIP-solvers do not scale well, we present a new method: the \emph{Integrality Generation Algorithm} (IGA) (see Section~\ref{sec:MIO}) that allows us to \emph{solve} (to optimality) the MIP problems for instances that are larger than current 
methods~\citep{bertsimas2017logistic,BertsimasSparseClassification,sato2016feature}. 
The key idea behind our proposed IGA is to solve a sequence of relaxations of~\eqref{eq:mainlagrangian} by allowing only a subset of variables to be binary. On the contrary, a direct MIP formulation for~\eqref{eq:mainlagrangian} requires $p$ many binary variables; and can be prohibitively expensive for moderate values of $p$.

\textbf{Statistical Properties:}
Statistical properties of high-dimensional linear regression have been widely studied \citep{candes-sparse-estimation, raskutti_wainwright, bic-tsybakov, candes2007dantzig, lasso-dantzig}. One important statistical performance measure is the $\ell_2$-estimation error defined as $\| \B{\beta}^{*} - \hat{\B{\beta}} \|^2_2$, where $\B{\beta}^{*}$ is the $k$-sparse vector used in generating the true model and $\hat{\B{\beta}}$ is an estimator. For regression problems, \citet{candes-sparse-estimation, raskutti_wainwright} established a $(k/n)\log(p/k)$ lower bound on the $\ell_2$-estimation error. This optimal minimax rate is known to be achieved by a global minimizer of an $\ell_0$-regularized estimator~\citep{bic-tsybakov}. 
It is well known that the Dantzig Selector and Lasso estimators achieve a $(k/n)\log(p)$ error rate~\citep{candes2007dantzig, lasso-dantzig} under suitable assumptions for the high-dimensional regression setting. 
{Compared to regression, there has been limited work in deriving estimation error bounds for classification tasks. \citet{tarigan} study margin adaptation for $\ell_1$-norm SVM. A sizable amount of work focuses on 
the analysis of generalization error and risk bounds~\citep{greenshtein2006best,vdg_linear_models}. 
\citet{vssvm} study variable selection consistency of a nonconvex penalized SVM estimator, using a local linear approximation method with a suitable initialization. Recently, \citet{L1-SVM} proved a $(k/n) \log(p)$ upper-bound for the $\ell_2$-estimation error of $\ell_1$-regularized support vector machines (SVM), where $k$ is the number of nonzeros in the estimator that minimizes the population risk. \citet{Wainwright-logreg} show consistent neighborhood selection for high-dimensional Ising model using an $\ell_1$-regularized logistic regression estimator. \citet{one-bit} show  that one can obtain an error rate of $k/n\log(p/k)$ for 1-bit compressed sensing problems. In this paper, we present (to our knowledge) new $\ell_{2}$-estimation error bounds for a (global) minimizer of Problem~\eqref{eq:l0vanilla}---our framework applies to a family of loss functions including the hinge and logistic loss functions.}

\textbf{Our Contributions:} We summarize our contributions below:
\begin{itemize}
\item We develop fast first-order algorithms based on cyclic CD and local combinatorial search to (approximately) solve Problem \eqref{eq:mainlagrangian} (see Section \ref{sec:algorithm}). We prove a new result which establishes the convergence of cyclic CD under an asymptotic linear rate. We show that combinatorial search leads to solutions of higher quality than IHT and CD-based methods. We discuss how solutions from the $\ell_0$-penalized formulation, i.e., Problem~\eqref{eq:mainlagrangian}, can be used to obtain solutions to the cardinality constrained variant of~\eqref{eq:mainlagrangian}. We open source these algorithms through our sparse learning toolkit \texttt{L0Learn}. 
\item We propose a new algorithm: IGA, for solving Problem~\eqref{eq:mainlagrangian} to optimality. On  some problems, our algorithm reduces the time for solving a MIP formulation of Problem~\eqref{eq:mainlagrangian} from the order of hours to seconds, and it can solve high-dimensional instances with $p \approx 50,000$ and small $n$. The algorithm is presented in Section \ref{sec:MIO}.
\item We establish upper bounds on the squared $\ell_2$-estimation error for a cardinality constrained variant of Problem~\eqref{eq:mainlagrangian}. Our $(k/n)\log(p/k)$ upper bound matches the optimal minimax rate known for regression.
\item On a series of high-dimensional synthetic and real data sets (with $p \approx 10^5$), we show that our proposed algorithms can achieve significantly better statistical performance in terms of prediction (AUC), variable selection accuracy, and support sizes, compared to state-of-the-art algorithms (based on $\ell_1$ and local solutions to $\ell_0$ and MCP regularizers). 
Our proposed CD algorithm compares favorably in terms of runtime compared to current popular toolkits~\citep{glmnet,ncvreg,bahmani2013greedy,zhou2019global} for sparse classification. 
\end{itemize}

\subsection{Preliminaries and Notation} \label{sec:formulation-notation}
For convenience, we introduce the following notation: 
$$g(\B{\beta}) \eqdef \frac{1}{n} \sum_{i=1}^n f \left( \langle \mathbf{x}_i,  \B{\beta} \rangle , y_i \right)~~~\text{and}~~~G(\B{\beta}) \eqdef  g(\B{\beta})  + \lambda_1 \| \B{\beta} \|_1 + \lambda_2 \| \B{\beta} \|_2^2.$$
Problem \eqref{eq:mainlagrangian} is an instance of the following (more general) problem:
\begin{equation} \label{eq:lagrangian}
 \min \limits_{ \B{\beta}  \in \mathbb{R}^p } P(\B{\beta}) \eqdef G(\B{\beta}) + \lambda_0 \| \B{\beta} \|_0.
\end{equation}
In particular, Problem \eqref{eq:mainlagrangian} with $q=1$ is equivalent to Problem \eqref{eq:lagrangian} with $\lambda_2 = 0$. Similarly, Problem \eqref{eq:mainlagrangian} with $q=2$ is equivalent to  Problem \eqref{eq:lagrangian} with $\lambda_1 = 0$. Problem (3) will be the focus in our algorithmic development.

We denote the set $\{1,2,\dots,p\}$ by $[p]$ and the canonical basis for $\mathbb{R}^p$ by $\B{e}_1, \dots, \B{e}_p$. For $\B{\beta}\in \mathbb{R}^p$, we use $\text{Supp}(\B{\beta})$ to denote the support of $\B{\beta}$, i.e., the indices of its nonzero entries. For $S \subseteq [p]$, $\B{\beta}_S \in \mathbb{R}^{|S|}$ denotes the subvector of $\B{\beta}$ with indices in $S$. Moreover, for a differentiable function $g(\B{\beta})$, we use the notation $\nabla_S g(\B{\beta})$ to refer to the subvector of the gradient $\nabla g(\B{\beta})$ restricted to coordinates in $S$. We let $\mathbb{Z}$ and $\mathbb{Z}_{+}$ denote the set of integers and non-negative integers, respectively. A convex function $g(\B\beta)$ is said to be $\mu$-strongly convex if $\B\beta \mapsto g(\B\beta)-\mu\|\B\beta\|_{2}^2/2$ is convex. A function $h(\B{\beta})$ is said to be Lipschitz with parameter $L$ if $\| h(\B{\beta}) - h(\B{\alpha}) \|_2 \leq L \| \B{\beta} - \B{\alpha} \|_2$ for all $\B{\beta}, \B{\alpha}$ in the domain of the function.

\subsection{Examples of Loss Functions Considered} \label{sec:supportedloss}
In Table \ref{table:lossfunctions}, we give examples of popular classification loss functions that fall within the premise of our algorithmic framework and statistical theory. The column ``{FO \& Local Search}" indicates whether these loss functions are amenable to our first-order and local search algorithms (discussed in Section \ref{sec:algorithm}).\footnote{That is, these algorithms are guaranteed to converge to a stationary point (or a local optimum) for the corresponding optimization problems.} The column ``{MIP}" indicates whether the loss function leads to an optimization problem that can be solved (to optimality) via the MIP methods discussed in Section \ref{sec:MIO}. 
Finally, the column ``Error Bounds" indicates if the statistical error bounds (estimation error) discussed in Section \ref{sec:error-bound} apply to the loss function.

\begin{table}[tb]
\centering
\begin{tabular}{@{}llccc@{}}
\toprule
Loss             & $f(\hat{v},v)$ & FO \& Local Search & MIP & Error Bounds \\ \midrule
Logistic &  $\log(1+e^{- \hat{v} v })$             &        \cmark             &     \cmark       &    \cmark                       \\
Squared Hinge        &  $\max(0, 1 - \hat{v} v )^2$             &      \cmark               &    \cmark        &     \xmark                      \\
Hinge                &   $\max(0, 1 - \hat{v} v )$            &      \  \cmark*              &   \cmark         &    \cmark                       \\ \bottomrule
\end{tabular}
\caption{Examples of loss functions we consider. ``*'' denotes that our proposed first-order and local search methods apply upon using \citet{nesterov2012efficiency}'s smoothing on the non-smooth loss function.}
\label{table:lossfunctions}
\end{table}

\section{First-Order and Local Combinatorial Search Algorithms}
\label{sec:algorithm}

Here we present fast cyclic CD and local combinatorial search algorithms for obtaining high-quality local minima (we make this notion precise later) for Problem~\eqref{eq:lagrangian}. Our framework assumes that $g(\B\beta)$ is differentiable and has a Lipschitz continuous gradient. We first present a brief overview of the key ideas presented in this section, before diving into the technical details.

Due to the nonconvexity of~\eqref{eq:lagrangian}, the quality of the solution obtained depends on the algorithm---with local search and MIP-based algorithms leading to solutions of higher quality. The fixed points of the algorithms considered satisfy certain necessary optimality conditions for~\eqref{eq:lagrangian}, leading to different classes of local minima. 
In terms of solution quality, there is a hierarchy among these classes. 
We show that for Problem~\eqref{eq:lagrangian}, the minima corresponding to the different algorithms satisfy the following hierarchy:
\begin{equation}\label{eq:hierarchy}
\text{MIP Minima} ~\subseteq~ \text{Local Search Minima}  ~\subseteq~  \text{CD Minima}  ~\subseteq~ \text{IHT Minima}.
\end{equation}
The fixed points of the IHT algorithm contain the fixed points of the CD algorithm. As we move to the left in the hierarchy, the fixed points of the algorithms satisfy stricter necessary optimality conditions. At the top of the hierarchy, we have the global minimizers, which can be obtained by solving a MIP formulation of~\eqref{eq:lagrangian}.

Our CD and local search algorithms can run in times comparable to the fast $\ell_1$-regularized approaches~\citep{glmnet}. These algorithms can lead to high-quality solutions that can be used as warm starts for MIP-based algorithms. The MIP framework of Section~\ref{sec:MIO} can be used to certify the quality of these solutions via dual bounds, and to improve over them (if they are sub-optimal).

In Section \ref{section:cd}, we introduce cyclic CD for Problem~\eqref{eq:lagrangian} and study its convergence properties. Section~\ref{section:localsearch} discusses how the solutions of cyclic CD can be improved by local search and presents a fast heuristic for performing local search in high dimensions. In Section \ref{sec:algo-IHT}, we discuss how our algorithms can be used to obtain high-quality (feasible) solutions to the {\emph{cardinality constrained}} counterpart of~\eqref{eq:lagrangian}, in which the complexity measure $\| \B\beta \|_0$ appears as a constraint and not a penalty as in~\eqref{eq:lagrangian}.  
Finally, in Section \ref{section:L0Learn}, we briefly present implementation aspects of our toolkit \texttt{L0Learn}.

\subsection{Cyclic Coordinate Descent: Algorithm and Computational Guarantees} \label{section:cd}
We describe a cyclic CD algorithm for Problem \eqref{eq:lagrangian} and establish its convergence to stationary points of~\eqref{eq:lagrangian}.

\textbf{Why cyclic CD?} We briefly discuss our rationale for choosing cyclic CD. Cyclic CD has been shown to be among the fastest algorithms for fitting generalized linear models with convex and nonconvex regularization (e.g., $\ell_1$, MCP, and SCAD) \citep{glmnet,sparsenet, ncvreg}. Indeed, it can effectively exploit sparsity and active-set updates, making it suitable for solving high-dimensional problems (e.g., with $p \sim 10^6$ and small $n$). Algorithms that require evaluation of the full gradient at every iteration (such as proximal gradient, stepwise, IHT or greedy CD algorithms) have difficulties in scaling with $p$~\citep{nesterov2012efficiency}. In an earlier work, \citet{randomCDL0} proposed random CD for problems similar to \eqref{eq:lagrangian} (without an 
$\ell_1$-regularization term in the objective). However, recent studies have shown that cyclic CD can be faster than random 
CD~\citep{BeckConvergence,gurbuzbalaban2017cyclic,fastbestsubset}.
Furthermore, for $\ell_0$-regularized regression problems, cyclic CD is empirically seen to obtain solutions of higher quality (e.g., in terms of optimization and statistical performance) compared to random CD~\citep{fastbestsubset}.

\textbf{Setup.} Our cyclic CD algorithm for~\eqref{eq:lagrangian} applies to problems where $g(\B{\beta})$ is convex, continuously differentiable, and non-negative. Moreover, we will assume that the gradient of $\B\beta \mapsto g(\B{\beta})$ is coordinate-wise Lipschitz continuous, i.e., for every $i \in [p]$, $\B{\beta} \in \mathbb{R}^p$ and $s \in \mathbb{R}$, we have:
\begin{equation} \label{eq:coordinateLipschitz}
| \nabla_i g(\B{\beta} + \B{e}_i s) -  \nabla_i g(\B{\beta} ) | \leq L_i |s|,
\end{equation}
where $L_i > 0$ is the Lipschitz constant for coordinate $i$. This assumption  leads to the block Descent Lemma \citep{bertsekas2016nonlinear}, which states that 
\begin{equation} \label{eq:blockdescent}
g(\B{\beta} + \B{e}_i s) \leq g(\B{\beta}) + s \nabla_i g(\B{\beta}) + 
\frac12{L_i}s^2.
\end{equation}
Several popular loss functions for classification  fall under the above setup. For example, logistic loss and squared hinge loss satisfy~\eqref{eq:coordinateLipschitz} with $L_i = {\| \B X_i \|_2^2}/{4n}$ and $L_i = {2 \| \B X_i \|_2^2}/{n}$, respectively (here, $\M{X}_i$ denotes the $i$-th column of the data matrix $\M X$).

\textbf{CD Algorithm.} Cyclic CD~\citep{bertsekas2016nonlinear} updates one coordinate at a time (with others held fixed) in a cyclical fashion. 
Given a solution $\B{\beta} \in \mathbb{R}^p$, we attempt to find a new solution by changing the $i$-th coordinate of $\B\beta$---i.e., we find $\B\alpha$ such that ${\alpha}_j = {\beta}_j$ for all $j \neq i$ and $\alpha_i$ minimizes the one-dimensional function: $\beta_{i} \mapsto P(\B{\beta})$. However, for the examples we consider (e.g., logistic and squared hinge losses), there is no closed-form expression for this minimization problem. This makes the algorithm computationally inefficient compared to $g$ being the squared error loss~\citep{fastbestsubset}. Using~\eqref{eq:blockdescent}, we consider a quadratic upper bound $\tilde{g}(\B\alpha; \B\beta)$ for $g(\B\alpha)$ as follows:
\begin{equation}\label{defn-tilde-g}
g(\B{\alpha}) \leq \tilde{g}(\B\alpha; \B\beta) \eqdef g(\B{\beta}) + (\alpha_i - \beta_i) \nabla_i g(\B{\beta}) + \frac{\hat{L}_i}{2} (\alpha_i - \beta_i)^2,
\end{equation}
where $\hat{L}_i$ is a constant which satisfies $\hat{L}_i > L_i$. (For notational convenience, we hide the dependence of $\tilde{g}$ on $i$.) 
Let us define the function 
$$\psi(\B\alpha) = \sum_{i} \psi_{i} (\alpha_{i})~~~\text{where}~~~\psi_{i}(\alpha_i) = \lambda_0 \mathbb{1}(\alpha_i \neq 0) + \lambda_1 |\alpha_i| + \lambda_2 \alpha_i^2.$$ 
Adding $\psi(\B\alpha)$ to both sides of equation~\eqref{defn-tilde-g}, we get:
\begin{equation} \label{eq:Ptilde}
P(\B{\alpha}) \leq \widetilde{P}_{\hat{L}_i}(\B{\alpha};\B{\beta}) \eqdef \tilde{g}(\B\alpha;\B{\beta})  + \psi(\B\alpha).
\end{equation}
We can approximately minimize $P(\B{\alpha})$ w.r.t.~$\alpha_{i}$ (with other coordinates held fixed)
by minimizing its upper bound $\widetilde{P}_{\hat{L}_i}(\B{\alpha};\B{\beta})$ w.r.t.~$\alpha_i$. A solution $\hat{\alpha}_{i}$ for this one-dimensional optimization problem is given by
\begin{align} \label{eq:cdupperbd}
\hat{\alpha}_{i} \in \argmin_{\alpha_{i}}~\widetilde{P}_{\hat{L}_i}(\B{\alpha};\B{\beta})~~=~~\argmin_{\alpha_i} ~ \frac{\hat{L}_i}{2} \left( \alpha_i - \left(\beta_i - \frac{1}{{\hat{L}_i }} {\nabla_i g(\B{\beta} )} \right)  \right)^2 + \psi_{i} (\alpha_{i}).
\end{align}
Let $\hat{\B\alpha}$ be a vector whose $i$-th component is $\hat{\alpha}_{i}$, and $\hat{\alpha}_{j} = \beta_{j}$ for all $j \neq i$. Note that $P(\hat{\B{\alpha}}) \leq P({\B{\beta}})$---i.e., updating the $i$-th coordinate via~\eqref{eq:cdupperbd} with all other coefficients held fixed, leads to a decrease in the objective value $P({\B{\beta}})$. A solution of~\eqref{eq:cdupperbd} can be computed in closed-form; and is given by the  thresholding operator $T: \mathbb{R} \to \mathbb{R}$ defined as follows:
\begin{equation} \label{eq:thresholding}
T(c; \B{\lambda}, \hat{L}_i) = \begin{cases}
\frac{\hat{L}_i}{\hat{L}_i + 2\lambda_2}  \Big(  |c| - \frac{\lambda_1}{\hat{L}_i} \Big) \sign(c) \quad & \text{ if } \frac{\hat{L}_i}{\hat{L}_i + 2\lambda_2}   \Big(  |c| - \frac{\lambda_1}{\hat{L}_i} \Big) \geq \sqrt{\frac{2 \lambda_0}{\hat{L}_i + 2 \lambda_2}} \\
0 \quad & \text{~~otherwise}
\end{cases}
\end{equation}
where $c = \beta_i - \nabla_i g(\B{\beta} )/\hat{L}_i$ and $\B{\lambda} = (\lambda_0, \lambda_1, \lambda_2)$. 

Algorithm 1 below, summarizes our cyclic CD algorithm.
\begin{tcolorbox}[colback=white]
\centering
\textbf{Algorithm 1: Cyclic Coordinate Descent (CD)}
\begin{itemize}
\item \textbf{Input:} Initialization $\B{\beta}^{0}$ and constant $\hat{L}_i > L_i$ for every $i \in [p]$
\item \textbf{Repeat for $l = 0, 1, 2, \dots$ until convergence: } \begin{enumerate}
                        \item $i \gets 1 + (l \mod p)$ and $\B{\beta}^{l+1}_j \gets \B{\beta}^{l}_j$ for all $j \neq i$
                        \item Update ${\beta}^{l+1}_i \gets \argmin_{\beta_i} \widetilde{P}_{\hat{L}_i}({\beta}^{l}_1, \dots, \beta_i, \dots, {\beta}^{l}_p; \B{\beta}^{l})$ using \eqref{eq:thresholding} with $c~=~{\beta}^{l}_i~-~{\nabla_i g(\B{\beta}^{l} )}/{\hat{L}_i }$
\end{enumerate}
\end{itemize}
\end{tcolorbox}

\textbf{Computational Guarantees.} The convergence of cyclic CD has been extensively studied for certain classes of continuous objective functions, e.g., see \citet{Tseng2001,bertsekas2016nonlinear,BeckConvergence} and the references therein. However, these results do not apply to our objective function due to the discontinuity in the $\ell_0$-norm. In Theorem~\ref{theorem:cdconvergence}, we establish a new result which shows that cyclic CD (Algorithm 1) converges at an asymptotic linear rate, 
and we present a characterization of the corresponding solution. Theorem~\ref{theorem:cdconvergence} is established under the following assumption:
\begin{asu} \label{assumption:strongconvexity}
Problem \eqref{eq:lagrangian} satisfies at least one of the following conditions:
 \begin{enumerate}
 \item Strong convexity of the continuous regularizer, i.e., $\lambda_2 > 0$.
 \item Restricted Strong Convexity: For some $u \in [p]$, the function $\B\beta_{S} \mapsto g(\B{\beta}_S)$ is strongly convex for every $S \subseteq [p]$ such that $|S| \leq u$. Moreover, $\lambda_0$ and the initial solution $\B{\beta}^{0}$ are chosen such that $P(\B{\beta}^{0}) < u \lambda_0$.
\end{enumerate}
\end{asu}
\begin{theorem} \label{theorem:cdconvergence}
Let $\{ \B{\beta}^{l} \}$ be the sequence of iterates generated by Algorithm 1. Suppose that Assumption \ref{assumption:strongconvexity} holds, then:
\begin{enumerate}
\item The support of $\B\beta^l$ stabilizes in a finite number of iterations, i.e., there exists an integer $N$ and support $S \subset [p]$ such that $\text{Supp}(\B{\beta}^{l}) = S$ for all $l \geq N$.
\item  Let $S$ be the support as defined in Part 1.  The sequence $\{ \B{\beta}^{l} \}$ converges to a solution $\B{\beta}^{*}$ with support $S$, satisfying:

\begin{equation}
\begin{aligned} \label{eq:CWminima}
& \B{\beta}^{*}_S \in \argmin_{\B{\beta}_S} G(\B{\beta}_S) \\
& |{\beta}^{*}_i | \geq \sqrt{\frac{2 \lambda_0}{\hat{L}_i + 2 \lambda_2}} \quad \text{ for } i \in S  \\
\text{and}~~~~~~& |\nabla_i g(\B{\beta}^{*} )| - \lambda_1 \leq \sqrt{2 \lambda_0 (\hat{L}_i + 2 \lambda_2)} \quad \text{ for } i \in S^c. 
\end{aligned}
\end{equation}
\item  
Let $S$ be the support as defined in Part 1. Let us define $H(\B\beta_{S})\eqdef g(\B{\beta}_S) + \lambda_2 \| \B{\beta}_S \|_2^2$.
Let $\sigma_{S}$ be the strong convexity parameter of $\B\beta_{S} \mapsto H(\B\beta_{S})$, and let $L_S$ be the Lipschitz constant of $\B\beta_{S} \mapsto \nabla_{S} H(\B\beta_{S})$. Denote $\hat{L}_{\max} = \max_{i \in S} \hat{L}_i $ and $\hat{L}_{\min} = \min_{i \in S} \hat{L}_i $. Then, there exists an integer $N'$ such that the following holds for all $t \geq N'$:
\begin{equation} \label{eq:cdrate}
{P(\B{\beta}^{(t+1)p}) - P(\B{\beta}^{*})} \leq \left( 1 - \frac{\sigma_S}{\gamma} \right)  
\left(P(\B{\beta}^{t p}) - P(\B{\beta}^{*}) \right),
\end{equation}
where $\gamma^{-1} = 2 \hat{L}_{\max} (1 + |S| {L_S^2} \hat{L}^{-2}_{\min}).$
\end{enumerate}
\end{theorem}

We provide a proof of Theorem \ref{theorem:cdconvergence} in the appendix. The proof is different from that of CD for $\ell_0$-regularized regression \citep{fastbestsubset} since we use inexact minimization for every coordinate update, whereas \citet{fastbestsubset} use exact minimization. At a high level, the proof proceeds as follows. In Part 1, we  prove a sufficient decrease condition which establishes that the support stabilizes in a finite number of iterations. In Part 2, we show that under Assumption \ref{assumption:strongconvexity}, the objective function is strongly convex when restricted to the stabilized support. After restriction to the stabilized support, we obtain convergence from standard results on cyclic CD \citep{bertsekas2016nonlinear}. In Part 3, we  show an asymptotic linear rate of convergence for  Algorithm~1. To establish this rate, we extend the linear rate of convergence of cyclic CD for smooth strongly convex functions by \citet{BeckConvergence} to our objective function (note that due to the presence of the $\ell_1$-norm, our objective is not smooth even after support stabilization).

\textbf{Stationary Points of CD versus IHT:} 
The conditions in \eqref{eq:CWminima} describe a fixed point of the cyclic CD algorithm and are necessary optimality conditions for Problem~\eqref{eq:lagrangian}.  We now show that the stationary conditions~\eqref{eq:CWminima} are strictly contained within the class of stationary points arising from the IHT algorithm \citep{blumensath2009iterative,BeckSparsityConstrained,bestsubset}. Recall that IHT can be interpreted as a proximal gradient algorithm, whose updates for 
Problem~\eqref{eq:lagrangian} are given by:
\begin{align} \label{eq:iht}
    \B{\beta}^{l+1} \in \argmin_{\B{\beta}}~ \left\{ \frac{1}{2\tau} \| \B{\beta} - (\B{\beta}^{l} - \tau \nabla g(\B{\beta}^{l}))  \|_2^2 + \psi(\B{\beta}) \right\},
\end{align}
where $\tau > 0$ is a step size. Let $L$ be the Lipschitz constant of $\B\beta \mapsto \nabla g(\B\beta)$, and let $\hat{L}$ be any constant satisfying $\hat{L} > L$. Update~\eqref{eq:iht} is guaranteed to converge to a stationary point if $\tau = 1/\hat{L}$ (e.g., see \citealt{PenalizedIHT,fastbestsubset}). Note that $\tilde{\B\beta}$ is a fixed point for~\eqref{eq:iht} if it satisfies \eqref{eq:CWminima} with $\hat{L}_i$ replaced with $\hat{L}$. The component-wise Lipschitz constant $L_i$ always satisfies $L_i \leq L$. For high-dimensional problems, we may have $L_i \ll L$ (see discussions in \citealt{BeckSparsityConstrained,fastbestsubset} for problems where $L$ grows with $p$ but $L_i$ is constant). Hence, the CD optimality conditions in~\eqref{eq:CWminima} are more restrictive than IHT---justifying a part of the hierarchy mentioned in~\eqref{eq:hierarchy}. An important practical consequence of this result is that CD may lead to solutions of higher quality than IHT.

\begin{remark}
A solution $\B\beta^*$ that satisfies the CD or IHT stationarity conditions is a local minimizer in the traditional sense used in nonlinear optimization.\footnote{That is, given a stationary point $\B\beta^*$, there is a small $\epsilon>0$ such that any $\B\beta$ lying in the set $\|\B\beta - \B\beta^*\|_2 \leq \epsilon$ will have an objective that is at least as large as the current objective value $P(\B\beta^*)$.} Thus, we use the terms stationary point and local minimizer interchangeably in our exposition.
\end{remark}

\subsection{Local Combinatorial Search} \label{section:localsearch}
We propose a local combinatorial search algorithm to improve the quality of solutions obtained by Algorithm 1. Given a solution from Algorithm~1, the idea is to perform small perturbations to its support in an attempt to improve the objective. This approach has been recently shown to be very effective (e.g, in terms of statistical performance) for $\ell_0$-regularized regression~\citep{fastbestsubset}, especially under difficult statistical settings (high feature correlations or $n$ is small compared to $p$). 
Here, we extend the approach of~\citet{fastbestsubset} to general loss 
functions, discussed in Section~\ref{section:cd}. 
As we consider a general loss function, performing exact local minimization becomes computationally expensive---we thus resort to an approximate minimization scheme. This makes our approach different from the least squares setting considered in~\citet{fastbestsubset}.

Our local search algorithm is iterative. It performs the following two steps at every iteration $t$: 
\begin{enumerate}
\item \textbf{Coordinate Descent}: We run cyclic CD (Algorithm 1) initialized from the current solution, to obtain a solution $\B{\beta}^{t}$ with support $S$.
\item \textbf{Combinatorial Search}: We attempt to improve $\B{\beta}^{t}$ by 
making a change to its current support $S$ via a \emph{swap} operation. In particular, we search for two subsets of coordinates $S_1 \subset S$ and $S_2 \subset S^c$, each of size at most $m$, such that removing coordinates $S_1$ from the support, adding $S_2$ to the support, and then optimizing over the coefficients in $S_2$, improves the current objective value. 
\end{enumerate}
To present an optimization formulation for the combinatorial search step (discussed above), we introduce some notation. Let $U^S$ denote a $p \times p$ matrix whose $i$-th row is $\B{e}_i^T$ if $i \in S$ and zero otherwise. Thus, for any $\B{\beta} \in \mathbb{R}^{p}$, $(U^S \B{\beta})_i = \beta_i$ if $i \in S$ and $(U^S \B{\beta})_i = 0$ if $i \notin S$. The combinatorial search step solves the following optimization problem:
\begin{equation} \label{eq:swaps}
 \min_{S_{1}, S_{2}, \B{\beta}} ~~ P(\B{\beta}^{t} - U^{S_1} \B{\beta}^{t} + U^{S_2} \B{\beta}) ~~~~~ \text{s.t. } ~~~~ S_1 \subset S, S_2 \subset S^c, |S_1| \leq m, |S_2| \leq m,
\end{equation}
where the optimization variables are the subsets $S_{1}$, $S_{2}$ and the coefficients of $\B\beta$ restricted to $S_{2}$.
If there is a feasible solution $\hat{\B{\beta}}$ to \eqref{eq:swaps} satisfying $P(\hat{\B{\beta}}) < P(\B{\beta}^{t})$, then we move to $\hat{\B{\beta}}$. Otherwise, the current solution $\B{\beta}^{t}$ cannot be improved by swapping subsets of coordinates, and the algorithm terminates. We summarize the algorithm below.
\begin{tcolorbox}[colback=white]
\centering
\textbf{Algorithm 2: CD with Local Combinatorial Search}
\begin{itemize}
\item \textbf{Input:} Initialization $\hat{\B{\beta}}^{0}$ and swap subset size $m$.
\item \textbf{Repeat for $t = 1, 2, \dots$: } \begin{enumerate}
                        \item ${\B{\beta}}^{t} \gets$ Output of cyclic CD initialized from $\hat{\B{\beta}}^{t-1}$. Let $S \gets \text{Supp}(\B{\beta}^{t})$.
                        \item Find a feasible solution $\hat{\B{\beta}}$ to \eqref{eq:swaps} satisfying $P(\hat{\B{\beta}}) < P(\B{\beta}^{t})$.
\item If Step 2 succeeds, then set $\hat{\B{\beta}}^{t} \gets \hat{\B{\beta}}$. Otherwise, if Step 2 fails, \textbf{terminate}.
\end{enumerate}
\end{itemize}
\end{tcolorbox}
Theorem~\ref{theorem:localsearch} shows that Algorithm 2 terminates in a finite number of iterations and provides a description of the resulting solution. 
\begin{theorem} \label{theorem:localsearch}
Let $\{ \B{\beta}^{t} \}$ be the sequence of iterates generated by Algorithm~2. Then, under Assumption~\ref{assumption:strongconvexity}, $\B{\beta}^{t}$ converges in finitely many steps to a solution $\B{\beta}^{*}$ (say). Let $S = \text{Supp}(\B{\beta}^{*})$. Then, $\B{\beta}^{*}$ satisfies the stationary conditions in~\eqref{eq:CWminima} (see Theorem~\ref{theorem:cdconvergence}). In addition, for every $S_1 \subset S$ and $S_2 \subset S^c$ with $|S_1| \leq m$, $|S_2| \leq m$, the solution $\B{\beta}^{*}$ satisfies the following condition:
\begin{equation} \label{eq:inescapable}
 P(\B{\beta}^{*}) \leq \min_{\B{\beta}} ~~ P(\B{\beta}^{*} - U^{S_1} \B{\beta}^{*} + U^{S_2} \B{\beta}).
\end{equation}
\end{theorem}
Algorithm 2 improves the solutions obtained from Algorithm~1. This observation along with the discussion in Section~\ref{section:cd}, establishes the hierarchy of local minima in \eqref{eq:hierarchy}. The choice of $m$ in Algorithm 2 controls the quality of the local minima returned---larger values of $m$ will lead to solutions with better objectives. For a sufficiently large value of $m$, Algorithm~2 will deliver a global minimizer of Problem \eqref{eq:lagrangian}. The computation time of solving Problem~\eqref{eq:swaps} increases with $m$. We have observed empirically that small choices of $m$ (e.g., $m=1$) can 
lead to a global minimizer of Problem~\eqref{eq:lagrangian} even for some challenging high-dimensional problems where the features are highly correlated (see the experiments in Section \ref{sec:experiments}). 

Problem~\eqref{eq:swaps} can be formulated using MIP---this is discussed in Section~\ref{sec:MIO}, where we also present methods to solve it for large problems. The MIP-based framework allows us to (i) obtain good feasible solutions (if they are available) or (ii) certify (via dual bounds) that the current solution cannot be improved by swaps corresponding to size $m$. Note that due to the restricted search space, solving~\eqref{eq:swaps} for small values of $m$ can be much easier than solving Problem~\eqref{eq:lagrangian} using MIP solvers.  A solution $\B\beta^*$ obtained from Algorithm~2 has an appealing interpretation: being a fixed point of~\eqref{eq:inescapable}, $\B\beta^*$ cannot be improved by locally perturbing its support. This serves as a certificate describing the quality of the current (locally optimal) solution $\B\beta^*$.

In what follows, we present a fast method to obtain a good solution to Problem~\eqref{eq:swaps} for the special case of $m=1$.

\textbf{Speeding up Combinatorial Search when $m=1$:}
For Problem~\eqref{eq:swaps}, we first check if removing variable $i$, without adding any new variables to the support, improves the objective, i.e., $P(\B{\beta}^{t} - \B{e}_i {\beta}_i^{t}) < P(\B{\beta}^{t})$. If the latter inequality holds, we declare a success in Step~2 (Algorithm~2). Otherwise, we find a feasible solution to Problem~\eqref{eq:swaps} by solving 
\begin{equation} \label{eq:contmin}
\min_{\beta_j}~~ G(\B{\beta}^{t} - \B{e}_i {\beta}_i^{t} + \B{e}_j \beta_j )
\end{equation}
for every pair $S_{1}=\{i\}$ and $S_{2}=\{j\}$. Performing the full minimization in~\eqref{eq:contmin} for every $(i,j)$ can be expensive---so we propose an approximate scheme that is found to work relatively well in our numerical experience. 
We perform a \emph{few} proximal gradient updates by applying the thresholding operator defined in \eqref{eq:thresholding} with the choice $\hat{L}_j = L_j$, $\lambda_0 = 0$, and using $\beta_j=0$ as an  initial value, to approximately minimize~\eqref{eq:contmin}---this helps us identify if the inclusion of coordinate $j$ leads to a success in Step~2.

The method outlined above requires approximately solving Problem~\eqref{eq:contmin} for $|S|(p-|S|)$ many $(i,j)$-pairs (in the worst case). This cost can be further reduced if we select $j$ from a small subset of coordinates outside the current support $S$, i.e., 
$j \in J \subset S^c$, where $|J| < p-|S|$.
We choose $J$ so that it corresponds to the $q$ largest (absolute) values of the gradient $|\nabla_{j} g(\B{\beta}^{t} - \B{e}_i {\beta}_i^{t})|$, $j \in S^c$. As explained in Section~\ref{sec:choice-J}, this choice of $J$ ensures that we search among coordinates $j \in S^c$ that lead to the maximal decrease in the current objective with one step of a proximal coordinate update initialized from $\beta_{j}=0$. 
We summarize the proposed method in Algorithm 3.
\begin{figure}[tb]
\begin{tcolorbox}[colback=white]
\begin{itemize}
\item[] \textbf{Algorithm~3: Fast Heuristic for Local Search when $m=1$}
\item \textbf{Input}: Restricted set size $q \in \mathbb{Z}_{+}$ such that $q \leq p - |S|$.
\item \textbf{For every $i \in S$}:
  \begin{enumerate}
  \item If $P(\B{\beta}^{t} - \B{e}_i {\beta}_i^{t}) < P(\B{\beta}^{t})$ then \textbf{terminate} and return $\B{\beta}^{t} - \B{e}_i {\beta}_i^{t}$.
  \item Compute $\nabla_{S^c} g(\B{\beta}^{t} - \B{e}_i {\beta}_i^{t})$ and let $J$ be the set of indices of the $q$ components with the largest values of $|\nabla_j g(\B{\beta}^{t} - \B{e}_i {\beta}_i^{t})|$ for $j \in S^c$.
  \item For every $j \in J$:
  
  Solve $\hat{\beta}_j \in \argmin_{\B{\beta}_{j} \in \mathbb{R}} G(\B{\beta}^{t} - \B{e}_i {\beta}_i^{t} + \B{e}_j \beta_j)$ by iteratively applying the thresholding operator in~\eqref{eq:thresholding} (with $\hat{L}_i = L_i$ and $\lambda_0 = 0$).
  If $P(\B{\beta}^{t} - \B{e}_i {\beta}_i^{t} + \B{e}_j \hat{\beta}_j) < P(\B{\beta}^{t})$, \textbf{terminate} and return $\B{\beta}^{t} - \B{e}_i {\beta}_i^{t} + \B{e}_j \hat{\beta}_j$. 
  \end{enumerate}
\end{itemize}
\end{tcolorbox}
\end{figure}

The cost of applying the thresholding operator in step 3 of Algorithm 3 is $\mathcal{O}(n)$.\footnote{Assuming that the derivative of $f(\cdot,v)$ w.r.t the first argument can be computed in $\mathcal{O}(1)$, which is the case for common loss functions that we consider here.} For squared error loss, the cost can be improved to $\mathcal{O}(1)$ by reusing previously computed quantities from CD (see \citealt{fastbestsubset} for details). Our numerical experience suggests that using the above heuristic with values of $q \approx 0.05 \times p$ often leads to the same solutions returned by full exhaustive search. Moreover, when some of the features are highly correlated, Algorithm 2 with the heuristic above (or solving Problem~\ref{eq:swaps} exactly) performs better in terms of variable selection and prediction performance compared to state-of-the-art sparse learning algorithms (see Section~\ref{exp:varysamples} for numerical results).

\subsection{Solutions for the Cardinality Constrained Formulation} \label{sec:algo-IHT}
Algorithms~1 and 2 deliver good solutions for the $\ell_0$-penalized problem~\eqref{eq:lagrangian}. 
We now discuss how they can be used to obtain solutions to the cardinality constrained version:
\begin{align} \label{eq:constrainedopt}
   \min_{\B{\beta} \in \mathbb{R}^p} G(\B{\beta})~~~~\text{s.t.}~~~~\| \B{\beta} \|_0 \leq k,
\end{align}
where $k$ controls the support size of $\B\beta$. While the unconstrained formulation~\eqref{eq:lagrangian} is amenable to fast CD-based algorithms, some support sizes are often skipped as $\lambda_0$ is varied. For example, if we decrease $\lambda_0$ to $\lambda'_0$, the support size of the new solution can differ by more than one, even if $\lambda'_0$ is taken to be arbitrarily close to $\lambda_0$. On the other hand, formulation~\eqref{eq:constrainedopt} can typically return a solution with any desired support size,\footnote{Exceptions can happen if for a subset $T \subset [p]$ of size $k$, the minimum of $\B\beta_{T} 
\mapsto G(\B{\beta}_{T})$ has some coordinates in $\B\beta_{T}$ exactly set to zero. \label{footnote:degenerate}} and it  may be preferable over the unconstrained formulation in some applications due to its explicit control of the support size.

Suppose we wish to obtain a solution to Problem~\eqref{eq:constrainedopt} for a support size $k$  that is not available from a sequence of solutions from~\eqref{eq:lagrangian}.
We propose to apply the IHT algorithm on Problem~\eqref{eq:constrainedopt}, initialized by a solution from Algorithm 1 or 2. This leads to the following update sequence
\begin{equation} \label{eq:ihtcard}
\beta^{l+1} \gets \argmin_{ \|\B{\beta} \|_0 \leq k  }  \left\{\frac{1}{2 \tau}  \left\| \B{\beta} - \left(\B{\beta}^{l} - \tau \nabla g(\B{\beta}^{l})  \right) \right\|_2^2 + \lambda_1 \| \B{\beta} \|_1 + \lambda_2 \| \B{\beta} \|_2^2 \right\},
\end{equation}
for $l \geq 0$, with initial solution $\B\beta^0$ (available from the $\ell_0$-penalized formulation) and $\tau >0$ is a step size (e.g., see Theorem~\ref{theorem:ihtconvergence}). We note that update \eqref{eq:ihtcard} typically returns a solution of support size $k$ but there can be degenerate cases where a support size of $k$ is not possible (see footnote \ref{footnote:degenerate}).

\citet{BeckSparsityConstrained} has shown that greedy CD-like algorithms perform better than IHT for a class of problems similar to \eqref{eq:constrainedopt} (without $\ell_1$-regularization). However, it is computationally expensive to apply greedy CD methods to~\eqref{eq:constrainedopt} for the problem sizes we study here. Note that since we initialize IHT with a solution from Problem~\eqref{eq:lagrangian}, it converges rapidly to a high-quality feasible solution for Problem~\eqref{eq:constrainedopt}.


Theorem~\ref{theorem:ihtconvergence}, which follows from~\citet{lowsnr}, establishes that IHT is guaranteed to converge, and provides a characterization of its fixed points.

\begin{theorem} \label{theorem:ihtconvergence}
Let $\{ \B{\beta}^{l} \}$ be a sequence generated by the IHT algorithm 
updates~\eqref{eq:ihtcard}. Let $L$ be the Lipschitz constant of $\nabla g(\B{\beta})$ and let $\hat{L}> L$. Then, the sequence $\{ \B{\beta}^{l} \}$ converges for a step size $\tau ={1}/{\hat{L}}$. Moreover, $\B{\beta}^{*}$ with support $S$ is a fixed point of~\eqref{eq:ihtcard} iff  $\| \B{\beta}^{*} \|_0 \leq k$,
\begin{align*}
\B{\beta}^{*}_S \in \argmin_{\B{\beta}_S}~G(\B{\beta}_S) \label{eq:ihtstationary} \quad \text{ and } 
\quad |\nabla_{i} g(\B{\beta}^{*})| \leq  \delta_{(k)} \quad \text{ for } i \in S^c,
\end{align*}
where $\delta_{j}=|\hat{L} \beta^{*}_{j} - {\nabla_{j} g(\B{\beta}^{*})}|$ for $j \in [p]$ and 
$\delta_{(k)}$ is the $k$th largest value of $\{\delta_{j}\}_{1}^{p}$.
\end{theorem}
Lemma~\ref{lemma:comparison} shows that if a solution obtained by Algorithm 1 or 2 has a support size $k$, then it is a fixed point for the IHT update~\eqref{eq:ihtcard}.
\begin{lemma} \label{lemma:comparison}
Let $\gamma>1$ be a constant and $\B{\beta}^{*}$ with support size $k$ be a solution for 
Problem~\eqref{eq:lagrangian} obtained by using Algorithm 1 or 2 with $\hat{L}_i = \gamma L_i$. Set $\hat{L} = \gamma L$ and $\tau = {1}/{\hat{L}}$ in \eqref{eq:ihtcard}. Then, $\B{\beta}^{*}$ is a fixed point of the update~\eqref{eq:ihtcard}.
\end{lemma}

The converse of Lemma \ref{lemma:comparison} is not true, i.e., a fixed point of update~\eqref{eq:ihtcard} may not be a fixed point for Algorithm 1 or 2---see our earlier discussion around hierarchy~\eqref{eq:hierarchy}.\footnote{Assuming that we obtain the same support sizes for both the constrained and unconstrained formulations.} Thus, we can generally expect the solutions returned by Algorithm 1 or 2 to be of higher quality than those returned by IHT.

The following summarizes our procedure to obtain a path of solutions for Problem~\eqref{eq:constrainedopt} (here, we assume that $(\lambda_1, \lambda_2)$ are fixed and $\lambda_0$ varies):
\begin{enumerate}
\item Run Algorithm 1 or 2 for a sequence of $\lambda_0$ values to obtain a regularization path.
\item To obtain a solution to Problem~\eqref{eq:constrainedopt} with a support size (say $k$) that is not available in Step~1, we run the IHT updates~\eqref{eq:ihtcard}. The IHT updates are initialized with a solution from Step~1 having a support size smaller than $k$.
\end{enumerate}
As the above procedure uses high-quality solutions obtained by Algorithm 1 or 2 as advanced initializations for IHT, we expect to obtain notable performance benefits as compared to using IHT alone for generating the regularization path. Also, note that if a support size $k$ is available from Step 1, then there is no need to run IHT (see  Lemma~\ref{lemma:comparison}).

\subsection{L0Learn: A Fast Toolkit for $\ell_0$-regularized Learning} \label{section:L0Learn}
We implemented the algorithms discussed above in $\texttt{L0Learn}$: a fast sparse learning toolkit written in C++ along with an R interface. We currently support the logistic and squared-hinge loss functions,\footnote{This builds upon our earlier functionality for least squares loss~\citep{fastbestsubset}.} but the toolkit can be expanded and we intend to incorporate additional loss functions in the future. Following~\citet{glmnet,fastbestsubset}, we used several computational tricks to speed up the algorithms and improve the solution quality---these include: warm starts, active sets, correlation screening, a (partially) greedy heuristic to cycle through the coordinates, and efficient methods for updating the gradients by exploiting sparsity. 
Note that Problem~\eqref{eq:lagrangian} can lead to the same solution if two values of $\lambda_0$ are close to one another. 
To avoid this issue, we dynamically select a sequence of $\lambda_0$-values---this is an extension of an idea appearing in~\citet{fastbestsubset} for the least squares loss function.

\section{Mixed Integer Programming Algorithms}
\label{sec:MIO}

We present MIP formulations and a new scalable algorithm: IGA, for solving Problem \eqref{eq:lagrangian} to optimality. Compared to off-the-shelf MIP solvers (e.g., the commercial solver Gurobi), IGA leads to certifiably optimal solutions in significantly reduced computation times. IGA also applies to the local search problem of Section~\ref{section:localsearch}. We remind the reader that while Algorithm~1 (and Algorithm~2 for small values of $m$) leads to good feasible solutions quite fast, it does not deliver certificates of optimality (via dual bounds). The MIP framework can be used to certify the quality of solutions obtained by Algorithms 1 or 2 and potentially improve upon them.  

\subsection{MIP Formulations}\label{sec: basic-MIO}
Problem \eqref{eq:lagrangian} admits the following MIP formulation:
\begin{equation} \label{eq:MIP}
\begin{aligned}
\min\limits_{  \B{\beta}, \M{z}} ~~~ \left \{ G(\B{\beta}) + \lambda_0 \sum\limits_{i=1}^{p} z_i \right\}~~ & ~~\sbt &~~~|\beta_i| \leq  \mathcal{M} z_i,~ z_i \in \left\{0,1\right\}, ~ i \in [p] 
\end{aligned}
\end{equation}
where $\mathcal{M}$ is a constant chosen large enough so that some optimal solution $\B{\beta}^{*}$ to \eqref{eq:mainlagrangian} satisfies $\| \B{\beta}^{*} \|_{\infty} \leq \mathcal{M}$. Algorithm~1 and Algorithm~2 (with $m=1$) can be used to obtain good estimates of $\mathcal{M}$---see \citet{bestsubset,lowsnr} for possible alternatives on choosing $\mathcal{M}$.
In~\eqref{eq:MIP}, the binary variable $z_i$ controls whether $\beta_i$ is set to zero or not.  
If $z_i = 0$ then $\beta_i =0$, while if $z_i=1$ then $\beta_{i} \in [-{\mathcal M},{\mathcal M}]$ is allowed to be `free'. For the hinge loss with $q=1$, the problem is a Mixed Integer Linear Program (MILP). 
For $q=2$ and the hinge loss function,~\eqref{eq:MIP} becomes a Mixed Integer Quadratic Program (MIQP). Similarly, the squared hinge loss leads to a MIQP (for both $q\in \{1, 2\}$).
MILPs and MIQPs can be solved by state-of-the-art MIP solvers (e.g., Gurobi and CPLEX) for small-to-moderate sized instances. For other nonlinear convex loss functions such as logistic loss, two approaches are possible: (i)~nonlinear branch and bound  \citep{belotti2013mixed} or (ii)~outer-approximation in which a sequence of MILPs is solved until convergence (for example, see  \citealt{BertsimasSparseClassification,bertsimas2017sparse,sato2016feature}).\footnote{In this method, one obtains a convex piece-wise linear lower bound to a nonlinear convex problem. In other words, this leads to a polyhedral outer approximation (aka outer approximation) to the epigraph of the nonlinear convex function.} We refer the reader to~\citet{lee2011mixed} for a review of related approaches.

The local combinatorial search problem in \eqref{eq:swaps} can be cast as a 
variant of~\eqref{eq:MIP} with additional constraints; and is given by the following MIP:
\begin{subequations} \label{eq:localsearcMIP}
\begin{align}
\min\limits_{  \B{\beta}, \M{z}, \B{\theta}} \quad  & G(\B{\theta}) + \lambda_0 \sum\limits_{i=1}^{p} z_i \\
\text{s.t.}~~& \B{\theta} = \B{\beta}^{t} - \sum_{i \in S} \B{e}_i \beta^{t}_i (1 - z_i) + \sum_{i \in S^c} \B{e}_i \beta_i \label{eq:dummyvar}\\
& |\beta_i| \leq  \mathcal{M} z_i, ~~ i \in S^c \\
& \sum_{i \in S} z_i \geq |S| - m \label{eq:cut1}\\
& \sum_{i \in S^c} z_i \leq m \label{eq:cut2} \\
& z_i \in \left\{0,1\right\}, ~~ i \in [p] 
\end{align}
\end{subequations}
where $S = \text{Supp}(\B{\beta}^{t})$. The binary variables $z_i, i \in [p]$ perform the role of \emph{selecting} the subsets $S_1 \subset S$ and $S_2 \subset S^c$ (described in \eqref{eq:swaps}). Particularly, for $i \in S$, $z_i = 0$ means that $i \in S_1$---i.e., variable $i$ should be removed from the current support.
Similarly, for $i \in S^c$, $z_i = 1$ means that $i \in S_2$---i.e., variable $j$ should be added to the support in~\eqref{eq:localsearcMIP}. 
Constraints~\eqref{eq:cut1} and~\eqref{eq:cut2} enforce $|S_1| \leq m$ and $|S_2| \leq m$, respectively. The constraint in~\eqref{eq:dummyvar} forces the variable $\B{\theta}$ to be equal to $\B{\beta}^{t} - U^{S_1} \B{\beta}^{t} + U^{S_2} \B{\beta}$. 

Note that~\eqref{eq:localsearcMIP} has a smaller search space compared to the full formulation in \eqref{eq:MIP}---there are additional constraints in \eqref{eq:cut1} and \eqref{eq:cut2}, and the number of free continuous variables is $|S^c|$, as opposed to $p$ in \eqref{eq:MIP}. This reduced search space usually leads to notable reductions in the runtime compared to solving Problem \eqref{eq:MIP}. 

\subsection{Scaling up the MIP via the Integrality Generation Algorithm (IGA)}\label{sec:integrality-gen}
While state-of-the-art MIP solvers and outer-approximation based MILP approaches~\citep{BertsimasSparseClassification,sato2016feature} lead to impressive improvements over earlier MIP-based approaches, they often have long run times when solving high-dimensional classification problems with large $p$ and small $n$ (e.g., $p=50,000$ and $n=1000$). Our proposed algorithm, IGA, can solve~\eqref{eq:MIP} to \emph{global} optimality, for high-dimensional instances that appear to be beyond the capabilities of current MIP-based approaches. 
Loosely speaking, IGA solves a sequence of MIP-based relaxations or subproblems of Problem~\eqref{eq:MIP}; and exits upon obtaining a global optimality certificate for~\eqref{eq:MIP}.
The aforementioned MIP subproblems are obtained by relaxing a subset of the binary variables $\{z_{i}\}_{1}^{p}$ to lie within $[0,1]$, while the remaining variables are retained as binary. Upon solving the relaxed problem and examining the integrality of the continuous $z_i$'s, we create another (tighter) relaxation by allowing more variables to be binary---we continue in this fashion till convergence. The algorithm is formally described below.

The first step in our algorithm is to obtain a good upper bound for Problem~\eqref{eq:MIP}---this can be obtained by Algorithm 1 or 2. Let $\mathcal{I}$ denote the corresponding support of this solution. 
We then consider a relaxation of Problem \eqref{eq:MIP} by allowing
the binary variables in $\mathcal I^c$ to be continuous: 
\begin{subequations} \label{eq:PMIO}
\begin{align}
\min\limits_{\B{\beta}, \M{z}} \quad  & G(\B{\beta}) + \lambda_0 \sum\limits_{i=1}^{p} z_i \\
\text{s.t.}~~~& |\beta_i| \leq  \mathcal{M} z_i, ~~ i \in [p] \\
& z_i \in [0,1], ~~ i \in \mathcal{I}^c \\
& z_i \in \left\{0,1\right\}, ~~ i \in \mathcal{I}. 
\end{align}
\end{subequations}
The relaxation~\eqref{eq:PMIO} is a MIP (and thus nonconvex). The optimal objective of Problem~\eqref{eq:PMIO} is a lower bound to Problem~\eqref{eq:MIP}. In formulation~\eqref{eq:PMIO}, we place integrality constraints on $z_i, i \in \mathcal{I}$---all remaining variables $z_{i}$, $i \in {\mathcal I}^c$ are continuous. Let $\B{\beta}^{u}, \M{z}^{u}$ be the solution obtained from the $u$-th iteration of the algorithm. Then, in iteration $(u+1)$, we set $\mathcal{I} \gets \mathcal{I} \cup \{ i \ | \ z^{u}_i \neq 0, i \in {\mathcal I}^c \}$ and solve Problem~\eqref{eq:PMIO} (with warm-starting enabled). If at some iteration $u$, the vector $\M{z}^{u}$ is integral, then solution $\B{\beta}^{u}, \M{z}^{u}$ must be optimal for Problem \eqref{eq:MIP} and the algorithm terminates. We note that along the iterations, we obtain tighter lower bounds on the optimal objective of Problem~\eqref{eq:MIP}. Depending on the available computational budget, we can decide to terminate the algorithm at an early stage with a corresponding lower bound. The algorithm is summarized below:
\begin{tcolorbox}[colback=white]
\centering
\textbf{Algorithm 4: Integrality Generation Algorithm (IGA)}
\begin{itemize}
\item Initialize $\mathcal{I}$ to the support of a solution obtained by Algorithm 1 or 2.
\item \textbf{For $u=1,2,\dots$} perform the following steps till convergence: \begin{enumerate}
                        \item[1.] Solve the relaxed MIP~\eqref{eq:PMIO} to obtain a solution $\B{\beta}^{u}, \M{z}^{u}$.
                        \item[2.] Update $\mathcal{I} \gets \mathcal{I} \cup \{ i \ | \ z^{u}_{i} \neq 0, i \in {\mathcal I}^c  \}$.
						\end{enumerate}
\end{itemize}
\end{tcolorbox}
As we demonstrate in Section~\ref{sec:experiments}, Algorithm 4 can lead to significant speed-ups: it reduces the time to solve several sparse classification instances from the order of \emph{hours} to \emph{seconds}. This allows us to solve instances with $p \approx 50,000$ and $n \approx 1000$ within reasonable computation times. These instances are much larger than what has been reported in the literature prior to this work (see for example,~\citealt{sato2016feature,BertsimasSparseClassification}).
A main reason behind the success of Algorithm~4 is that Problem~\eqref{eq:PMIO} leads to a solution $\mathbf{z}$ with very few nonzero coordinates. Hence, a small number of indices are added to $\mathcal{I}$ in step 2 of the algorithm. Since $\mathcal{I}$ is typically small, \eqref{eq:PMIO} can be often solved significantly faster than the full MIP in \eqref{eq:MIP}---the branch-and-bound algorithm has a smaller number of variables to branch on. 
Lemma~\ref{lemma:sparse} provides some intuition on why the solutions of~\eqref{eq:PMIO} are sparse.
\begin{lemma} \label{lemma:sparse}
Problem \eqref{eq:PMIO} can be equivalently written as:
\begin{subequations} \label{eq:lassolike}
\begin{align} 
\min\limits_{  \B{\beta}, \M{z}_{\mathcal{I}}} \quad  & G(\B{\beta}) + \frac{\lambda_0}{\mathcal{M}} \sum_{i \in \mathcal{I}^c}  |\beta_{i}| + \lambda_0 \sum\limits_{i \in \mathcal{I}} z_i \\ 
\text{s.t.}~~& |\beta_i| \leq  \mathcal{M} z_i, ~~ i \in \mathcal{I} 
\\ & |{\beta}_{i} | \leq \mathcal{M}, ~~~ i \in {\mathcal I}^c
\\ & z_i \in \left\{0,1\right\}, ~~ i \in \mathcal{I}. 
\end{align}
\end{subequations}
\end{lemma}
We have observed empirically that in \eqref{eq:lassolike}, the $\ell_1$-regularization term $\sum_{i \in \mathcal{I}^c}  |\beta_{i}|$ encourages sparse solutions, i.e., many of the components ${\beta}_i, i \in {\mathcal{I}^c}$ are set to zero. Consequently, the corresponding $z_i$'s in~\eqref{eq:PMIO} are mostly zero at an optimal solution. The sparsity level is controlled by the regularization parameter ${\lambda_0}/{\mathcal{M}}$---larger values will lead to more $z_i$'s being set to zero in \eqref{eq:PMIO}. Thus, we expect Algorithm~4 to work well when $\lambda_0$ is set to a sufficiently large value (to obtain sufficiently sparse solutions). {Note that Problem \eqref{eq:lassolike} is different from the Lasso as it involves additional integrality constraints.}

\textbf{Optimality Gap and Early Termination:} 
Each iteration of Algorithm 4 provides an improved lower bound to Problem~\eqref{eq:MIP}. This lower bound, along with a good feasible solution (e.g., obtained from Algorithm~1 or 2), leads to an optimality gap: given an upper bound UB and a lower bound LB, the MIP optimality gap is defined by 
(UB - LB)/LB. This optimality gap serves as a certificate of global optimality. In particular, an early termination of Algorithm~4 leads to a solution with an associated certificate of optimality.

\textbf{Choice of $ ~\mathcal I$: } The performance of Algorithm~4 depends on the initial set $\mathcal{I}$. If the initial $\mathcal{I}$ is close to the support of an optimal solution to~\eqref{eq:MIP}, then our numerical experience suggests that Algorithm~4 can terminate within a few iterations. Moreover, our experiments (see Section \ref{exp:varysamples}) suggest that
Algorithm~2 can obtain an optimal or a near-optimal solution to Problem~\eqref{eq:MIP} quickly, leading to a high-quality initialization for $\mathcal{I}$.

In practice, if at iteration $u$ of Algorithm 4, the set $\{ i \ | \ z^{u}_{i} \neq 0  \}$ is large, then we add only a small subset of it to $\mathcal{I}$ (in our implementation, we choose the $10$ largest fractional $z_i$'s). Alternatively, while expanding $\mathcal I$, we can use a larger cutoff for the fractional $z_{i}$'s, i.e., we can take the indices $\{ i \ | \ z^{u}_{i} \geq \tau \}$
for some value of $\tau \in (0,1)$.
This usually helps in maintaining a small size in $\mathcal{I}$, which allows for solving the MIP subproblem \eqref{eq:PMIO} relatively quickly.

\textbf{IGA for Local Combinatorial Search:} While the discussion above was centered around the full MIP~\eqref{eq:MIP}---the IGA framework (and in particular, Algorithm 4)  extends, in principle, to the local combinatorial  search problem in~\eqref{eq:localsearcMIP} for $m\geq 2$.

\section{Statistical Properties: Error Bounds}
\label{sec:error-bound}
We derive non-asymptotic upper bounds on the coefficient estimation error for a family of $\ell_0$-constrained classification estimators (this includes the loss functions discussed in Section~\ref{sec:supportedloss}, among others). For our analysis, we assume that: $( \mathbf{x}_i,y_i), i \in [n]$ are i.i.d. draws from an unknown distribution $\mathbb{P}$. Using the notation of Section \ref{sec:algorithm}, we consider a loss function $f$ and define its population risk $\mathcal{L}(\B{\beta}) = \mathbb{E} \left( f \left( \langle \mathbf{x},  \B{\beta} \rangle ;  y \right)  \right)$, where the expectation is w.r.t.~the (population) distribution~$\mathbb{P}$. We let $ \B{\beta}^*$ denote a minimizer of the risk, that is:
\begin{equation} \label{def-beta0}
\B{\beta}^* \in \argmin\limits_{ \B{\beta} \in \mathbb{R}^{p}}~\mathcal{L}(\B{\beta}):= \mathbb{E} \left( f \left( \langle \mathbf{x},  \B{\beta} \rangle ;  y \right) \right).
\end{equation}
In the rest of this section, we let $k = \| \B{\beta}^*\|_0$ and $R=\| \B{\beta}^*\|_2$---i.e., the number of nonzeros and the Euclidean norm (respectively) of $\B\beta^*$. 
We assume $R\ge 1$. 
To estimate $\B\beta^*$, we consider the following estimator:\footnote{We drop the dependence of $\hat{\B\beta}$ on $R,k$ for notational convenience.}
\begin{equation} \label{learning-l0}
\hat{\B\beta} ~~ \in~~ \argmin \limits_{ \substack{ \B{\beta}  \in \mathbb{R}^p \\   \| \B{\beta}  \|_0 \le k, \ \| \B{\beta}  \|_2 \le 2R } } \ \  \frac{1}{n}  \sum_{i=1}^n f \left( \langle \mathbf{x}_i,  \B{\beta} \rangle ;  y_i \right),
\end{equation}
which minimizes the empirical loss with a constraint on the number of nonzeros in $\B\beta$ and a bound on the $\ell_2$-norm of $\B\beta$. The $\ell_2$-norm constraint in~\eqref{learning-l0} makes $\B\beta^*$ feasible for Problem~\eqref{learning-l0}; and ensures that $\hat{\B\beta}$ lies in a bounded set (which is useful for the technical analysis).

Section \ref{sec: framework} presents the assumptions we need for our analysis---see the works of~\citet{L1-SVM}, \citet{Wainwright-logreg} and \citet{quantile-reg} for related assumptions in the context of $\ell_1$-based classification and quantile regression procedures. In Section~\ref{sec:theory-main-results}, we establish a high probability upper bound on $\| \hat{\B{\beta}} - \B{\beta}^*\|_2^2$.

\subsection{Assumptions}\label{sec: framework}
We first present some assumptions for establishing the error bounds.

\textbf{Loss Function.}
We list our basic assumptions on the loss function. 
\begin{asu} \label{asu1}
	The function $t \mapsto f( t; y)$ is non-negative, convex, and Lipschitz continuous with constant $L$, that is $| f(t_1; y) - f(t_2; y) | \le L | t_1 -t_2 |, \ \forall t_1, t_2$. 
\end{asu}
We let $\partial f(t;y)$ denote a subgradient of $t \mapsto f(t;y)$---i.e., 
$f(t_2; y) - f(t_1; y) \ge \partial f(t_1;y) (t_2 -t_1), \ \forall t_1, t_2$. Note that the hinge loss satisfies Assumption 2 with $L=1$; and has a subgradient given by $\partial f(t;y) = \mathbf{1} \left(  1 -  y t \ge 0\right) y$. The logistic loss function satisfies Assumption 2 with $L=1$; and its subgradient coincides with the gradient.   

\textbf{Differentiability of the Population Risk.}
The following assumption is on the uniqueness of $\B{\beta}^*$ and differentiability of the population risk $\mathcal{L}$. 

\begin{asu} \label{asu2}
	Problem~\eqref{def-beta0} has a unique minimizer. The population risk $\B\beta \mapsto {\mathcal L}(\B\beta)$ is twice continuously differentiable, with gradient $\nabla \mathcal{L}(\B\beta)$ and Hessian $\nabla^2 \mathcal{L}(\B\beta)$.  In particular, the following holds:
	\begin{equation}\label{gradient-relation}
	\nabla \mathcal{L}(\B\beta) = \mathbb{E}\left( \partial f \left( \langle \mathbf{x}, \B\beta \rangle ;  y \right) \mathbf{x} \right).
	\end{equation}
\end{asu}
When $f$ is the hinge loss, \citet{lemma2} discuss conditions under which  Assumption \ref{asu2} holds. In particular, Assumption (A1) in~\citet{lemma2}  requires that the conditional density functions of $\mathbf{x}$ for the two classes are continuous and have finite second moments. Under this assumption, the Hessian $\nabla^2 {\mathcal L}(\B\beta)$ is well defined and continuous in $\B\beta$. Under Assumption (A4)~\citep{lemma2}, the Hessian is positive definite at an optimal solution, therefore~\eqref{def-beta0} has a unique solution. We refer the reader to~\citet{Wainwright-logreg,vdg_linear_models} for discussions pertaining to logistic regression.


\indent {\bf Restricted Eigenvalue Conditions.}
Assumption \ref{asu4} is a \textit{restricted eigenvalue condition} similar to that used in regression problems \citep{lasso-dantzig,stats-HDD}. For an integer $\ell>0$, we assume that the quadratic forms associated with the Hessian matrix $\nabla^2 {\mathcal L}(\B{\beta}^*)$ and the covariance matrix $n^{-1}\mathbf{X}^T\mathbf{X}$ are respectively lower-bounded and upper-bounded on the set of $2\ell$-sparse vectors.

\begin{asu} \label{asu4} Let $\ell>0$ be an integer. Assumption \ref{asu4}$(\ell)$ is said to hold if there exists constants $\kappa(\ell), \lambda(\ell) >0 $ such that almost surely the following holds:
	$$ \kappa(\ell) \le \inf \limits_{ \M{z} \neq 0, \| \mathbf{z}  \|_0 \le 2 \ell  } \left\{ \frac{ \mathbf{z}^T \nabla^2\mathcal{L}(\B{\beta}^*)  \mathbf{z}  }{ \|\mathbf{z}\|_2^2  } \right\}~~~~~~~\text{and}~~~~~\lambda(\ell) \ge \sup \limits_{  \M{z} \neq 0, \| \mathbf{z}  \|_0 \le 2\ell  } \left\{ \frac{ \| \mathbf{X}  \mathbf{z}  \|_2^2  }{ n \|\mathbf{z}\|_2^2  } \right\}.
	$$
\end{asu}
In the rest of this section, we consider Assumption \ref{asu4} with $\ell=k$.
Assumption (A4) in \citet{L1-SVM} for linear SVM is similar to our 
Assumption~\ref{asu4}.  For logistic regression, related assumptions appear in the literature, e.g., Assumptions A1 and A2 in~\citet{Wainwright-logreg} (in the form of a dependency and an incoherence condition on the population Fisher information matrix).

\textbf{Growth condition.}
As $\B{\beta}^*$ minimizes the population risk, we have $\nabla \mathcal{L}(\B{\beta}^*) = 0$. Under the above regularity assumptions and when Assumption \ref{asu4}$(k)$ is satisfied,  the population risk is lower-bounded by a quadratic function in a neighborhood of $\B{\beta}^*$. By continuity, we let $r(k)$ denote the maximal radius for which the following lower bound holds:
\begin{equation}\label{growth-cond-defn}
r(k) = \max \left\{ r>0 \ \Bigg\rvert \ \mathcal{L}(\B{\beta}^* + \mathbf{z} )  \ge \mathcal{L}(\B{\beta}^*) + \frac{1}{4} \kappa(k) \| \mathbf{z} \|_2^2~~~\forall \mathbf{z}~~~\text{s.t.}~~~\| \mathbf{z}  \|_0 \le 2 k , \| \mathbf{z} \|_2 \le r \right\}.
\end{equation}

Below we make an assumption on the growth condition. 
\begin{asu} \label{asu5}
	Let $\delta \in (0,1/2)$. We say that Assumption \ref{asu5}($\delta$) holds if the parameters $n,p,k,R$ satisfy:  $$\frac{24L}{\kappa(k)} \sqrt{\frac{\lambda(k)}{n} \left( k \log\left( Rp/k \right) + \log\left( 1/\delta \right) \right) } < r(k),$$
	and the following holds: $({k}/{n}) \log(p/k) \le 1$ and $7n e \le 3L \sqrt{\lambda(k) } p \log \left( p/k \right)$.
\end{asu}
Assumption~\ref{asu5} is similar to the scaling conditions in \citet{Wainwright-logreg}[Theorem~1] 
for logistic regression. \citet{quantile-reg}~also 
makes use of a growth condition for
their analysis in $\ell_1$-sparse quantile regression problems.  
Note that although Assumption~\ref{asu5} is not required to prove  Theorem~\ref{restricted-strong-convexity}, we will need it to derive the error bound of  Theorem~\ref{main-results}. We now proceed to derive an upper bound on coefficient estimation error.

\subsection{Main Result}\label{sec:theory-main-results}
We first show (in Theorem~\ref{restricted-strong-convexity}) that the loss function $f$ satisfies a form of restricted strong convexity~\citep{M-estimators} around $\B\beta^*$, a minimizer of~\eqref{def-beta0}.

\begin{theorem} \label{restricted-strong-convexity}
Let $\textbf{h} = \hat{\B{\beta}} - \B{\beta}^*$, $\delta \in (0,1/2)$, and $\tau = 6L \sqrt{\frac{\lambda(k)}{n} \left(  k\log\left( Rp/k \right) + \log\left( 1/\delta \right) \right) }$. If Assumptions \ref{asu1}, \ref{asu2}, \ref{asu4}($k$) and \ref{asu5}($\delta$) are satisfied, then with probability  at least $1 - \delta$, the following holds:
\begin{equation}\label{conclude-theorem-restricted-strong-convexity}
\begin{aligned}
\frac{1}{n} \sum_{i=1}^n f \left( \langle \mathbf{x}_i,  \B{\beta}^* + \mathbf{h}  \rangle ;  y_i \right)  - \frac{1}{n} \sum_{i=1}^n  f \left( \langle \mathbf{x}_i,  \B{\beta}^* \rangle ;  y_i \right) \\
~~~~ \ge \frac{1}{4}  \kappa(k) \left\{  \|\mathbf{h}\|_2^2 \wedge  r(k) \|\mathbf{h}\|_2  \right\}  - \tau \| \mathbf{h} \|_2\vee \tau^2.
\end{aligned}
\end{equation}
\end{theorem}
Theorem~\ref{restricted-strong-convexity} is used to derive the error bound presented in Theorem~\ref{main-results}, which is the main result in this section.
\begin{theorem} \label{main-results}
	Let $\delta \in \left(0, 1/2\right)$ and assume that Assumptions \ref{asu1}, \ref{asu2}, \ref{asu4}($k$)  and \ref{asu5}($\delta$) hold. Then the estimator $\hat{\B{\beta}}$ defined as a solution of Problem \eqref{learning-l0} satisfies with probability at least $1-\delta$:
	\begin{equation}\label{error-bound-thm-eqn}
	\| \hat{\B{\beta}} - \B{\beta}^*\|_2^2 \lesssim L^2 \lambda(k)\widetilde{\kappa}^2 \left( \frac{k \log(Rp/k)}{n} + \frac{\log(1/ \delta)}{n} \right)	    
	\end{equation}
	where, $\widetilde{\kappa} = \max \{1/\kappa(k), 1\}$.
\end{theorem}
In~\eqref{error-bound-thm-eqn}, the symbol ``$\lesssim$'' stands for ``$\leq$" up to a universal constant. The proof of Theorem~\ref{main-results} is presented in Appendix \ref{sec: appendix_main-results}. The rate appearing in Theorem~\ref{main-results}, of the order of $k/n \log (p/k)$, is the best known rate for a (sparse) classifier; and this coincides with the optimal scaling in the case of regression. In comparison, \citet{L1-SVM} and \citet{Wainwright-logreg} derived a bound scaling as $k/n \log(p)$ for 
$\ell_1$-regularized SVM (with hinge loss) and logistic regression, respectively. Note that the bound in Theorem~\ref{main-results} holds for any sufficiently small $\delta>0$. Consequently, by integration, we obtain the following result in expectation. 
\begin{cor} \label{main-corollary}
Suppose Assumptions \ref{asu1}, \ref{asu2}, \ref{asu4}($k$) hold true and Assumption \ref{asu5}($\delta$) is true for $\delta$ small enough, then:
	$$ \mathbb{E}  \| \hat{\B{\beta}} - \B{\beta}^* \|_2^2 \lesssim L^2 \lambda(k) \widetilde{\kappa}^{2}  \frac{k \log(Rp/k)}{n}, $$
	where $\widetilde{\kappa}$ is defined in Theorem~\ref{main-results}.
\end{cor}

\section{Experiments} \label{sec:experiments}
In this section, we compare the statistical and computational performance of our proposed algorithms versus the state of the art, on both synthetic and real data sets.
\subsection{Experimental Setup} \label{section:expsetup}
 {\bf{Data Generation}.} For the synthetic data sets, we generate a multivariate Gaussian data matrix $\M{X}_{n \times p} \sim \text{MVN}(\M{0}, \B\Sigma)$ and a sparse vector of coefficients $\B{\beta}^{\dagger}$ with $k^{\dagger}$ nonzero entries, such that $\beta^{\dagger}_i=1$ for $k^{\dagger}$ equi-spaced indices $i \in [p]$. Every coordinate $y_i$ of the outcome vector $\M{y} \in \{ -1, 1 \}^n$ is then sampled independently from a Bernoulli distribution with success probability:
$P(y_{i} = 1 | \M{x}_{i} ) =(1 + \exp ( {-s \langle \B{\beta}^{\dagger} , \M{x}_{i} } \rangle ))^{-1},$
where $\M{x}_i$ denotes the $i$-th row of $\M{X}$, and $s$ is a parameter that controls the signal-to-noise ratio. Specifically, smaller values of $s$ increase the variance in the response ${y}$, and when $s \to \infty$ the generated data becomes linearly separable.

\textbf{Algorithms and Tuning.} 
We compare our proposal, as implemented in  our package \texttt{L0Learn}, with: (i) $\ell_1$-regularized logistic regression (\texttt{glmnet} package \citealp{glmnet}), (ii) MCP-regularized logistic regression (\texttt{ncvreg} package \citealp{ncvreg}), and (iii) two packages for sparsity constrained minimization (based on hard thresholding): \texttt{GraSP} \citep{bahmani2013greedy} and \texttt{NHTP} \citep{zhou2019global}. For \texttt{GraSP} and \texttt{NHTP}, we use  cardinality constrained logistic regression with ridge regularization---that is, we optimize Problem \eqref{eq:constrainedopt} with $\lambda_1 = 0$. Tuning is done on a separate validation set under the fixed design setting, i.e., we use the same features used for training but a new outcome vector $\M{y}'$ (independent of $\M{y}$). 
The tuning parameters are selected so as to minimize the loss function on the validation set (e.g., for regularized logistic regression, we minimize the unregularized negative log-likelihood). 

We use $\ell_0$-$\ell_q$ as a shorthand to denote the penalty $\lambda_0 \| \B\beta \|_0 + \lambda_q \| \B\beta \|_q^q$ for $q \in \{1, 2\}$. 
For all penalties that involve 2 tuning parameters---i.e., $\ell_0$-$\ell_q$ (for $q \in \{1, 2\}$), MCP, \texttt{GraSP}, and \texttt{NHTP}, we sweep the parameters over a two-dimensional grid. For our penalties, we choose $100$ $\lambda_0$ values as described in Section~\ref{section:L0Learn}. For \texttt{GraSP} and \texttt{NHTP}, we sweep the number of nonzeros between $1$ and $100$. For $\ell_0$-$\ell_1$, and we choose a sequence of $10$ $\lambda_1$-values in $[a,b]$, where $a$ corresponds to a zero solution and $b=10^{-4}a$. Similarly, for $\ell_0$-$\ell_2$, \texttt{GraSP}, and \texttt{NHTP}, we choose $10$ $\lambda_2$ values between $10^{-4}$ and $100$ for the experiment in Section~\ref{exp:varysamples}; and between $10^{-8}$ and $10^{-4}$ for that in Section~\ref{exp:largescale}. For MCP, the sequence of $100$ $\lambda$ values is set to the default values selected by $\texttt{ncvreg}$, and we vary the second parameter $\gamma$ over $10$ values between $1.5$ and $25$. For the $\ell_1$-penalty, the grid of 100 $\lambda$ values is set to the default sequence chosen by \texttt{glmnet}. 

\textbf{Performance Measures.} We use the following measures to evaluate the performance of an estimator $\hat{\B{\beta}}$:
\begin{itemize}
\item \textbf{AUC}: The area under the curve of the ROC plot.
\item \textbf{Recovery F1 Score}: This is the F1 score for support recovery, i.e., it is the harmonic mean of precision and recall: F1 Score = $2 P R/(P + R)$, where $P$ is the precision given by $ |\text{Supp}(\hat{\B{\beta}}) \cap \text{Supp}(\B{\beta}^{\dagger})|/ | \text{Supp}(\hat{\B{\beta}})|$, and $R$ is the recall given by $ |\text{Supp}(\hat{\B{\beta}}) \cap \text{Supp}(\B{\beta}^{\dagger})|/ | \text{Supp}(\B{\beta}^{\dagger})|$ . An F1 Score of $1$ implies full support recovery; and a value of zero implies that the supports of the true and estimated coefficients have no overlap. 
\item \textbf{Support Size}: The number of nonzeros in $\hat{\B{\beta}}$.
\item \textbf{False Positives}: This is equal to $| \text{Supp}(\hat{\B{\beta}}) \setminus \text{Supp}(\B{\beta}^{\dagger})|$.
\end{itemize}
\begin{figure}[tb]
\centering
{\small High Correlation Setting ($\Sigma_{ij} = 0.9^{|i-j|}$)}
\includegraphics[scale=0.36]{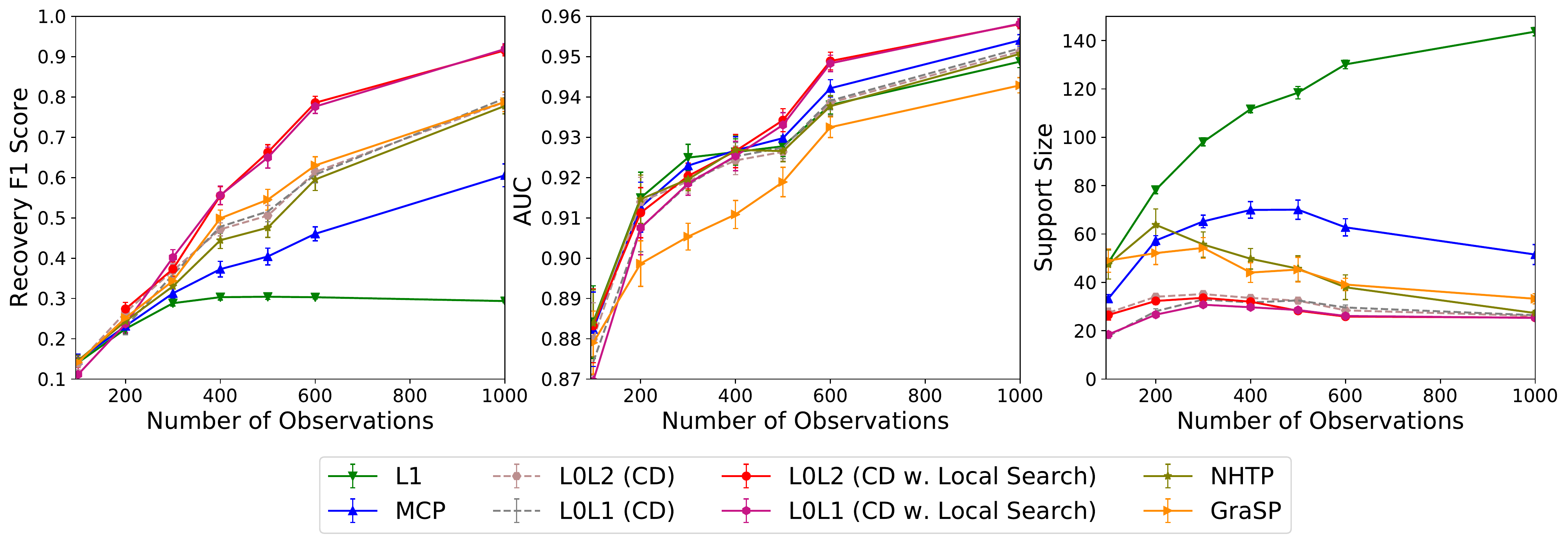}
\caption{{\small{Performance for varying $n\in [100, 10^3]$ and $\Sigma_{ij} = 0.9^{|i-j|}, p=1000, k^{\dagger} = 25, s=1$. In this high-correlation setting, our proposed algorithm (using local search) for the $\ell_0$-$\ell_q$ (for both $q \in \{1, 2\}$) penalized estimators---denoted as L0L2 (CD w. Local Search) and L0L1 (CD w. Local Search) in the figure---seems to outperform state-of-the-art methods (MCP, $\ell_1$, \texttt{GraSP} and \texttt{NHTP}) in terms of both variable selection and prediction. The variable selection performance improvement (higher F1 score and smaller support size) is more notable than AUC.
Local search with CD shows benefits compared to its variants  that do not employ local search, the latter denoted by L0L1 (CD) and L0L2 (CD) in the figure.}}}
\label{fig:highcorr}
\end{figure}

\subsection{Performance for Varying Sample Sizes} \label{exp:varysamples}
In this experiment, we fix $p$ and vary the number of observations $n$ to study its effect on the performance of the different algorithms and penalties. We hypothesize that when the statistical setting is difficult (e.g., features are highly correlated and/or $n$ is small), good optimization algorithms for $\ell_0$-regularized problems lead to estimators that can  significantly outperform estimators obtained from convex regularizers and common (heuristic) algorithms for nonconvex regularizers. To demonstrate our hypothesis, we perform experiments on the following data sets:
\begin{itemize}
\item \textbf{High Correlation}: $\Sigma_{ij} = 0.9^{|i-j|}, p=1000, k^{\dagger} = 25, s=1$
\item \textbf{Medium Correlation}: $\Sigma_{ij} = 0.5^{|i-j|}, p=1000, k^{\dagger} = 25, s=1$
\end{itemize}

\begin{figure}[tb]
\centering
{\small Medium Correlation Setting ($\Sigma_{ij} = 0.9^{|i-j|}$)}
\includegraphics[scale=0.36]{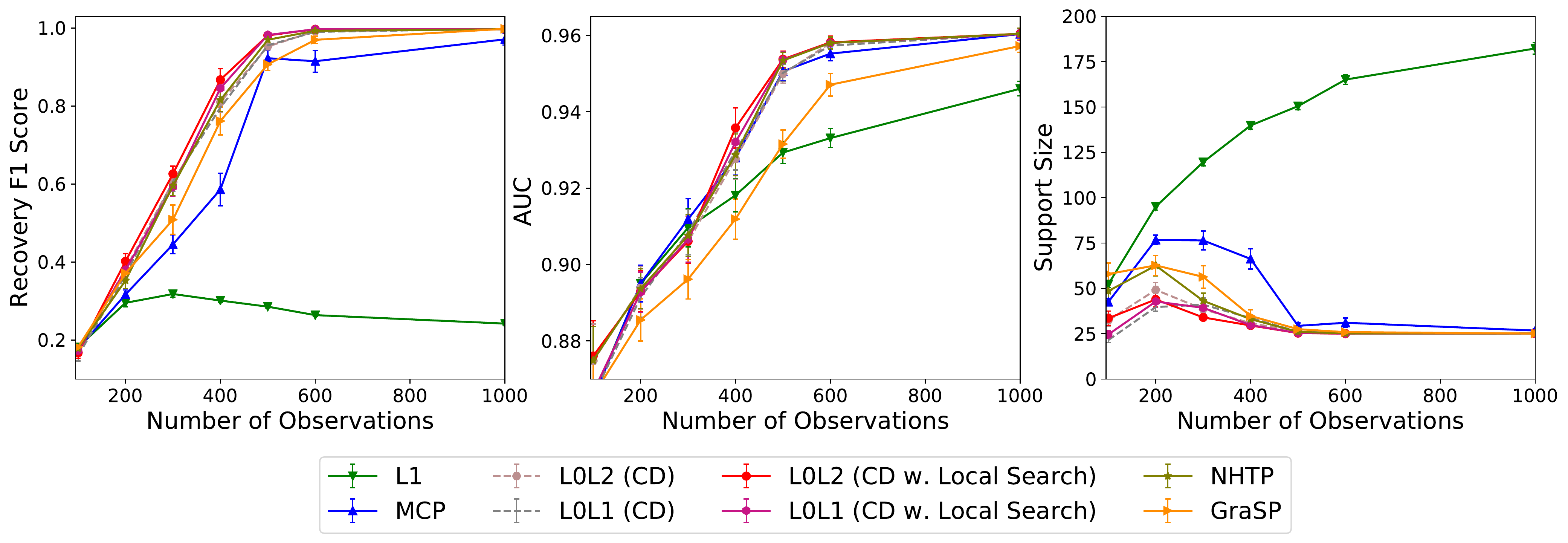}
\caption{Performance for varying $n$ and $\Sigma_{ij} = 0.5^{|i-j|}, p=1000, k^{\dagger} = 25, s=1$. In contrast to Figure \ref{fig:highcorr}, this is a medium-correlation setting. Once again Algorithm~2 (CD with local search) for the $\ell_0$-$\ell_1$ and $\ell_0$-$\ell_2$ penalties seems  to perform quite well in terms of variable selection and prediction performance, though its improvement over the other methods appears to be less prominent compared to the high-correlation setting in Figure~\ref{fig:highcorr}. In this example, local search does not seem to offer much improvement over pure CD (i.e., Algorithm~1).}
\label{fig:mediumcorr}
\end{figure}

For each of the above settings, we use the logistic loss function and 
consider $20$ random repetitions. We report the averaged results for the high correlation setting in Figure \ref{fig:highcorr} and for the medium correlation setting in Figure \ref{fig:mediumcorr}.

Figure \ref{fig:highcorr} shows that in the high correlation setting, Algorithm~2 (with $m=1$) for the $\ell_0$-$\ell_2$ penalty and the $\ell_{0}$-$\ell_{1}$ penalty---denoted by L0L2 (CD w. Local Search) and L0L1 (CD w. Local Search) (respectively) in the figure---achieves the best support recovery and AUC across different sample sizes $n$. 
We note that in this example, the best F1 score falls below $0.8$ for $n$ smaller than $600$---suggesting that none of the algorithms can do full support recovery. 
For larger values of $n$, the difference between Algorithm~2 and others become more pronounced in terms of F1 score, suggesting an important edge in terms of variable selection performance.  Moreover, our algorithms select the smallest support size (i.e., the most parsimonious model) for all $n$. In contrast, the $\ell_1$ penalty (i.e., L1) selects significantly larger support sizes (exceeding 100 in some cases) and suffers in terms of support recovery. It is also worth mentioning that the other $\ell_0$-based algorithms, i.e.,  Algorithm 1, {\texttt{NHTP}} and {\texttt{GraSP}}, outperform MCP and L1 in terms of support recovery (F1 score) and support sizes. The performance of {\texttt{NHTP}} and {\texttt{GraSP}} appears to be similar in terms of support sizes and F1 score, though {\texttt{NHTP}} appears to be performing better compared to {\texttt{GraSP}} in terms of AUC.

In Figure~\ref{fig:mediumcorr}, for the medium correlation setting, we see that Algorithms 1 and 2 perform similarly: local search (with CD) performs similar to CD (without local search)---Algorithm~1 can recover the correct support with around $500$ observations. In this case, the performance of MCP becomes similar to the $\ell_0$-$\ell_{q}$ penalized estimators in terms of AUC, though there are notable differences in terms of variable selection properties. Compared to Figure~\ref{fig:highcorr}, we observe that $\ell_1$ performs better in this example, but still appears to be generally outperformed by other algorithms in terms of all measures. We also observe that \texttt{NHTP}'s performance is comparable to the best methods in terms of F1 score and AUC. However, it results in models that are more dense compared to L0L1 and L0L2 (both Algorithms~1 and 2). That being said, we will see in Section~\ref{sec:timings} that in terms of running time, Algorithm~1 has a significant edge over \texttt{NHTP}. Furthermore, for larger problem instances (Section~\ref{exp:largescale}) the performance of \texttt{NHTP} suffers in terms of variable selection properties compared to the methods we propose in this paper. \texttt{NHTP} appears to perform better than \texttt{GraSP} across all metrics in this setting.

Figures~\ref{fig:highcorr} and~\ref{fig:mediumcorr} suggest that when the correlations are high, local search with $\ell_0$ can notably outperform competing methods (especially, in terms of variable selection).  When the correlations are medium to low, the differences across the different nonconvex methods become less pronounced. The performance of nonconvex penalized estimators (especially, the $\ell_0$-based estimators) is generally better than $\ell_1$-based estimators in these examples.

\subsection{Performance on Larger Instances} \label{exp:largescale}
We now study the performance of the different algorithms 
for some large values of $p$ under the following settings:
\begin{itemize}
\item \textbf{Setting 1}: $\B\Sigma = \M{I}, n=1000, p=50,000, s=1000$, and $k^{\dagger} = 30$
\item \textbf{Setting 2}: $\Sigma_{ij} = 0.3$ for $i \neq j$, $\Sigma_{ii} = 1$, $n=1000, p=10^5, s=1000$, and $k^{\dagger} = 20$.
\end{itemize}
Table~\ref{table:largeexp} reports the results available from Algorithms~1 and 2 ($m=1$) versus the other methods, for different loss functions. 
\begin{table}[tb]
\centering
\hspace{3.5cm} Setting 1 \hspace{5.5cm} Setting 2 \\
\begin{tabular}{ccc}
\begin{tabular}{lcc|}
\toprule
{Penalty/Loss} &     FP &  $\| \hat{\B{\beta}} \|_0$ \\
\midrule
$\ell_0$-$\ell_2$/Logistic (Algorithm 1)    &    $0.0 \pm 0.0$ &      $ 30.0 \pm 0.0 $ \\
$\ell_0$-$\ell_1$/Logistic (Algorithm 1)    &     $0.0 \pm 0.0$ &      $ 30.0 \pm 0.0 $ \\ 
$\ell_0$-$\ell_2$/Logistic (Algorithm 2)    &    $0.0 \pm 0.0$ &      $ 30.0 \pm 0.0 $ \\ 
$\ell_0$-$\ell_1$/Logistic (Algorithm 2)    &    $0.0 \pm 0.0$ &      $ 30.0 \pm 0.0 $ \\ 
$\ell_0$-$\ell_2$/Sq. Hinge (Algorithm 1) &    $ 2.8 \pm 0.3 $ &      $ 32.8 \pm 0.3 $ \\
$\ell_1$/Logistic        &  $ 617.2 \pm 8.3 $ &     $ 647.2 \pm 8.3 $ \\
MCP/Logistic       &    $0.0 \pm 0.0$ &      $ 30.0 \pm 0.0 $ \\
NHTP/Logistic           &   $ 44.7 \pm 5.6 $ &      $ 74.7 \pm 5.6 $ \\
GraSP/Logistic             &   $ 50.5 \pm 5.6 $ &      $ 80.5 \pm 5.6 $ \\
\bottomrule
\end{tabular}
\begin{tabular}{cc}
\toprule
     FP &  $\| \hat{\B{\beta}} \|_0$ \\
\midrule
   $ 21.6 \pm 0.9 $ &      $ 26.2 \pm 0.5 $ \\
   $ 11.2 \pm 0.7 $  & $ 14.8 \pm 0.3 $ \\
   $11.5 \pm 0.4$  & $ 14.6 \pm 0.3 $ \\
   $ 11.2 \pm 0.4$  & $ 14.5 \pm 0.3 $ \\
   $ 23.9 \pm 0.7 $ &      $ 27.0 \pm 0.4 $ \\
  $ 242.2 \pm 4.8 $ &     $ 256.9 \pm 5.3 $ \\
   $ 80.1 \pm 10.9 $ &      $ 91.5 \pm 12.1 $ \\
   $ 92.5 \pm 0.6$ &      $ 100.0 \pm 0.0 $ \\
    $ 45.3 \pm 3.0 $ &     $ 54.4 \pm 2.8 $ \\
\bottomrule
\end{tabular}
\end{tabular}
\caption{Variable selection performance for different penalty and loss combinations, under high-dimensional settings. FP refers to the number of false positives. We consider ten repetitions and report the averages (and standard errors) across the repetitions.}
\label{table:largeexp}
\end{table}
In Settings 1 and 2 above, all methods achieve an AUC of 1 (approximately), and the main differences across the methods lie in variable selection. For Setting 1, both Algorithms~1 and 2 applied to the logistic loss; and \texttt{ncvreg} for the MCP penalty (with logistic loss) correctly recover the support. In contrast, $\ell_1$ captures a large number of false positives, leading to large support sizes. Both \texttt{NHTP} and \texttt{GraSP} have a considerable number of false positives and result in large support sizes.

In Setting 2, none of the algorithms correctly recovered the support. This setting is more difficult than Setting 1 as it has higher correlation and a larger $p$. In Setting 2, we observe that CD with local search can have an edge over plain CD (see, logistic loss with $\ell_0$-$\ell_2$ penalty). In terms of small FP and compactness of model size, CD with local search appears to be the winner, with Algorithm~1 being quite close. Both Algorithms~1 and 2 appear to  work better compared to earlier methods in this setting. 
We note that in Setting 2, our proposed algorithms select supports with roughly 3 times fewer nonzeros than the MCP penalized problem, and 10 times fewer nonzeros than those delivered by the $\ell_1$-penalized problem. For both settings, our proposed methods offer important improvements over {\texttt{NHTP}} and {\texttt{GraSP}} as well.  

\subsection{Performance on Real Data Sets}\label{sec:real-datasets}
We compare the performance of $\ell_1$ with $\ell_0$-$\ell_q$  ($q \in \{1,2\}$) regularization, using the logistic loss function.\footnote{We tried MCP regularization using \texttt{ncvreg}, however it ran into convergence issues and did not terminate in over an hour. We also tried NHTP, but it ran into convergence issues and took more than $2$ hours to solve for a single solution in the regularization path. Also, note that based on our experiment in Section \ref{sec:timings}, GraSP is slower than NHTP and cannot handle such high-dimensional problems. Therefore, we do not include MCP, NHTP and GraSP in this   experiment.} We consider the following three binary classification data sets taken from the NIPS 2003 Feature Selection Challenge \citep{nipschallenge}:
\begin{itemize}
\item \textbf{Arcene}: This data set is used to identify cancer vs non-cancerous patterns in mass-spectrometric data. The data matrix is dense with $p = 10,000$ features. We used 140 observations for training and 40 observations for testing. 
\item \textbf{Dorothea}: This data set is used to distinguish active chemical compounds in a drug. The data matrix is sparse with $p=100,000$ features. We used 805 observations for training and 230 observations for testing. 
\item \textbf{Dexter:} The task here is to identify text documents discussing corporate acquisitions. The data matrix is sparse with $p = 20,000$ features. We used 420 observations for training and 120 observations for testing. 
\end{itemize}
We obtained regularization paths for $\ell_1$ using \texttt{glmnet} and for $\ell_0$-$\ell_2$ and $\ell_0$-$\ell_1$ using both Algorithm 1 and Algorithm 2 (with $m=1$). In Figure~\ref{fig:aucvssupp}, we plot the support size versus the test AUC for $\ell_1$ and Algorithm 2 (with $m=1$) using $\ell_0$-$\ell_2$. 
To avoid overcrowded plots, the results for our other algorithms and penalties are presented in Appendix~\ref{appendix-sec:expts}.
\begin{figure}[tb] 
\centering
\includegraphics[scale=0.5]{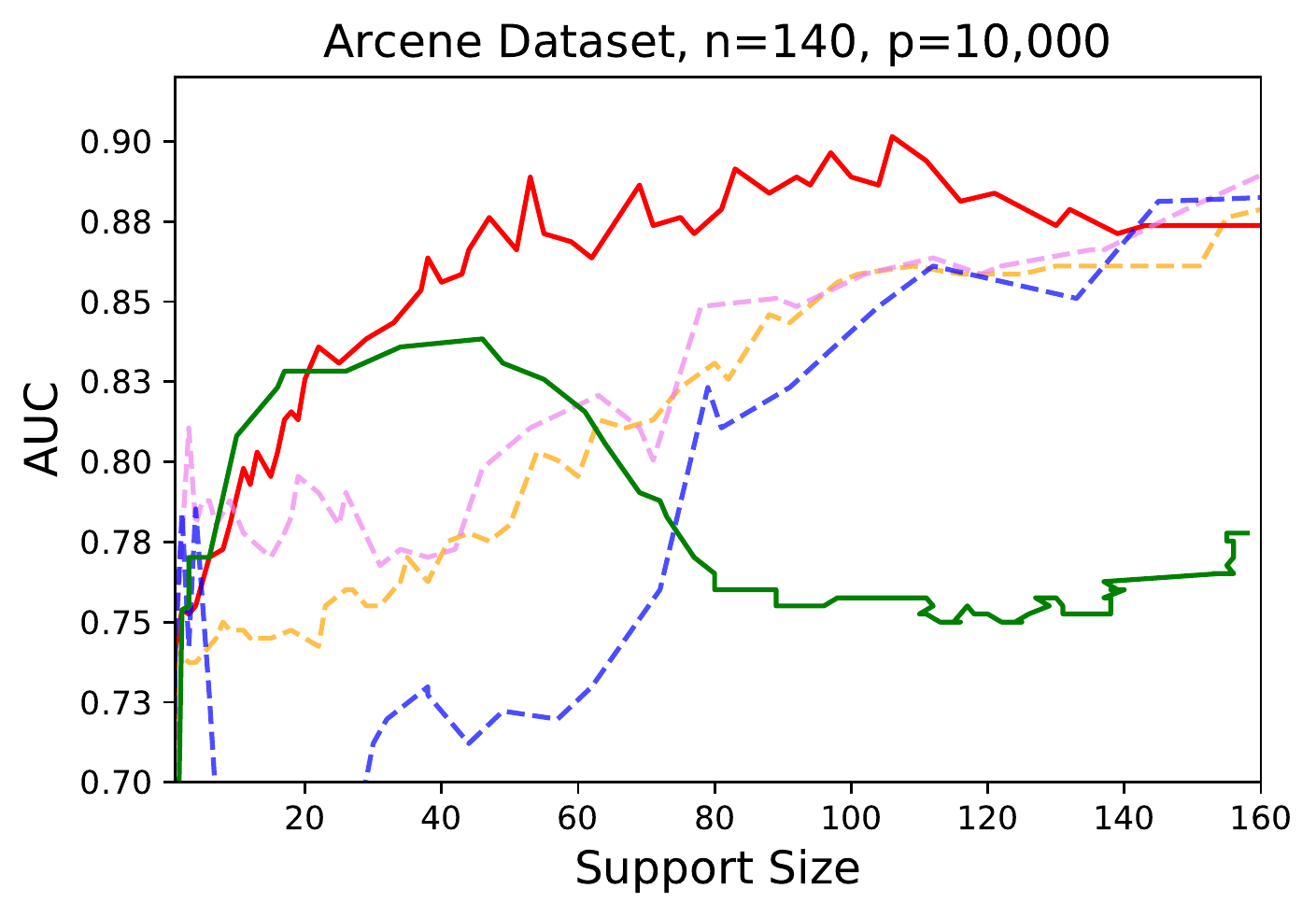}
\includegraphics[scale=0.5]{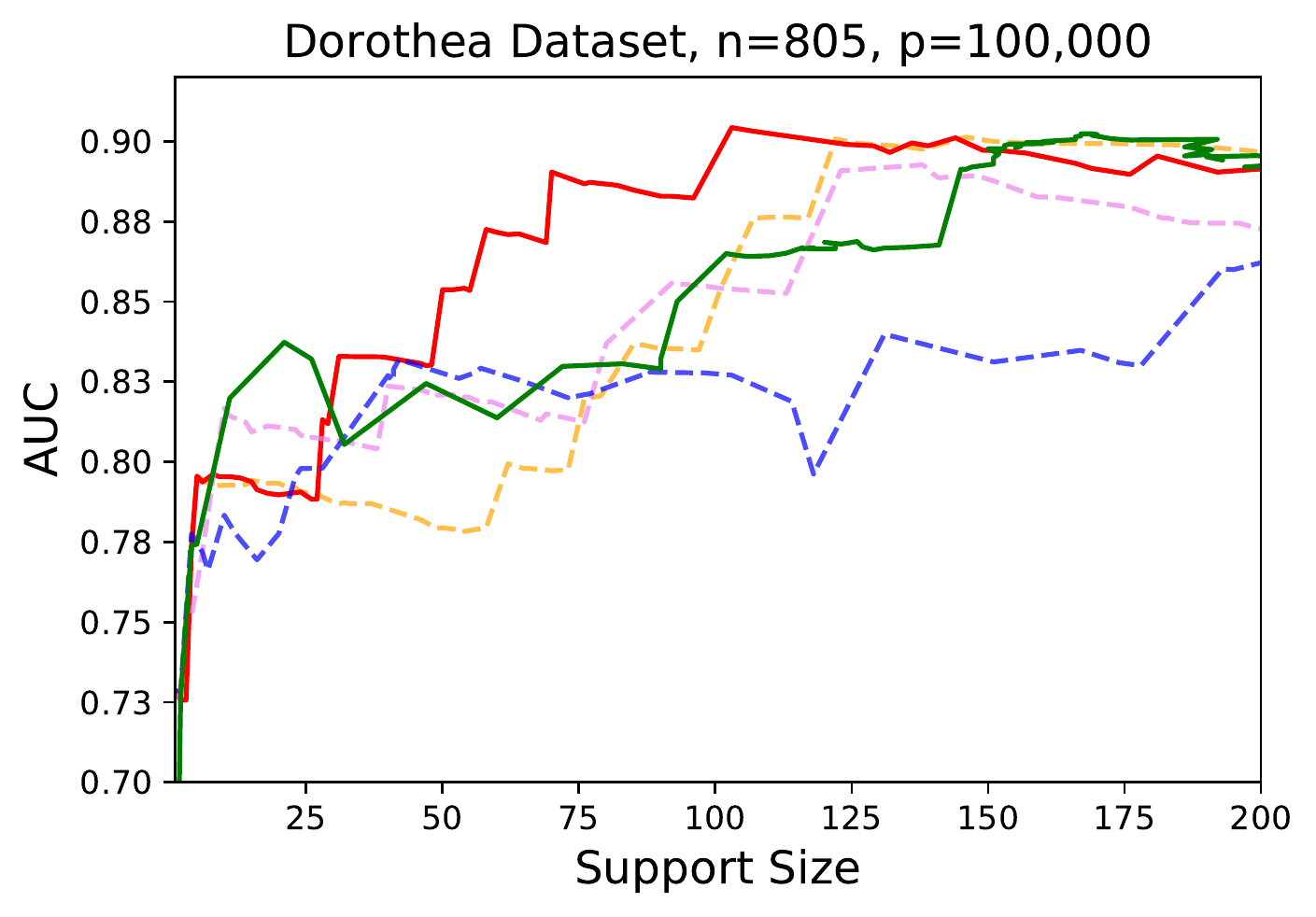}
\includegraphics[scale=0.5]{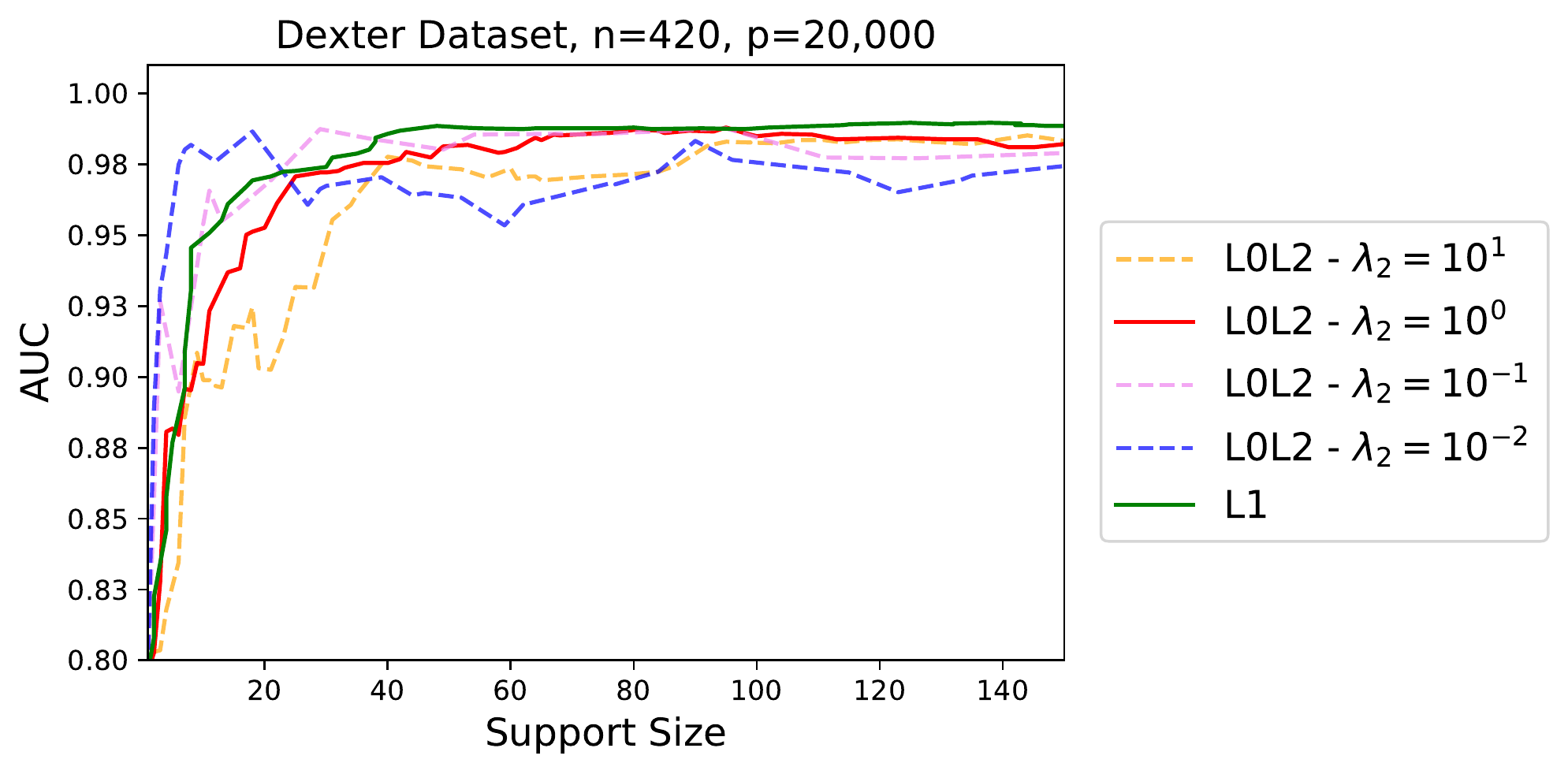}
\caption{Plots of the AUC versus the support size for the Arcene, Dorothea, and Dexter data sets. The green curves correspond to logistic regression with $\ell_1$ regularization. The other curves correspond to logistic regression with $\ell_0$-$\ell_2$ regularization (using Algorithm~2 with $m=1$) for different values of $\lambda_2$ (see legend). }
\label{fig:aucvssupp}
\end{figure}
For the Arcene data set, $\ell_0$-$\ell_2$ with $\lambda_2 = 1$ (i.e., the red curve) outperforms $\ell_1$ (the green curve) for most of the support sizes, and it reaches a peak AUC of 0.9 whereas $\ell_1$ does not exceed 0.84. The other choices of $\lambda_2$ also achieve a higher peak AUC than $\ell_1$ but do not uniformly outperform it. A similar pattern occurs for the Dorothea data set, where $\ell_0$-$\ell_2$ with $\lambda_2 = 1$ achieves a peak AUC of 0.9 at around 100 features, whereas $\ell_1$ needs around 160 features to achieve the same peak AUC. For the Dexter data set, $\ell_0$-$\ell_2$ with $\lambda_2 = 10^{-2}$ (the blue curve) achieves a peak AUC of around 0.98 using less than 10 features, whereas $\ell_1$ requires around 40 features to achieve a similar AUC. In this case, larger choices of $\lambda_2$ do not lead to any significant gains in AUC. This phenomenon is probably due to a higher signal in this data set compared to the previous two. Overall, we conclude that for all the three data sets, Algorithm~2 for $\ell_0$-$\ell_2$ can achieve higher AUC-values with (much) smaller support sizes.

\subsection{Timings}\label{sec:timings}
In this section, we compare the running time of our algorithms versus several state-of-the-art algorithms for sparse classification. 
\subsubsection{Obtaining good solutions: Upper Bounds}\label{sec:good-upperbounds}

We study the running time of fitting sparse logistic regression models, using the following packages: our package \texttt{L0Learn}, \texttt{glmnet}, \texttt{ncvreg}, and \texttt{NHTP}.\footnote{We also considered \texttt{GraSP}, but we found it to be several orders of magnitude slower than the other toolkits, e.g., it takes 3701 seconds at $p=10^5$. All the other toolkits require less than $70$ seconds at $p=10^5$. Therefore, we did not include \texttt{GraSP} in our timing experiments.} In \texttt{L0Learn}, we consider both Algorithm~1 (denoted by L0Learn 1) and Algorithm~2 with $m=1$ (denoted by L0Learn 2). We generate synthetic data as described in Section \ref{section:expsetup}, with $n=1000, s=1, \B\Sigma=I$, $k^\dagger=5$, and we vary $p$. For all toolkits except \texttt{NHTP},\footnote{\texttt{NHTP} does not have a parameter to control the tolerance for convergence.} we set the convergence threshold to $10^{-6}$ and solve for a regularization path with 100 solutions. For \texttt{L0Learn}, we use a grid of $\lambda_0$-values varying between $\lambda_0^{\text{max}}$ (the value which sets all coefficients to zero) and $0.001 \lambda_0^{\text{max}}$. For \texttt{L0Learn} and \texttt{NHTP}, we set $\lambda_2 = 10^{-7}$. For \texttt{glmnet} and \texttt{ncvreg}, we compute a path using the default choices of tuning parameters. The experiments were carried out on a machine with a 6-core Intel Core i7-8750H processor and 16GB of RAM, running macOS 10.15.7, R 4.0.3 (with vecLib's BLAS implementation), and MATLAB R2020b. The running times are reported in Figure~\ref{fig:Rruntimes}.

Figure~\ref{fig:Rruntimes} indicates that \texttt{L0Learn}~1 (Algorithm 1) and \texttt{glmnet} achieve the fastest runtimes across the $p$ range. For $p > 40,000$, \texttt{L0Learn}~1 (Algorithm 1) runs slightly faster \texttt{glmnet}. \texttt{L0Learn}~2 (Algorithm 2) is slower due to the local search, but it can obtain solutions in reasonable times, e.g., less than a minute for 100 solutions at $p = 10^5$. It is important to note that the toolkits are optimizing for different objective functions, and the speed-ups in \texttt{L0Learn} are partly due to the nature of $\ell_0$-regularization which selects fewer nonzeros compared to those available from $\ell_1$ or MCP regularization.
\begin{figure}[tb] 
    \centering
    \includegraphics[scale=0.6]{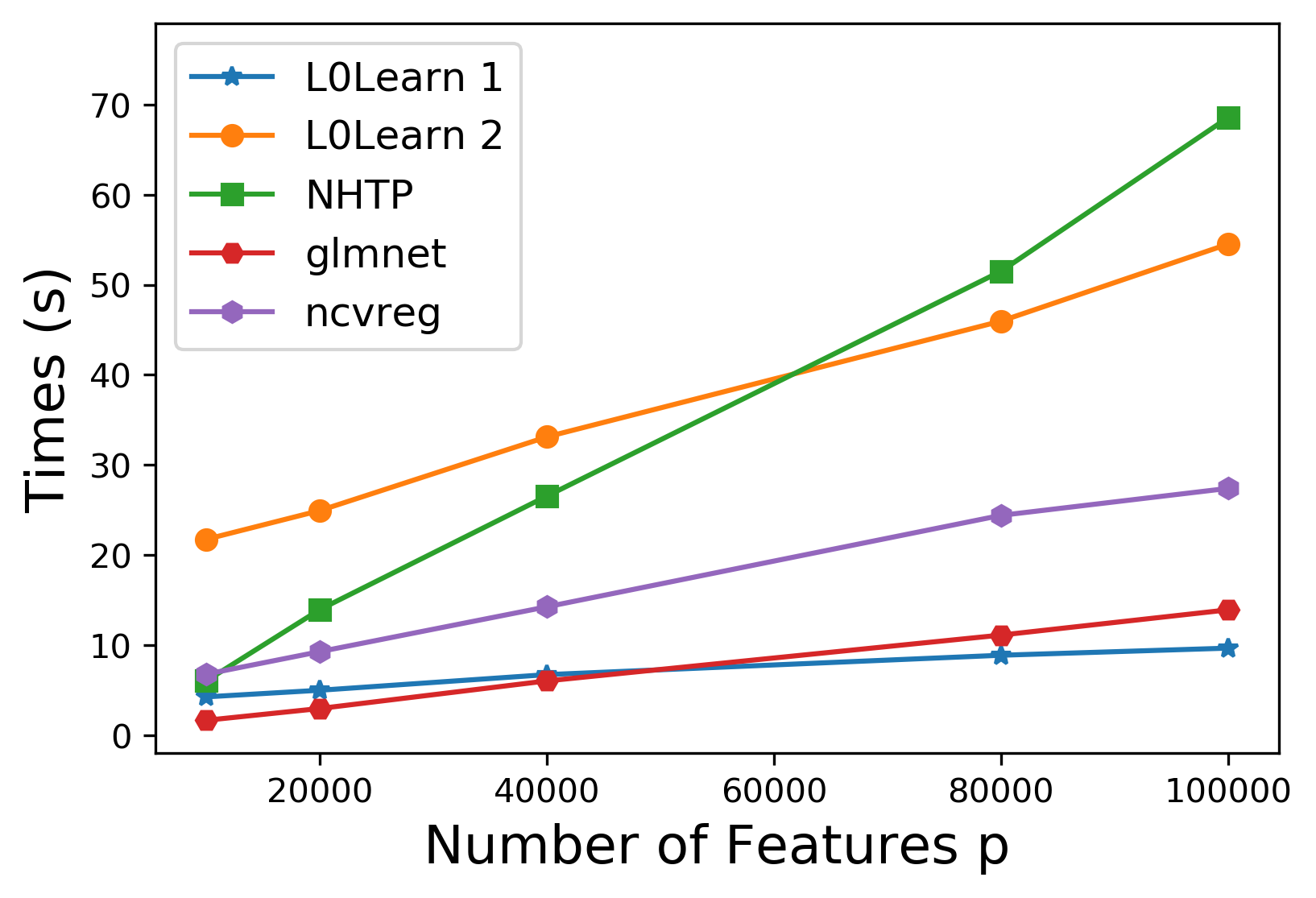}
    \caption{Runtimes to obtain good feasible solutions: Time (s) for obtaining a regularization path (with 100 solutions) for different values of $p$ (as discussed in Section~\ref{sec:good-upperbounds}). L0Learn 1 runs Algorithm 1 (CD), while L0Learn 2 runs Algorithm 2 (CD with Local Search).}
    \label{fig:Rruntimes}
\end{figure}
\subsubsection{Timings for MIP Algorithms: Global Optimality Certificates}
We compare the running time to solve Problem~\eqref{eq:MIP} by Algorithm 4 (IGA) versus solving~\eqref{eq:MIP} directly i.e., without using IGA. We consider the hinge loss and take $q=2$. We use Gurobi's MIP solver for our experiments. We set $\B\Sigma=\M{I}$ and $k^{\dagger}=5$ as described in Section \ref{section:expsetup}. We then generate $\epsilon_{i} \stackrel{\text{iid}}{\sim} N(0, \sigma^2)$ and set $y_i = \sign(\M{x}_{i}'\B\beta^{\dagger} + \epsilon_i)$, where the signal-to-noise ratio SNR=$\text{Var}(\M{X} \B\beta^{\dagger})/\sigma^2~=~10$. We set $n=1000$ and vary $p\in \{10^4, 2\times10^4, 3\times10^4, 4\times10^4, 5\times10^4\}$. 
For a fixed $\lambda_2 = 10$, $\lambda_0$ is chosen such that the final (optimal) solution has 5 nonzeros. The parameter $\mathcal{M}$ is set to $1.2 \|\tilde{\B\beta}\|_{\infty}$, where $\tilde{\B\beta}$ is the warm start obtained from Algorithm~2. We note that in all cases, the optimal solution recovers the support of the true solution $\B{\beta}^{\dagger}$. 

For this experiment, we use a machine with a 12-core Intel Xeon E5 @ 2.7 GHz and 64GB of RAM, running OSX 10.13.6 and Gurobi v8.1. The timings are reported in Table~\ref{table:miptimings}. 
\begin{table}[tb]
\centering
\begin{tabular}{@{}cccccc@{}}
\toprule
Tookit/p & 10000 & 20000 & 30000 & 40000 & 50000 \\ \midrule
IGA (this paper)    & 35    & 80   & 117   & 169   & 297*   \\
Gurobi   & 4446  & 21817 &  -     &  -     &  -     \\ \bottomrule
\end{tabular}
\caption{Runtimes to certify global optimality: Time(s) for solving an $\ell_0$-regularized problem with a hinge loss function, to optimality. ``-'' denotes that the algorithm does not terminate in a day. ``*'' indicates that the algorithm is terminated with a 0.05\% optimality gap.}
\label{table:miptimings}
\end{table}
The results indicate significant speed-ups. For example, for $p=20,000$ our algorithm terminates in 80 seconds whereas Gurobi takes 6 hours to solve~\eqref{eq:MIP} (to global optimality). For larger values of $p$, Gurobi cannot terminate in a day, while our algorithm can terminate to optimality in few minutes. It is worth noting that, empirically we observe IGA can attain low running times when sparse solutions are desired ($< 30$ nonzeros in our experience). This observation also applies to the state-of-the-art MIP solvers for sparse regression, e.g., see \citet{bestsubset, bertsimas2017sparse, hazimeh2020sparse}. For denser solutions, IGA can still be useful for obtaining optimality gaps, but certifying optimality with a very small optimality gap is expected to take longer.

\section{Conclusion}
We considered the problem of linear classification regularized with a combination of the $\ell_0$ and $\ell_q$ (for $q \in \{1,2\}$) penalties. We developed both approximate and exact algorithms for this problem. Our approximate algorithms are based on coordinate descent and local combinatorial search. We established convergence guarantees for these algorithms and demonstrated empirically that they can run in times comparable to the fast $\ell_1$-based solvers. Our exact algorithm can solve to optimality high-dimensional instances with $p \approx 50,000$. This scalability is achieved through the novel idea of integrality generation, which solves a sequence of mixed integer programs with a small number of binary variables, until converging to a globally optimal solution. We also established new estimation error bounds for a class of $\ell_0$-regularized classification problems and showed that these bounds compare favorably with the best known bounds for $\ell_1$ regularization. We carried out experiments on both synthetic and real datasets with $p$ up to $10^5$. The results demonstrate that our $\ell_0$-based combinatorial algorithms can have a significant statistical edge (in terms of variable selection and prediction) compared to state-of-the-art methods for sparse classification, such as those based on $\ell_1$ regularization or simple greedy procedures for $\ell_0$ regularization.

There are multiple promising directions for future work. From a modeling perspective, our work can be generalized to structured sparsity problems based on $\ell_0$ regularization \citep{hazimeh2021grouped}. Recent work shows that specialized BnB solvers can be highly scalable for $\ell_0$-regularized regression \citep{hazimeh2020sparse}. One promising direction is to scale our proposed integrality generation algorithm further by developing specialized BnB solvers for the corresponding MIP subproblems.

\acks{\sloppy We would like to thank the anonymous reviewers for their comments that helped improve the paper. 
The authors acknowledge research funding from the Office of Naval Research [Grants ONR-N000141512342 and ONR-N000141812298 (Young Investigator Award)], and the National Science Foundation [Grant NSF-IIS-1718258].}


\appendix
\section{Proofs and Technical Details}
\subsection{Proof of Theorem \ref{theorem:cdconvergence}}
We first present Lemma~\ref{lemma:cddecrease}, which establishes that after updating a single coordinate, there is a sufficient decrease in the objective function. The result of this Lemma will be used in the proof of the theorem.
\begin{lemma} \label{lemma:cddecrease}
Let $\{ \beta^{l} \}$ be the sequence of iterates generated by Algorithm 1. Then, the following holds for any $l$ and $i = 1 + (l\mod p)$:
\begin{equation}
P(\B\beta^{l}) - P(\B{\beta}^{l+1}  ) \geq \frac{\hat{L}_i - L_i}{2} ({\beta}^{l+1}_i - {\beta}^{l}_i)^2.
\end{equation}
\end{lemma}
\begin{proof}
Consider some $l \geq 0$ and fix $i = 1 + (l \mod p)$ (this corresponds to the coordinate being updated via CD). 
Applying~\eqref{eq:blockdescent} to $\B{\beta}^{l}$ and $\B{\beta}^{l+1}$ and adding $\psi(\B{\beta}^{l+1})$ to both sides, we have:
\begin{align*}
P(\B{\beta}^{l+1}) & \leq g(\B{\beta}^{l}) + ({\beta}^{l+1}_i - {\beta}^{l}_i) \nabla_i g(\B{\beta}^{l}) + \frac{L_i}{2} ({\beta}^{l+1}_i - {\beta}^{l}_i)^2 + \psi(\B{\beta}^{l+1}).
\end{align*}
By writing $\frac{L_i}{2} ({\beta}^{l+1}_i - {\beta}^{l}_i)^2$ as the sum of two terms: $\frac{L_i - \hat{L}_i}{2} ({\beta}^{l+1}_i - {\beta}^{l}_i)^2  + \frac{\hat{L}_i}{2} ({\beta}^{l+1}_i - {\beta}^{l}_i)^2$ and regrouping the terms we get:
\begin{align} \label{eq:Gupperbd}
P(\B{\beta}^{l+1}) \leq \widetilde{P}_{\hat{L}_i}(\B{\beta}^{l+1}; \B{\beta}^{l}) + \frac{L_i - \hat{L}_i}{2} ({\beta}^{l+1}_i -{\beta}^{l}_i)^2,
\end{align}
where, in the above equation, we used the definition of $\widetilde{P}$ from~\eqref{eq:Ptilde}. Recall from the definition of Algorithm 1 that ${\beta}^{l+1}_i \in \argmin_{\beta_i} \widetilde{P}_{\hat{L}_i}({\beta}^{l}_1, \dots, \beta_i, , \dots, {\beta}^{l}_p; \B{\beta}^{l})$ which implies that $\widetilde{P}_{\hat{L}_i}(\B{\beta}^{l}; \B{\beta}^{l}) \geq \widetilde{P}_{\hat{L}_i}(\B{\beta}^{l+1}; \B{\beta}^{l})$. But $ \widetilde{P}_{\hat{L}_i}(\B{\beta}^{l}; \B{\beta}^{l}) = P(\B{\beta}^{l})$ (this follows directly from \eqref{eq:Ptilde}). Therefore, we have $P(\B{\beta}^{l}) \geq \widetilde{P}_{\hat{L}_i}(\B{\beta}^{l+1}; \B{\beta}^{l})$. Plugging the latter inequality into \eqref{eq:Gupperbd} and rearranging the terms, we arrive to the result of the lemma.  
\end{proof} 

\smallskip

\noindent Next, we present the proof of Theorem \ref{theorem:cdconvergence}.

\begin{proof} 
\begin{itemize}
\item \textbf{Part 1}: We will show that $\text{Supp}(\B{\beta}^{l}) \neq \text{Supp}(\B{\beta}^{l+1})$ cannot happen infinitely often. Suppose for some $l$ we have $\text{Supp}(\B{\beta}^{l}) \neq \text{Supp}(\B{\beta}^{l+1})$. Then, for $i = 1 + (l \mod p)$, one of the following two cases must hold: (I) ${\beta}^{l}_i = 0 \neq {\beta}^{l+1}_i$ or (II) ${\beta}^{l}_i \neq 0 = {\beta}^{l+1}_i$. Let us consider case (I). 
Recall that the minimization step in Algorithm 1 is done using the thresholding operator defined in \eqref{eq:thresholding}. Since ${\beta}^{l+1}_i \neq 0$, \eqref{eq:thresholding} implies that $|{\beta}^{l+1}_i| \geq \sqrt{\frac{2 \lambda_0}{\hat{L}_i + 2 \lambda_2}}$. Plugging the latter bound into the result of Lemma \ref{lemma:cddecrease}, we arrive to 
\begin{equation} \label{eq:constdec}
P(\B{\beta}^{l}) - P(\B{\beta}^{l+1}) \geq \frac{\hat{L}_i - L_i}{\hat{L}_i + 2 \lambda_2} \lambda_0,
\end{equation}
where the r.h.s. of the above is positive, due to the choice  $\hat{L}_i > L_i$. The same argument can be used for case (II) to arrive to \eqref{eq:constdec}. Thus, whenever the support changes, \eqref{eq:constdec} applies, and consequently the objective function decreases by a constant value. Therefore, the support cannot change infinitely often (as $P(\B{\beta})$ is bounded below).

\item \textbf{Part 2}: 
First, we will show that under Assumption~\ref{assumption:strongconvexity}, the function $\B\beta_{S} \mapsto G(\B{\beta}_S)$ is strongly convex. 
This holds for the first case of Assumption~\ref{assumption:strongconvexity}, i.e., when $\lambda_2 > 0$. Next, we will consider the second case of  Assumption~\ref{assumption:strongconvexity} which states that $ P(\B{\beta}^{0}) < \lambda_0 u$ and $g(\B{\beta})$ is strongly convex when restricted to a support of size at most $u$. Since Algorithm 1 is a descent algorithm, we get $P(\B{\beta}^{l}) < \lambda_0 u$ for any $l \geq 0$. This implies $\| \B{\beta}^{l} \|_0 < u$ for any $l \geq 0$, and thus $|S| < u$. Consequently,  the function $\B\beta_{S} \mapsto g(\B{\beta}_S)$ is strongly convex (and so is $\B\beta_{S} \mapsto G(\B{\beta}_S)$). 

After support stabilization (from Part 1), the iterates generated by Algorithm 1 are the same as those generated by CD for minimizing the strongly convex function $G(\B{\beta}_S)$, which is guaranteed to converge (e.g., see \citealt{bertsekas2016nonlinear}). Therefore, we conclude that $\{ \B{\beta}^{l} \}$ converges to a stationary solution $\B{\beta}^{*}$ satisfying $\B{\beta}^{*}_S \in \argmin_{\B{\beta}_S} G(\B{\beta}_S)$ and $\B{\beta}^{*}_{S^c} = \B{0}$.

Finally, we will prove the two inequalities in \eqref{eq:CWminima}. Fix some $i \in S$. Then, from the definition of the thresholding operator in \eqref{eq:thresholding}, the following holds for every $l > N$ (i.e., after support stabilization) at which coordinate $i$ is updated:
\begin{equation}
|{\beta}^{l}_i | \geq \sqrt{\frac{2 \lambda_0}{\hat{L}_i + 2 \lambda_2}}.
\end{equation}
Taking the limit as $l \to \infty$ we arrive to the first inequality in  \eqref{eq:CWminima}, which also implies that $\text{Supp}(\B{\beta}^{*}) = S$. Now let us fix some $i \in S^c$. For every $l > N$ at which coordinate $i$ is updated, we have ${\beta}^{l}_i = 0$ and the following condition holds by the definition of the thresholding operator in \eqref{eq:thresholding}:
\begin{equation}
\frac{\hat{L}_i}{\hat{L}_i + 2\lambda_2}  \Big(  \Big|\frac{\nabla_i g(\B{\beta}^l )}{\hat{L}_i } \Big| - \frac{\lambda_1}{\hat{L}_i} \Big) < \sqrt{\frac{2 \lambda_0}{\hat{L}_i + 2 \lambda_2}}. 
\end{equation}
Taking the limit as $l \to \infty$ in the above and simplifying, we arrive to the second inequality in \eqref{eq:CWminima}.

\item \textbf{Part 3}: The support stabilization on $S$ (from Part 1) and the fact that $\B{\beta}^l \to \B{\beta}^{*}$ where $\text{Supp}(\B{\beta}^{*}) = S$ (from Part 2), directly imply that $\sign(\B{\beta}^l_S)$ stabilizes in a finite number of iterations. That is, there exists an integer $N'$ and a vector $\B{t} \in \{-1,1\}^{|S|}$ such that $\sign(\B{\beta}^l_S) = \B{t}$ for all $l \geq N'$. Therefore, for $l \geq N'$, the iterates of Algorithm 1 are the same as those generated by running coordinate descent to minimize the continuously differentiable function:
\begin{equation} \label{eq:aftersign}
g(\B{\beta})  + \lambda_0 | S | + \lambda_1 (  \sum_{i:t_i > 0 } \beta_i -  \sum_{i:t_i < 0 } \beta_i ) + \lambda_2 \| \B{\beta} \|_2^2.
\end{equation}
\citet{BeckConvergence} established a linear rate of convergence 
for the case of CD applied to a class of strongly convex and continuously differentiable functions, which includes~\ref{eq:aftersign}. Applying \citet{BeckConvergence}'s result to \eqref{eq:aftersign} leads to \eqref{eq:cdrate}.
\end{itemize}
\end{proof}

\subsection{Proof of Theorem \ref{theorem:localsearch}}
\begin{proof}
First, we will show that the algorithm terminates in a finite number of iterations. Suppose the algorithm does not terminate after $T$ iterations. Then, we have a sequence $\{ \B{\beta}^t \}_{0}^T$ of stationary solutions (since these are the outputs of Algorithm 1---see Theorem \ref{theorem:cdconvergence}). Let $S_{t} = \text{Supp}(\B{\beta}^t) $. Each stationary solution $\B{\beta}^t$ is a minimizer of the convex function $\B{\beta}_{S_t} \mapsto G(\B{\beta}_{S_t})$, and consequently $P(\B{\beta}^t) = \min \{ P(\B{\beta}) ~|~
\B{\beta} \in \mathbb{R}^{p}, \text{Supp}(\B\beta)=S_t \}$. Moreover, the definition of Step 2 and the descent property of cyclic CD imply $P(\B{\beta}^{T}) < P(\B{\beta}^{T-1}) < \dots < P(\B{\beta}^{0})$. Therefore, the same support cannot appear more than once in the sequence $\{ \B{\beta}^t \}_{0}^T$, and we conclude that the algorithm terminates in a finite number of iterations with a solution $\B{\beta}^{*}$. Finally, we note that $\B{\beta}^{*}$ is the output of Algorithm 1 so it must satisfy the characterization given in Theorem \ref{theorem:cdconvergence}. Moreover, the search in Step 2 must fail at $\B{\beta}^{*}$ (otherwise, the algorithm does not terminate), and thus \eqref{eq:inescapable} holds.
\end{proof}

\subsection{Proof of Lemma \ref{lemma:comparison}}
\begin{proof}
First, we recall that $\delta_{j}=|\hat{L} \beta^{*}_{j} - {\nabla_{j} g(\B{\beta}^{*})}|$ for any $j \in [p]$. Let $S = \text{Supp}(\B\beta^{*})$ and fix some $j \in S$. By Theorem \ref{theorem:cdconvergence}, we have $\B{\beta}^{*}_S \in \argmin_{\B{\beta}_S}~G(\B{\beta}_S)$, which is equivalent to $0 \in \partial G(\B\beta^{*}_{S})$. The zero subgradient condition directly implies that $\nabla_{j} g(\B\beta^{*}_{S}) = - \lambda_1 \sign(\beta_{j}^{*}) - 2 \lambda_2 |\beta_{j}^{*}|$. Substituting this expression into $\delta_j$, we get:
\begin{align}
    \delta_{j} & = | (\hat{L} + 2\lambda_2) \beta^{*}_{j} + \lambda_1 \sign(\beta_{j}^{*})| \nonumber \\
    & = (\hat{L} + 2\lambda_2) |\beta^{*}_{j}| + \lambda_1 \nonumber \\
    & \geq \sqrt{2 \lambda_0 (\hat{L} + 2 \lambda_2)} + \lambda_1,
    \label{eq:deltajcoef}
\end{align}
where \eqref{eq:deltajcoef} follows from inequality $|\beta^{*}_{j}| \geq \sqrt{\frac{2 \lambda_0}{\hat{L}_j + 2 \lambda_2}}$ (due to Theorem \ref{theorem:cdconvergence}) and the fact that $\hat{L} \geq \hat{L}_j$. Now fix some $i \notin S$. Using Theorem \ref{theorem:cdconvergence} and $\hat{L} \geq \hat{L}_i$:
\begin{align} \label{eq:deltaiout}
  \delta_i =  |\nabla_i g(\B{\beta}^{*} )|  \leq \sqrt{2 \lambda_0 (\hat{L}_i + 2 \lambda_2)} + \lambda_1 \leq \sqrt{2 \lambda_0 (\hat{L} + 2 \lambda_2)} + \lambda_1.
\end{align}
Inequalities \eqref{eq:deltajcoef} and \eqref{eq:deltaiout} imply that $\delta_j \geq \delta_i$ for any $j \in S$ and $i \notin S$. Since $\| \B{\beta}^{*} \|_0= k$, we have $\delta_{(k)} \geq \delta_i $ for any $i \notin S$, which combined with the fact that $\B{\beta}^{*}_S \in \argmin_{\B{\beta}_S}~G(\B{\beta}_S)$ (from Theorem \ref{theorem:cdconvergence}) implies that $\B{\beta}^{*}$ satisfies the fixed point conditions for IHT stated in Theorem \ref{theorem:ihtconvergence}.
\end{proof}

\subsection{Choice of $J$: sorted gradients}\label{sec:choice-J}
Let $\beta_j^1$ denote the solution obtained after the first application of the thresholding operator to minimize the function in \eqref{eq:contmin}. Note that we are interested in coordinates with nonzero $\beta_j^1$ (because in step 1 of Algorithm 3, we check whether $\beta_j$ should be zero). Coordinates with nonzero $\beta_j^1$ must satisfy $|\nabla_{j} g(\B{\beta}^{t} - \B{e}_i {\beta}_i^{t})| - \lambda_1 > 0$ (this follows from \eqref{eq:thresholding} with $\lambda_0 = 0$). Using \eqref{eq:blockdescent}, it can be readily seen that we have the following lower bound on the improvement in the objective if $|\nabla_{j} g(\B{\beta}^{t} - \B{e}_i {\beta}_i^{t})| - \lambda_1 > 0$:
\begin{equation} \label{eq:lowerbdimp}
P(\B{\beta}^{t} - \B{e}_i {\beta}_i^{t}) - P(\B{\beta}^{t} - \B{e}_i {\beta}_i^{t} + \B{e}_j \beta_j^1 ) \geq \frac{1}{2 (L_j + 2 \lambda_2)} (|\nabla_{j} g(\B{\beta}^{t} - \B{e}_i {\beta}_i^{t})| - \lambda_1)^2 - \lambda_0.
\end{equation}
The choices of $j \in S^c$ with larger $|\nabla_{j} g(\B{\beta}^{t} - \B{e}_i {\beta}_i^{t})|$ have a larger r.h.s. in inequality \eqref{eq:lowerbdimp} and thus are expected to have lower objectives. Therefore, instead of searching across all values of $j \in S^c$ in Step 2 of Algorithm 2, we restrict $j \in J$.

\subsection{Proof of Lemma \ref{lemma:sparse}}
\begin{proof}
Problem \eqref{eq:PMIO} can be rewritten as 
\begin{align*}
\min\limits_{\B{\beta}, \B{z}_{I}} \Big[ \min_{\B{z}_{\mathcal{I}^c}} \quad  & G(\B{\beta}) + \lambda_0 \sum\limits_{i=1}^{p} z_i  \Big] \\
& |\beta_i| \leq  \mathcal{M} z_i, ~~ i \in [p] \\
& z_i \in [0,1], ~~ i \notin \mathcal{I} \\
& z_i \in \left\{0,1\right\}, ~~ i \in \mathcal{I}. 
\end{align*}
Solving the inner minimization problem leads to $z_i = {|\beta_i|}/{\mathcal{M}}$ for every $i \in \mathcal{I}^c$. Plugging the latter solution into the objective function, we arrive to the result of the lemma.
\end{proof}

\subsection{Proof of Theorem~\ref{restricted-strong-convexity}} 

\subsubsection{A useful proposition}  \label{sec: hoeffding-sup}
Proposition \ref{hoeffding-sup} (stated below) is an essential step in the proof of Theorem~\ref{restricted-strong-convexity}. This allows us to control the supremum of a random variable (of interest) over a bounded set of $2k$ sparse vectors. \begin{prop} \label{hoeffding-sup}Let $\delta \in (0,1/2)$, $\tau =  6L \sqrt{\frac{\lambda(k)}{n} \left(  k \log\left( Rp/k \right) + \log\left( 1/\delta \right) \right)}$ and define
\begin{equation}\label{defn-deltah}
\Delta(\mathbf{z}) =
	\frac{1}{n} \sum_{i=1}^n f \left( \langle \mathbf{x}_i,  \B{\beta}^* + \mathbf{z}  \rangle ;  y_i \right)  - \frac{1}{n} \sum_{i=1}^n  f \left( \langle \mathbf{x}_i,  \B{\beta}^* \rangle ;  y_i \right),~~~\forall \mathbf{z}.
	    \end{equation}

If Assumptions \ref{asu1}, \ref{asu4}($k$) and  \ref{asu5}($\delta$) hold then:
	$$\mathbb{P} \left(  \sup \limits_{ \substack{ \mathbf{z} \in \mathbb{R}^{p} \\ \| \mathbf{z}  \|_0 \le 2k, \ \| \mathbf{z}  \|_2 \le 3R } } \left\{ \left|  \Delta(\mathbf{z})  - \mathbb{E}\left( \Delta(\mathbf{z}) \right) \right| -   \tau  \| \mathbf{z} \|_2\vee \tau^2 \right\} \ge 0  \right)  \le \delta.$$
\end{prop}

\begin{proof}
We divide the proof into 3 steps. First, we upper-bound the quantity $\left| \Delta(\mathbf{z}) - \mathbb{E}\left( \Delta(\mathbf{z})\right) \right|$ for any $2k$ sparse vector with Hoeffding inequality. Second, we extend the result to the maximum over an $\epsilon$-net. We finally control the maximum over the compact set and derive our proposition.

\textbf{Step 1: } We fix $\mathbf{z} \in \mathbb{R}^p$ such that $\| \mathbf{z} \|_0 \le 2k$ and introduce the random variables $Z_i, \forall i$ as follows
$$ Z_i = f \left( \langle \mathbf{x}_i,  \B{\beta^*} + \mathbf{z}  \rangle ;  y_i \right) -  f \left( \langle \mathbf{x}_i,  \B{\beta^*}   \rangle ;  y_i \right).$$ 
Assumption \ref{asu1} guarantees that $f(.;y)$ is Lipschitz with constant $L$, which leads to:
$$|Z_i | \le L \left| \langle \mathbf{x}_i, \mathbf{z}  \rangle \right|.$$
Note that $\Delta(  \mathbf{z}) = \frac{1}{n} \sum \limits_{i=1}^n Z_i$. We introduce a small quantity $\eta>0$ later explicited in the proof. Using Hoeffding's inequality and Assumption \ref{asu4}$(k)$ it holds: $\forall t>0$,
\begin{align} \label{upper-bound-2k-sparse}
\begin{split}
\mathbb{P}\left( \left| \Delta( \mathbf{z}) - \mathbb{E}\left( \Delta( \mathbf{z}) \right) \right| \ge t  \left(\| \mathbf{z} \|_2\vee \eta\right) \biggr\rvert \mathbf{X} \right) 
&\le 2\exp \left(- \frac{2 n^2 t^2  \left(\| \mathbf{z} \|_2\vee \eta\right)^2  }{ \sum_{i=1}^n L^2  \langle \mathbf{x}_i, \mathbf{z}  \rangle^2}  \right)\\
&= 2\exp \left(- \frac{2 n^2 t^2 \left(\| \mathbf{z} \|_2^2\vee \eta^2\right) }{ L^2 \| \mathbf{X} \mathbf{z} \|_2^2 }  \right)\\
&\le 2\exp \left(- \frac{2 n t^2  \left(\| \mathbf{z} \|_2^2\vee \eta^2\right) }{ L^2 \lambda(k) \| \mathbf{z} \|_2^2 }  \right)\\
&\le 2\exp \left(- \frac{2 n t^2}{ L^2 \lambda(k) }  \right).
\end{split}
\end{align}
Note that in the above display, the r.h.s. bound does not depend upon $\M{X}$, the conditioning event in the l.h.s. of display~\eqref{upper-bound-2k-sparse}. 

\textbf{Step 2: } We consider an $\epsilon$-net argument to extend the result to any $2k$ sparse vector satisfying $\| \mathbf{z} \|_2 \le 3R$. We recall that an $\epsilon$-net of a set $\mathcal{I}$ is a subset $\mathcal{N}$ of $\mathcal{I}$ such that each element of ${\mathcal I}$ is at a distance at most  $\epsilon$ of $\mathcal{N}$.

\sloppy Lemma 1.18 in \citet{lecture-notes} proves that for any value $\epsilon \in (0,1)$, the ball  $\left\{ \mathbf{z} \in \mathbb{R}^d: \ \|  \mathbf{z} \|_2 \le 3R  \right\}$ has an $\epsilon$-net of cardinality $| \mathcal{N} | \le \left(\frac{6R+1}{\epsilon} \right)^d.$ Consequently, we can fix an $\epsilon$-net $\mathcal{N}_{k,R}$ of the set $ \mathcal{I}_{k,R} = \left\{ \mathbf{z} \in \mathbb{R}^p: \ \|  \mathbf{z} \|_0 = 2k \ ; \ \|  \mathbf{z} \|_2 \le 3R  \right\}$ with cardinality $\binom{p}{2k} \left(\frac{6R+1}{\epsilon} \right)^{2k}$. This along with 
 equation~\eqref{upper-bound-2k-sparse} leads to: $\forall t>0$,
\begin{equation}\label{union-bound-11} 
\begin{aligned}
\mathbb{P}\left( \sup \limits_{ \mathbf{z} \in \mathcal{N}_{k,R} } \big(\left| \Delta( \mathbf{z}) - \mathbb{E}\left( \Delta( \mathbf{z}) \right) \right| - t  \left(\| \mathbf{z} \|_2\vee \eta\right) \big) \ge 0 \right)  \\
\le \binom{p}{2k} \left(\frac{6R+1}{\epsilon} \right)^{2k} 2 \exp \left(- \frac{2 n t^2 }{ L^2 \lambda(k) } \right).
\end{aligned}
\end{equation}

\textbf{Step 3: } We finally extend the result to any vector in $\mathcal{I}_{k,R}$. This is done by expressing the supremum of the random variable $\left| \Delta( \mathbf{z}) - \mathbb{E}\left( \Delta( \mathbf{z}) \right) \right|$  over the entire set $\mathcal{I}_{k,R}$ with respect to its supremum over the $\epsilon$-net $\mathcal{N}_{k,R}$.


For $\mathbf{z} \in \mathcal{I}_{k,R}$, there exists $\mathbf{z}_0 \in \mathcal{N}_{k,R}$ such that $\| \mathbf{z} - \mathbf{z}_0 \|_2 \le \epsilon.$ With Assumption \ref{asu1} and Cauchy-Schwartz inequality, we obtain:
\begin{equation}\label{delta-mz-11}
\left| \Delta(\mathbf{z})- \Delta(\mathbf{z}_0) \right| \le \frac{1}{n} \sum_{i=1}^n L \left| \langle \mathbf{x}_i, \mathbf{z} - \mathbf{z}_0 \rangle  \right| \le \frac{1}{\sqrt{n} } L \| \mathbf{X}(\mathbf{z} - \mathbf{z}_0 ) \|_2 \le  L \sqrt{\lambda(k)} \epsilon.
\end{equation}
For a fixed value of $t$, we define the operator 
$$f_t(\mathbf{z}) = \left| \Delta(\mathbf{z}) - \mathbb{E}\left( \Delta(\mathbf{z}) \right) \right| - t   \left(\| \mathbf{z} \|_2\vee \eta\right), \ \forall \mathbf{z}.$$ 

Using the (reverse) triangle inequality, it holds that:
\begin{equation}\label{triangle-ineq-111}
\begin{aligned}
\left|  \Delta(\mathbf{z})  -  \Delta(\mathbf{z}_0)   \right| \geq \left| \Delta(\mathbf{z}) - \mathbb{E}\left( \Delta(\mathbf{z}) \right) \right| - 
\left| \Delta(\mathbf{z}_0) - \mathbb{E}\left( \Delta(\mathbf{z}_0) \right) \right| -
\left| \mathbb{E}\left( \Delta(\mathbf{z}) \right) - \mathbb{E}\left( \Delta(\mathbf{z}_0) \right) \right|.
\end{aligned}
\end{equation}

In addition, note that:
\begin{equation}\label{triangle-ineq-112}
\left( \| \mathbf{z} - \mathbf{z}_0 \|_2 \vee \eta\right) ~~+~~~ (\| \M{z} \|_2 \vee \eta)~~~ \geq~~~\| \M{z}_0\|_2 \vee \eta.
\end{equation}
Using~\eqref{triangle-ineq-111} and~\eqref{triangle-ineq-112} in $f_{t}(\M{z})$ we have:
\begin{align}\label{supremum0} 
\begin{split}
f_t(\mathbf{z_0}) 
&\ge f_t(\mathbf{z}) -  \left| \Delta(\mathbf{z})- \Delta(\mathbf{z}_0) \right|  - \left| \mathbb{E}\left(\Delta(\mathbf{z})- \Delta(\mathbf{z}_0) \right)\right| - t \left( \| \mathbf{z} - \mathbf{z}_0 \|_2 \vee \eta\right).
\end{split}
\end{align}
Now using~\eqref{delta-mz-11} and Jensen's inequality, we have:
\begin{equation}\label{triangle-jensen1}
\begin{aligned}
\left| \Delta(\mathbf{z})- \Delta(\mathbf{z}_0) \right| \le  L \sqrt{\lambda(k)} \epsilon \\
\left|\mathbb{E}\left( \Delta(\mathbf{z})- \Delta(\mathbf{z}_0) \right)\right| \le  L \sqrt{\lambda(k)} \epsilon.
\end{aligned}
\end{equation}
Applying~\eqref{triangle-jensen1} and $\| \M{z} - \M{z}_0 \|_2 \leq \epsilon$ to the right hand side of~\eqref{supremum0}, we have:
\begin{align}\label{supremum}
f_t(\mathbf{z}_0) \ge  f_t(\mathbf{z}) - 2 L \sqrt{\lambda(k)} \epsilon - t (\epsilon \vee \eta).
\end{align}
Suppose we choose
$$\eta= 4 L \sqrt{ \frac{\lambda(k)}{n}k\log\left( p/k \right)}$$
$$ \epsilon= \frac{\eta^2}{4 L \sqrt{\lambda(k)}} ~~~=  \sqrt{ \lambda(k) } \frac{4 Lk \log \left( p/k \right)}{n}  $$
$$ t \geq \eta/4.$$
This implies that:
$$\epsilon \leq \eta \text{~~~~~~~~(using $k/n\log(p/k) \leq 1$ by Assumption 5)},$$ 
$$L \sqrt{\lambda(k)} \epsilon \leq t \eta\text{~~~~~~~~(using $t \geq \eta/4 \implies t\eta \geq \frac{\eta^2}{4}=L\sqrt{\lambda(k)}\epsilon$)}.$$ 
Using the above, we obtain a lower bound to the r.h.s. of~\eqref{supremum}. This leads to the following chain of inequalities:
\begin{align}\label{supremum2} 
\begin{split}
f_t(\mathbf{z}) - 2 L \sqrt{\lambda(k)} \epsilon - t (\epsilon \vee \eta)
&\ge f_t(\mathbf{z}) - 3t \eta \\
&=  \left| \Delta(\mathbf{z}) - \mathbb{E}\left( \Delta(\mathbf{z}) \right) \right| - t   \left(\| \mathbf{z} \|_2\vee \eta\right) - 3t \eta \\
&\ge \left| \Delta(\mathbf{z}) - \mathbb{E}\left( \Delta(\mathbf{z}) \right) \right| -4 t   \left(\| \mathbf{z} \|_2\vee \eta\right) \\
&=f_{4t}(\mathbf{z}).
\end{split}
\end{align}
Consequently, Equations \eqref{supremum} and \eqref{supremum2} lead to: 
{
$$\forall t \ge \eta/4, \ \forall \mathbf{z} \in \mathcal{I}_{k,R}, \ \exists \mathbf{z}_0 \in \mathcal{N}_{k,R}: \ f_{4t}(\mathbf{z}) \le f_t(\mathbf{z}_0),$$ }
which can be equivalently written as:
{
\begin{equation}\label{supremum-domination}
\sup \limits_{ \mathbf{z} \in \mathcal{I}_{k,R} } f_t(\mathbf{z}) \le \sup \limits_{ \mathbf{y} \in \mathcal{N}_{k,R} } f_{t/4}(\mathbf{y}), \ \forall t \ge \eta.
\end{equation}
}

Lemma 2.7 in \citet{lecture-notes} gives the relation $\binom{p}{2k} \le \left( \frac{pe}{2k}\right)^{2k}$. Using equation \eqref{supremum-domination} along with the 
union-bound~\eqref{union-bound-11}, it holds that: $\forall t \ge \eta$,
\begin{equation}\label{7Rpe-bound}
\begin{aligned} 
\mathbb{P}\left( \sup \limits_{ \mathbf{z} \in \mathcal{I}_{k,R} } f_t(\mathbf{z}) \ge 0 \right)
&\le \mathbb{P}\left( \sup \limits_{ \mathbf{z} \in \mathcal{N}_{k,R} } f_{t/4}(\mathbf{z}) \ge 0 \right)   \\
& \leq \binom{p}{2k} \left(\frac{6R+1}{\epsilon} \right)^{2k} 2 \exp \left(- \frac{2 n t^2 }{ 16L^2 \lambda(k) } \right)   \\
& \leq \left(\frac{pe}{2k}\right)^{2k} \left(\frac{6R+1}{\epsilon} \right)^{2k} 2 \exp \left(- \frac{2 n t^2 }{ 16L^2 \lambda(k) } \right)   \\
& \le  2\left( \frac{pe}{2k} \ \frac{7R}{\epsilon} \right)^{2k} \exp \left(- \frac{2 n t^2 }{16 L^2 \lambda(k) } \right). 
\end{aligned}
\end{equation}
By Assumption~\ref{asu5} we have $7e n \le 3L \sqrt{\lambda(k) } p \log \left( p/k \right)$; and using the definition of $\epsilon$ (above), it follows that:
$$ \frac{7e}{2\epsilon} \le \frac{p}{k} \le \frac{Rp}{k},$$
where in the last inequality we used the assumption $R \geq 1$.
Using this in~\eqref{7Rpe-bound} we have:
\begin{equation}\label{final-supremum}
\mathbb{P}\left( \sup \limits_{ \mathbf{z} \in \mathcal{I}_{k,R} } f_t(\mathbf{z}) \ge 0 \right)\le  2\left( \frac{Rp}{k} \right)^{4k}  \exp \left(- \frac{2 n t^2 }{16 L^2 \lambda(k) } \right), \ \forall t \ge \eta.
\end{equation}
We want the right-hand side of Equation~\eqref{final-supremum} to be smaller than $\delta$. To this end, we need to select $t^2 \ge \frac{16 L^2 \lambda(k)}{2n}\left[ 4k \log\left( \frac{ R p}{k }\right) + \log(2) + \log\left( \frac{1}{\delta}\right) \right]$ and $t \ge \eta$. 
A possible choice for $t$ that we use is:
$$\tau = 6L \sqrt{\frac{\lambda(k)}{n} \left(  k \log\left( Rp/k \right) + \log\left( 1/\delta \right) \right) }.$$

We conclude that with probability at least $1-\delta$: 
$$ \sup \limits_{ \mathbf{z} \in \mathcal{I}_{k,R} } f_{\tau}(\mathbf{z}) \le 0.$$
Note that 
$$f_{\tau}(\mathbf{z}) = \left| \Delta(\mathbf{z}) - \mathbb{E}\left( \Delta(\mathbf{z}) \right) \right| - \tau   \left(\| \mathbf{z} \|_2\vee \eta\right) \ge  \left| \Delta(\mathbf{z}) - \mathbb{E}\left( \Delta(\mathbf{z}) \right) \right| - \tau \| \mathbf{z} \|_2\vee \tau^2.$$ 
It then holds with probability at least $1-\delta$ that:  
$$\sup \limits_{ \substack{ \mathbf{z} \in \mathbb{R}^{p} \\ \| \mathbf{z}  \|_0 \le 2k, \ \| \mathbf{z}  \|_2 \le 3R } } \left\{ \left|  \Delta(\mathbf{z})  - \mathbb{E}\left( \Delta(\mathbf{z}) \right) \right| -   \tau  \| \mathbf{z} \|_2\vee \tau^2 \right\} \le 0,$$
which concludes the proof.
\end{proof}

\subsubsection {Proof of Theorem \ref{restricted-strong-convexity} }  \label{sec: appendix_restricted-strong-convexity}
\begin{proof}
The $2k$ sparse vector  $\mathbf{h} = \hat{\B{\beta}} - \B{\beta}^*$ satisfies: $\| \mathbf{h} \|_2 \le \| \hat{\B{\beta}} \|_2 + \| \B{\beta}^*\|_2 \le 3R.$
Thus applying 
Proposition \ref{hoeffding-sup} to $\Delta(\M{h})$ (see the definition in Equation~\ref{defn-deltah})
it holds with probability at least $1 - \delta$ that:
\begin{align} \label{lower-bound-expectation}
\begin{split}
\Delta(\mathbf{h})
& \ge  \mathbb{E}\left( \Delta(\mathbf{h}) \right) -   \tau  \| \mathbf{h}\|_2\vee \tau^2\\
&=  \frac{1}{n} \mathbb{E} \left( \sum_{i=1}^n f \left( \langle \mathbf{x}_i,  \B{\beta}^* + \mathbf{h}  \rangle ;  y_i \right)  - \sum_{i=1}^n  f \left( \langle \mathbf{x}_i,  \B{\beta}^* \rangle ;  y_i \right)  \right) -  \tau  \| \mathbf{h}\|_2\vee \tau^2\\
&= \mathbb{E} \left[ f \left( \langle \mathbf{x}_i,  \B{\beta}^* + \mathbf{h}  \rangle ;  y_i \right)  -  f \left( \langle \mathbf{x}_i,  \B{\beta}^*  \rangle ;  y_i \right)  \right] -   \tau \| \mathbf{h}\|_2\vee \tau^2\\
&=\mathcal{L}(\B{\beta}^* + \mathbf{h}) - \mathcal{L}(\B{\beta}^*) -   \tau  \| \mathbf{h}\|_2\vee \tau^2.
\end{split}
\end{align}
To control the difference $\mathcal{L}(\B{\beta}^* + \mathbf{h}) - \mathcal{L}(\B{\beta}^*)$ we consider the following two cases.


\noindent {\bf{Case 1}} ($\| \mathbf{h} \|_2 \le r(k)$): If $\| \mathbf{h} \|_2 \le r(k)$ then by definition of $r(k)$ in~\eqref{growth-cond-defn}  it holds that
\begin{equation}\label{case1-LB}
\mathcal{L}(\B{\beta}^* + \mathbf{h}) - \mathcal{L}(\B{\beta}^*) \ge \frac{1}{4}  \kappa(k) \|\mathbf{h} \|_2^2.
\end{equation}


\noindent {\bf{Case 2}:} 
We consider the case when $\| \mathbf{h} \|_2 > r(k)$. Since $\B\beta \mapsto \mathcal{L}(\B\beta)$ is convex, the one-dimensional function 
$t \to \mathcal{L}\left( \B{\beta}^* + t \mathbf{z} \right)$ is convex and it holds:
\begin{align} \label{trick}
\begin{split}
\mathcal{L}(\B{\beta}^* + \mathbf{h}) - \mathcal{L}(\B{\beta}^*) 
& \ge  \frac{\|\mathbf{h} \|_2}{r(k) } \left\{ \mathcal{L} \left(\B{\beta}^* + \frac{r(k) }{\|\mathbf{h} \|_2} \mathbf{h} \right)   - \mathcal{L}	(\B{\beta}^* )  \right\} \\
& \ge  \frac{\|\mathbf{h} \|_2}{r(k) } \inf \limits_{ \substack{\mathbf{z}: \  \| \mathbf{z} \|_0 \le 2k \\ \ \ \ \ \ \| \mathbf{z} \|_2 = r(k)}  }  \left\{ \mathcal{L}(\B{\beta}^* + \mathbf{z} )   - \mathcal{L}(\B{\beta}^*)  \right\} \\
&\ge    \frac{\|\mathbf{h} \|_2}{r(k)} \ \frac{1}{4} \kappa(k)  r(k)^2 \\
&= \frac{1}{4} \kappa(k) r(k)  \|\mathbf{h} \|_2.
\end{split}
\end{align}
In~\eqref{trick}, the first inequality follows by observing that the one-dimensional function
$x \mapsto (g(a + x) - g(a))/x$ defined on $x >0$ is increasing for any one-dimensional convex function $x\mapsto g(x)$.
The last inequality in display~\eqref{trick} follows from the definition of $r(k)$ as in~\eqref{growth-cond-defn}.

Combining Equations \eqref{lower-bound-expectation}, \eqref{case1-LB} and \eqref{trick}, we obtain the following restricted strong convexity property (which holds with probability at least $1-\delta$):
$$\Delta(\mathbf{h}) \ge \frac{1}{4}  \kappa(k) \left\{ \|\mathbf{h}\|_2^2 \wedge  r(k)  \|\mathbf{h}\|_2 \right\}  -  \tau  \| \mathbf{h} \|_2\vee \tau^2.$$
\end{proof}

\subsection {Proof of Theorem \ref{main-results} } \label{sec: appendix_main-results}
\begin{proof}
Using the observation that $\hat{\B\beta}$ is a minimizer of Problem~\eqref{learning-l0};
and the representation $\hat{\B{\beta}} = \B{\beta}^* + \mathbf{h}$, it holds that:
$$ \frac{1}{n} \sum_{i=1}^n f \left( \langle \mathbf{x}_i,  \B{\beta}^* + \mathbf{h}  \rangle ;  y_i \right)  \le \frac{1}{n} \sum_{i=1}^n  f \left( \langle \mathbf{x}_i,  \B{\beta}^* \rangle ;  y_i \right).$$
This relation is equivalent to saying that $\Delta(\mathbf{h}) \le 0$. Consequently, by combining this relation with~\eqref{conclude-theorem-restricted-strong-convexity} (see Theorem~\ref{restricted-strong-convexity}), it holds with probability at least $1-\delta$:
\begin{equation}\label{thm5-ineq1}
\frac{1}{4}  \kappa(k) \left\{ \|\mathbf{h}\|_2^2 \wedge r(k) \|\mathbf{h}\|_2 \right\} \le \tau  \|\mathbf{h}\|_2\vee \tau^2.
\end{equation}
We now consider two cases.


\noindent{Case 1:} Let $\|\mathbf{h}\|_2 \le \tau$. Then we have that
$$\|\mathbf{h}\|_2 \le \tau = 6L \sqrt{\frac{\lambda(k)}{n} \left(  k\log\left( Rp/k \right) + \log\left( 1/\delta \right) \right) }.$$ 


\noindent{Case 2:} We consider the case where 
$\|\mathbf{h}\|_2 > \tau$.  With probability $1-\delta$ it holds (from Inequality~\ref{thm5-ineq1}) that:
$$\frac{1}{4}  \kappa(k) \left\{ \|\mathbf{h}\|_2^2 \wedge r(k) \|\mathbf{h}\|_2 \right\} \le \tau  \|\mathbf{h}\|_2,$$
which can be simplified to:
\begin{equation} \label{main-theorem-constant}
\|\mathbf{h}\|_2 \wedge r(k) \le \frac{4\tau}{\kappa(k)}
=  \frac{24L}{\kappa(k)} \sqrt{\frac{\lambda(k)}{n} \left( k \log\left( R p/k \right) + \log\left( 1/\delta \right) \right) }.
\end{equation}
Note that by Assumption~\ref{asu5}($\delta$), the r.h.s. of the above is smaller than $r(k)$. The l.h.s. of~\eqref{main-theorem-constant} is the minimum of two terms $\|\mathbf{h}\|_2$ and $r(k)$; and since this is lower than $r(k)$, we conclude that 
$$\|\mathbf{h}\|_2 \leq \frac{24L}{\kappa(k)} \sqrt{\frac{\lambda(k)}{n} \left( k \log\left( R p/k \right) + \log\left( 1/\delta \right) \right) }. $$ 

Combining the results from Cases 1 and 2, we conclude that with probability at least $1 - \delta$, the following holds:
$$ \|\mathbf{h}\|_2 \leq C L \max \left\{1 , \frac{1}{\kappa(k)} \right\} \sqrt{\frac{\lambda(k)}{n} \left( k \log\left( R p/k \right) + \log\left( 1/\delta \right) \right) },$$
where $C$ is an universal constant. This concludes the proof of the theorem.

\end{proof}

\subsection {Proof of Corollary \ref{main-corollary}: } \label{sec: appendix_main-corollary}
\begin{proof}
Let us define the random variable: 
$$ W =  \frac{1}{L^2 \widetilde{\kappa}^2 \lambda(k)} \| \hat{\B{\beta}}  - \B{\beta}^*\|_2^2.$$
Since the difference $\hat{\B{\beta}}  - \B{\beta}^*$ is bounded, then $W$ is upper-bounded by a constant. In addition, because Assumption \ref{asu5} is satisfied for $\delta>0$ small enough, Theorem \ref{main-results} leads to the existence of a constant $C$ such that
$$\mathbb{P} \left( W \le \frac{C}{n}   \left(k \log(Rp/k) + \log(1/ \delta)  \right) \right) \ge 1 - \delta, \ \forall \delta \in (0,1).$$

This relation can be equivalently expressed as $$\mathbb{P}\left( W/C \ge \frac{k \log(p/k)}{n} + t \right) \le e^{-nt}, \  \forall t \ge 0.$$ 

Let us define $H = (k/n) \log(p/k)$. By integration it holds:
\begin{align} 
\begin{split}
\mathbb{E}(W) &= \displaystyle \int_0^{+ \infty}  C\mathbb{P}\left( |W| /C \ge t \right)dt\\
&\le \displaystyle \int_0^{+ \infty}  C \mathbb{P}\left( |W| /C \ge t + H \right) dt + CH\\
&\le \displaystyle \int_0^{+ \infty}  C e^{-nt}dt + CH =  \frac{C}{n} + CH \le 2CH
\end{split}
\end{align}

We finally conclude:
$$\mathbb{E} \| \hat{\B{\beta}} - \B{\beta}^*  \|^2_2  \lesssim  L^2 \lambda(k)\widetilde{\kappa}^2 \frac{k \log(p/k)}{n}.$$
\end{proof}

\section{Additional Experiments}\label{appendix-sec:expts}
Here we present additional experimental results complementing the real data set experiments presented in Section~\ref{sec:real-datasets}.
We compare the performance of our proposed algorithm with $\ell_{1}$-regularization. In particular, Figure~\ref{fig:aucvssupp-appendix-L0L2CD} considers the $\ell_0$-$\ell_2$ penalty with CD (Algorithm~1);
Figure~\ref{fig:aucvssupp-appendix-L0L1CDwLoc} considers the $\ell_0$-$\ell_1$ penalty with local search (Algorithm~2 with $m=1$); and 
Figure~\ref{fig:aucvssupp-appendix-L0L1CD} considers the  $\ell_0$-$\ell_1$ penalty with CD (Algorithm~1).

\begin{figure}[tb]
\centering
\textbf{L0L2 (CD) vs. L1}\par\medskip
\includegraphics[scale=0.5]{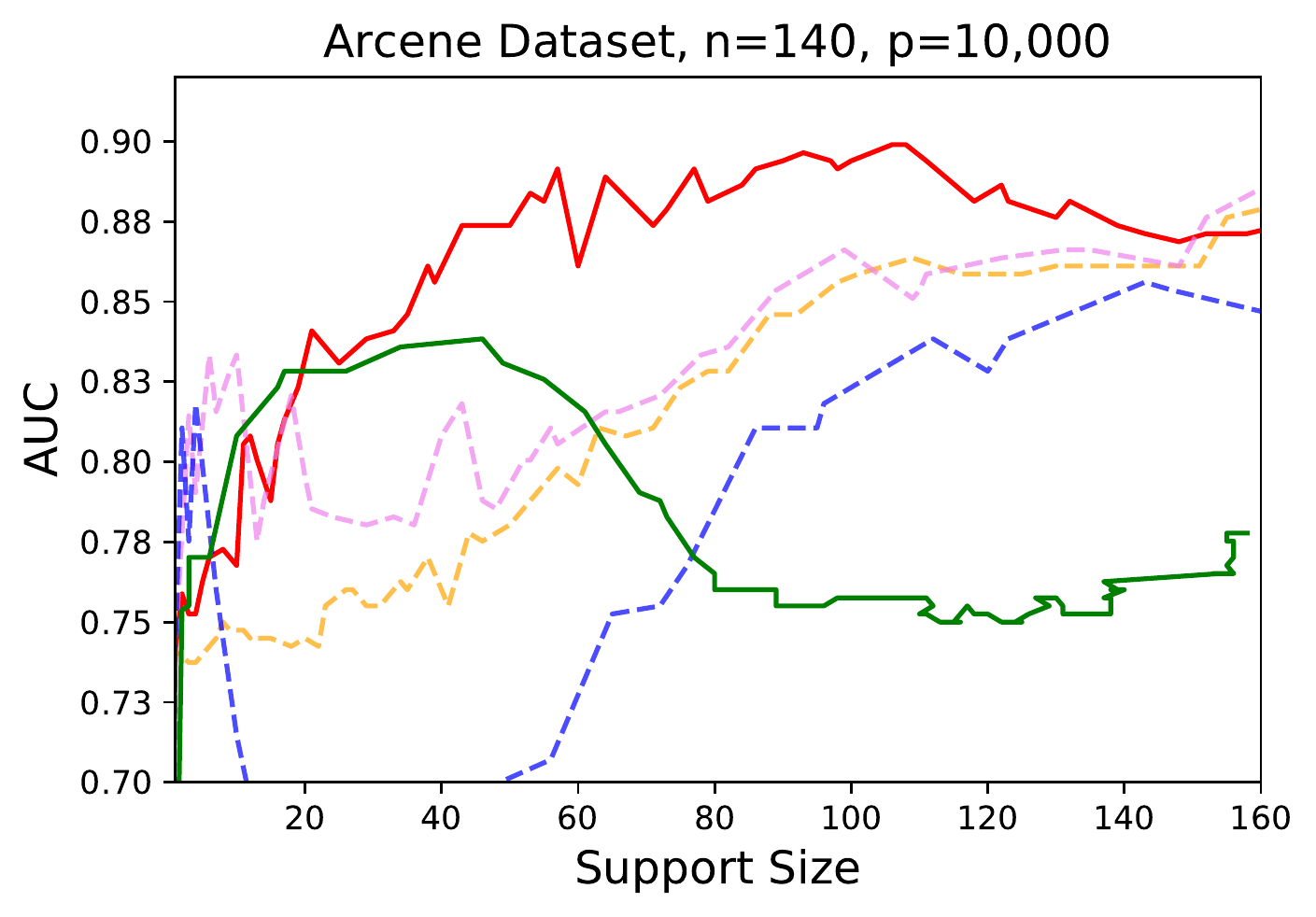}
\includegraphics[scale=0.5]{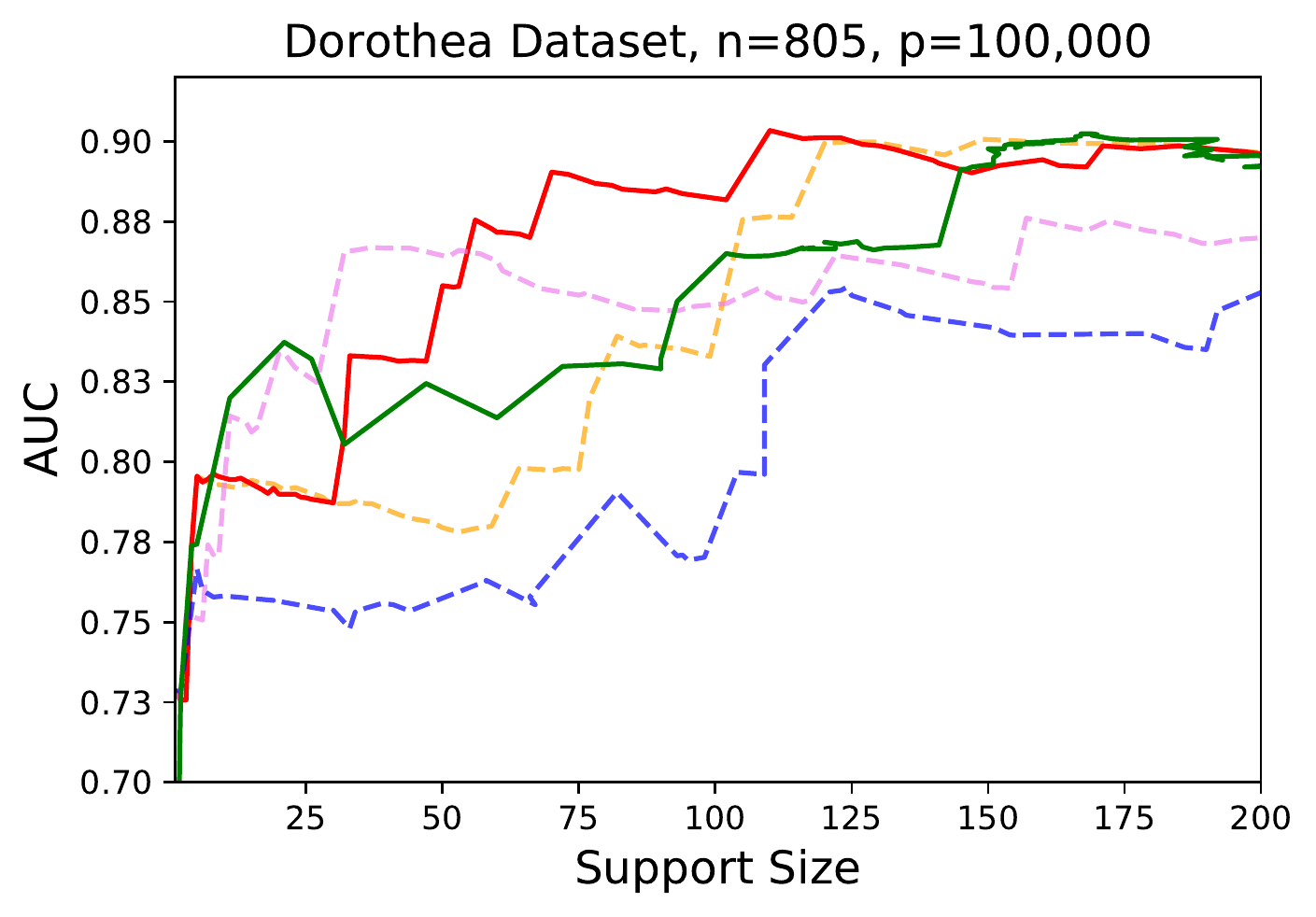}
\includegraphics[scale=0.5]{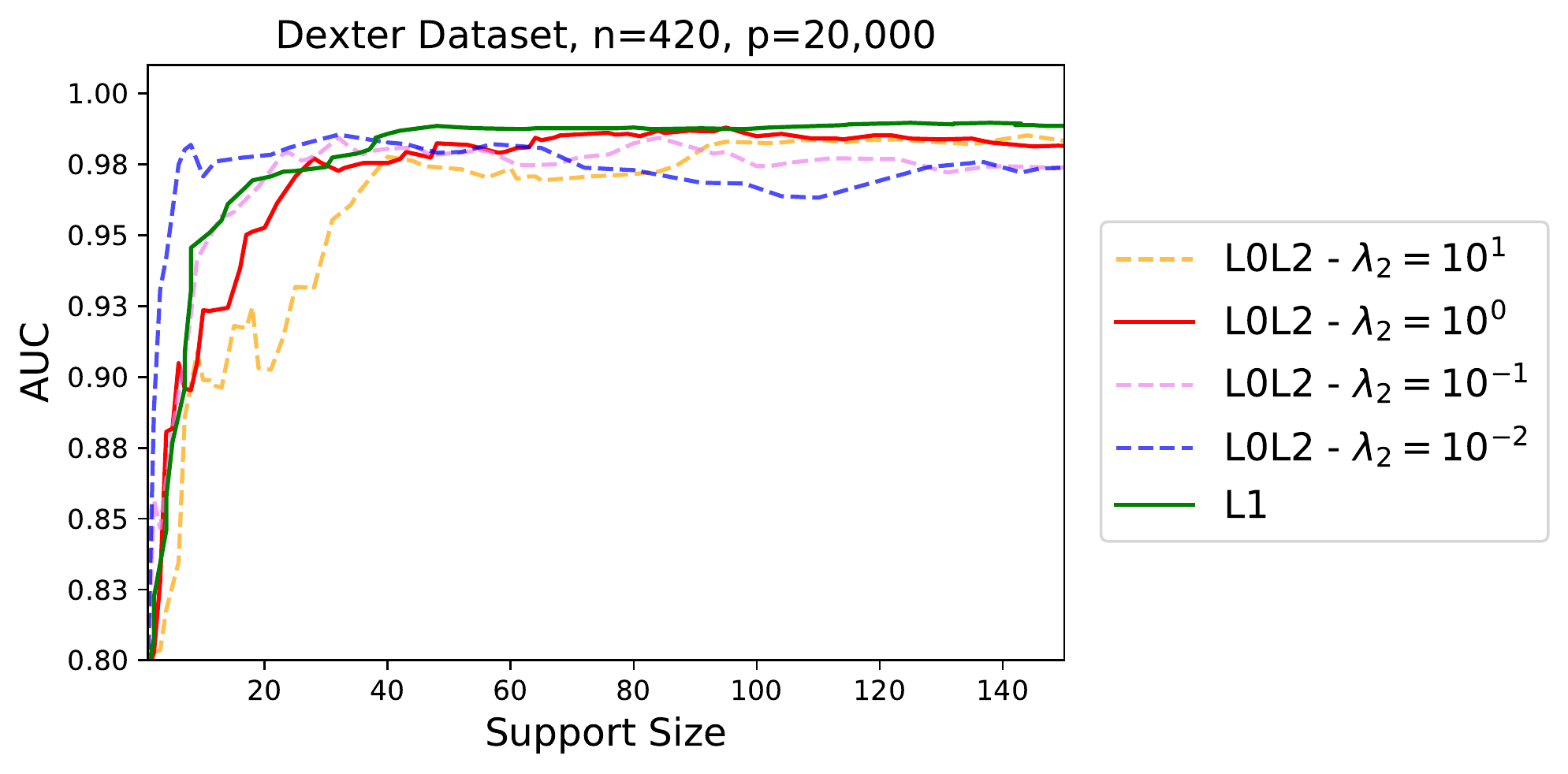}
\caption{Plots of AUC versus corresponding support sizes for the Arcene, Dorothea, and Dexter data sets. The green curves correspond to logistic regression with $\ell_1$ regularization. The other curves correspond to logistic regression with $\ell_0$-$\ell_2$ regularization using Algorithm 1 for different values of $\lambda_2$ (see legend). }
\label{fig:aucvssupp-appendix-L0L2CD}
\end{figure}

\begin{figure}[tb] 
\centering
\textbf{L0L1 (CD w. Local Search) vs. L1}\par\medskip
\includegraphics[scale=0.5]{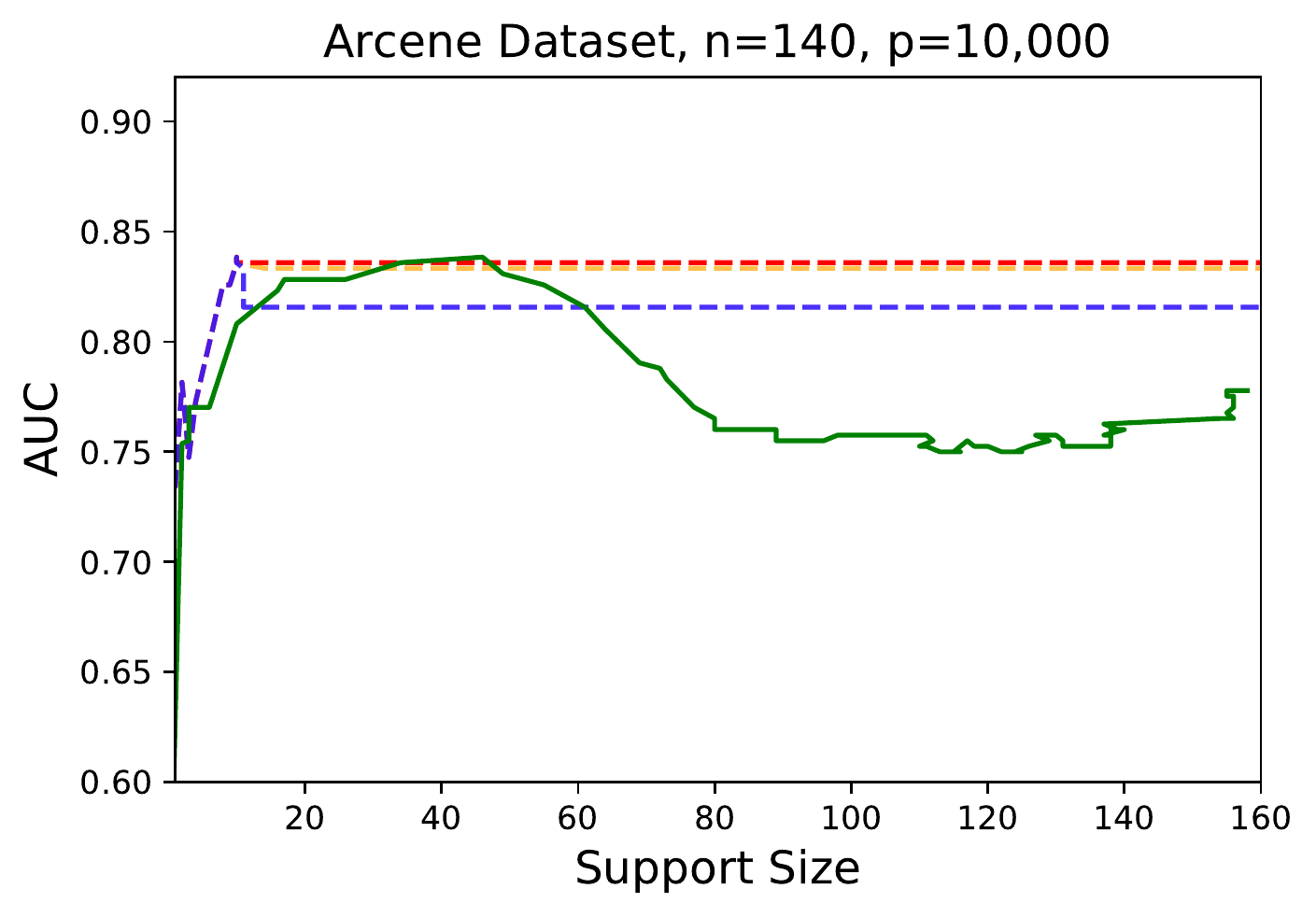}
\includegraphics[scale=0.5]{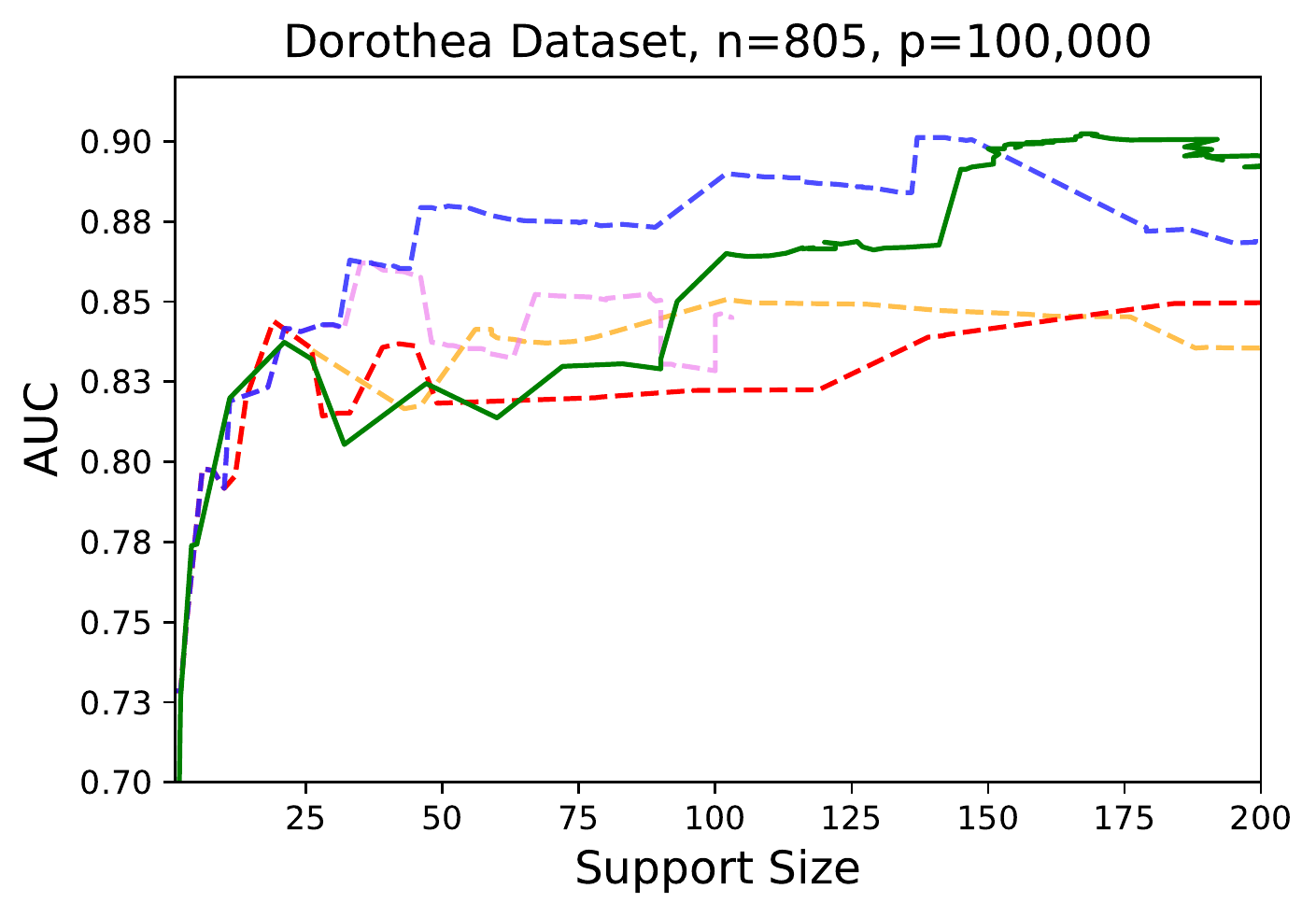}
\includegraphics[scale=0.5]{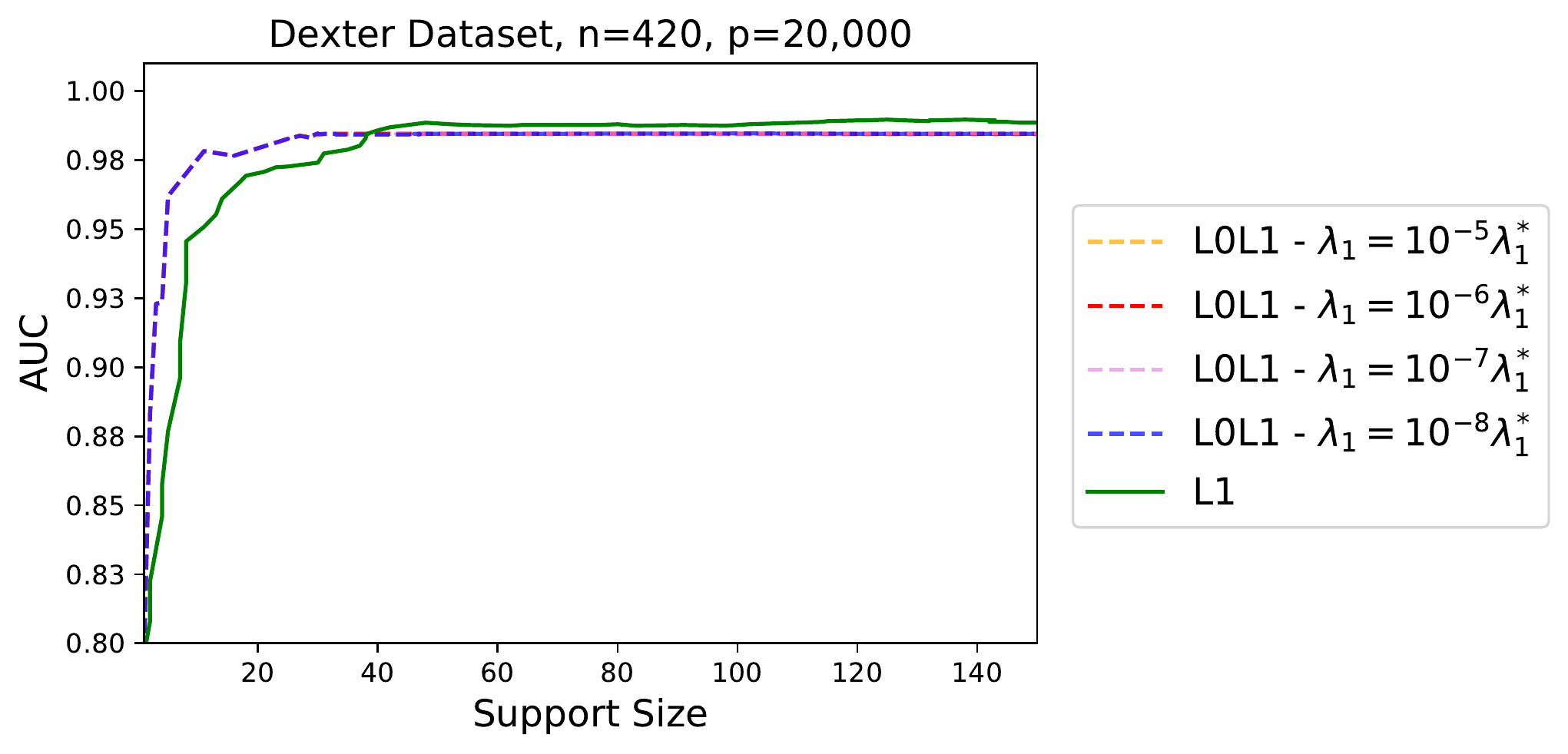}
\caption{Plots of the AUC versus the support size for the Arcene, Dorothea, and Dexter data sets. The green curves correspond to logistic regression with $\ell_1$ regularization. The other curves correspond to logistic regression with $\ell_0$-$\ell_1$  regularization using Algorithm 2 (with $m=1$) for different values of $\lambda_2$ (see legend). }
\label{fig:aucvssupp-appendix-L0L1CDwLoc}
\end{figure}

\begin{figure}[tb] 
\centering
\textbf{L0L1 (CD) vs. L1}\par\medskip
\includegraphics[scale=0.5]{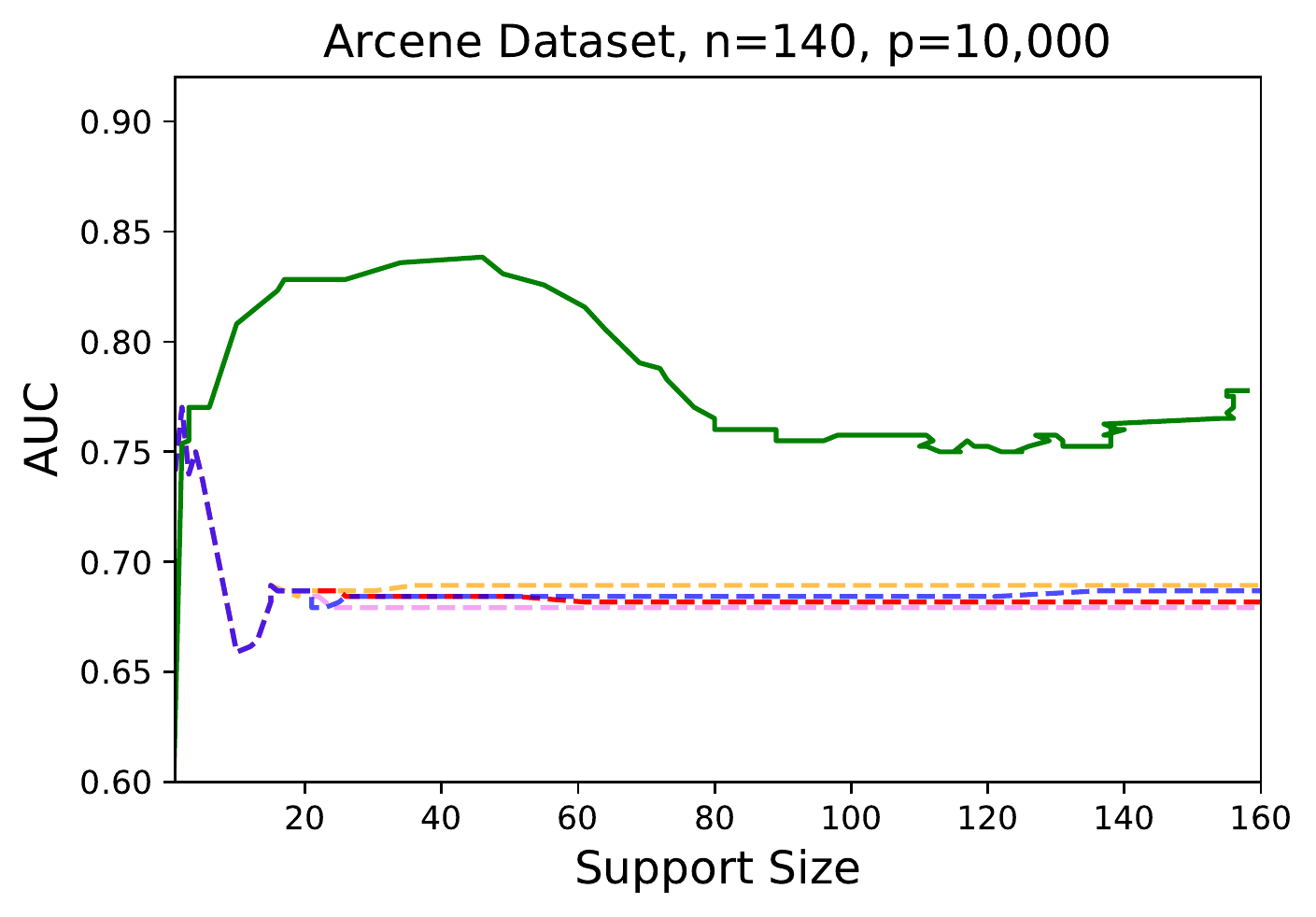}
\includegraphics[scale=0.5]{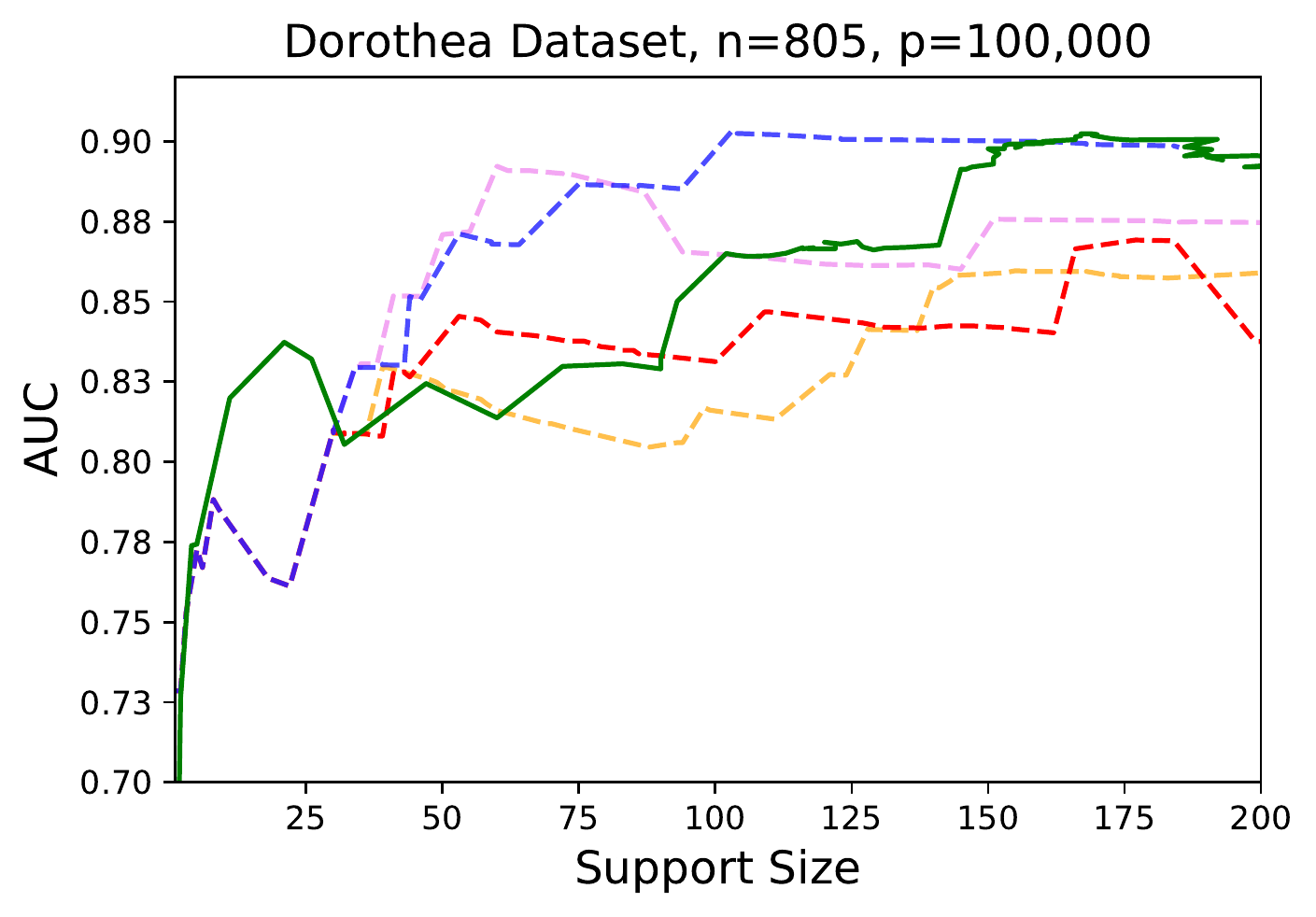}
\includegraphics[scale=0.5]{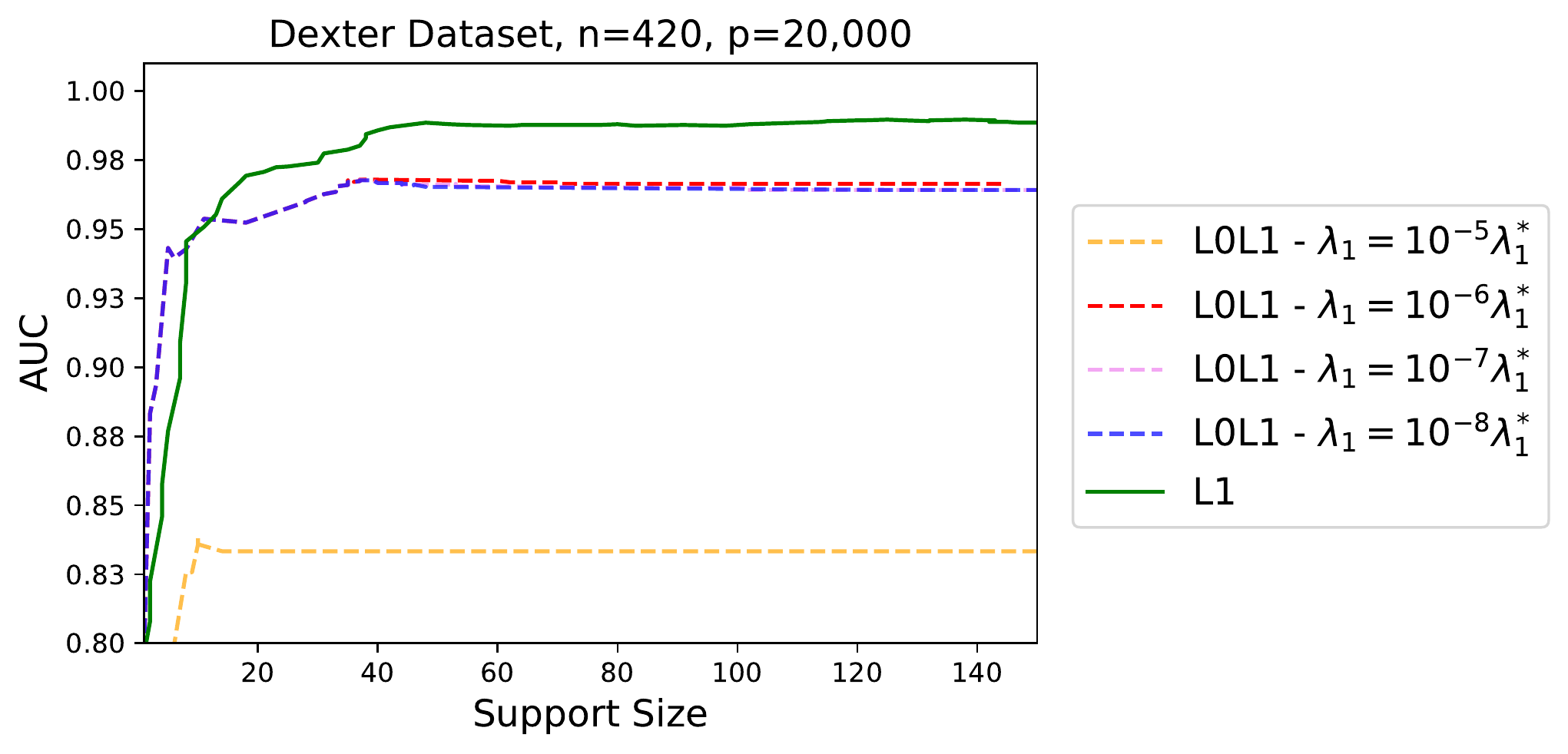}
\caption{Plots of the AUC versus the support size for the Arcene, Dorothea, and Dexter data sets. The green curves correspond to logistic regression with $\ell_1$ regularization. The other curves correspond to logistic regression with $\ell_0$-$\ell_1$ regularization using Algorithm 1 for different values of $\lambda_2$ (see legend). }
\label{fig:aucvssupp-appendix-L0L1CD}
\end{figure}

\clearpage
\bibliography{ref}

\end{document}